\documentclass[11pt]{article}
\usepackage[a4paper, left=3cm, right=3cm, top=3cm, bottom=3cm]{geometry}
\usepackage[utf8]{inputenc}
\usepackage{graphics}
\usepackage{times}
\usepackage{amsmath, amsfonts, amssymb, amsthm}
\usepackage{stmaryrd}
\usepackage{appendix}
\usepackage[hidelinks]{hyperref}

\usepackage{rotating}

\usepackage{todonotes}
\usepackage{caption}
\usepackage{subcaption}
\usepackage{algorithmic}
\usepackage{algorithm}
\usepackage{diagbox}
\usepackage{tablefootnote}
\usepackage{array}
\usepackage{makecell}

\usepackage{natbib}

\usepackage{booktabs}

\DeclareFontFamily{U}{mathx}{\hyphenchar\font45}
\DeclareFontShape{U}{mathx}{m}{n}{
      <5> <6> <7> <8> <9> <10>
      <10.95> <12> <14.4> <17.28> <20.74> <24.88>
      mathx10
      }{}
\DeclareSymbolFont{mathx}{U}{mathx}{m}{n}
\DeclareFontSubstitution{U}{mathx}{m}{n}
\DeclareMathAccent{\widecheck}{0}{mathx}{"71}

\newtheorem{definition}{Definition}
\newtheorem{assumption}{Assumption}
\newtheorem{proposition}{Proposition}
\newtheorem{theorem}{Theorem}
\newtheorem{lemma}{Lemma}
\newtheorem{corollary}{Corollary}

\newtheoremstyle{myrem}%
{2pt}
{2pt}
{}
{}
{\bfseries}
{.}
{.5em}
{}%
\theoremstyle{myrem}

\newcommand{\N}{\mathbb{N}}
\newcommand{\R}{\mathbb{R}}
\newcommand{\E}{\mathbb{E}}
\newcommand{\V}{\mathbb{V}}
\newcommand{\cE}{\mathcal E}
\newcommand{\cK}{\mathcal K}
\newcommand{\cO}{\mathcal O}
\newcommand{\cI}{\mathcal I}
\renewcommand{\P}{\mathbb{P}}
\newcommand{\Proba}{\mathbb{P}}
\newcommand{\hg}{\widehat{g}}

\newcommand{\thetat}{\theta^{(t)}}

\newcommand{\ind}[1]{\mathbf 1_{#1}}

\newcommand{\setint}[1]{\llbracket #1 \rrbracket}

\DeclareMathOperator*{\argmin}{argmin}

\DeclareMathOperator{\proj}{proj}

\DeclareMathOperator{\gmed}{GMed}

\newcommand{\diag}{\mathrm{diag}}

\newcommand{\sign}{\mathrm{sign}}

\DeclareMathOperator{\median}{median}

\newcommand{\bX}{\boldsymbol X}

\newcommand{\cX}{\mathcal X}
\newcommand{\cY}{\mathcal Y}

\newcommand{\wh}{\widehat}
\newcommand{\grad}{\nabla}

\title{Robust supervised learning with coordinate gradient descent}

\author{
St\'ephane Ga\"iffas%
\thanks{LPSM, UMR 8001, Universit\'e Paris Diderot, Paris, France and DMA, École normale supérieure}
\and
Ibrahim Merad%
\thanks{LPSM, UMR 8001, Universit\'e Paris Diderot, Paris, France}\\
}

\begin{document}

\maketitle

\begin{abstract}
    This paper considers the problem of supervised learning with linear methods when both features and labels can be corrupted, either in the form of heavy tailed data and/or corrupted rows. 
    We introduce a combination of coordinate gradient descent as a learning algorithm together with robust estimators of the partial derivatives.
    This leads to robust statistical learning methods that have a numerical complexity \emph{nearly identical} to non-robust ones based on empirical risk minimization.
    The main idea is simple: while robust learning with gradient descent requires the computational cost of robustly estimating the whole gradient to update all parameters, a parameter can be updated immediately using a robust estimator of a single partial derivative in coordinate gradient descent.
    We prove upper bounds on the generalization error of the algorithms derived from this idea, that control both the optimization and statistical errors with and without a strong convexity assumption of the risk.
    Finally, we propose an efficient implementation of this approach in a new \texttt{Python} library called \texttt{linlearn}, and demonstrate through extensive numerical experiments that our approach introduces a new interesting compromise between robustness, statistical performance and numerical efficiency for this problem.

\end{abstract}

\section{Introduction}
\label{sec:intro}

Outliers and heavy tailed data are a fundamental problem in supervised learning.
As explained by~\cite{hawkins1980identification}, an outlier is a sample that differs from the data's ``global picture''.
A rule-of-thumb is that a typical dataset may contain between 1\% and 10\% of outliers~\citep{hampel2011robust}, or even more than that depending on the considered application. 
For instance, the inherently complex and random nature of users' web browsing makes web-marketing datasets contain a significant proportion of outliers and have heavy-tailed distributions~\citep{Gupta2016}.
Statistical handling of outliers was already considered in the early 50's~\citep{dixon1950analysis, grubbs1969procedures} and motivated in the 70's the development of \emph{robust statistics}~\citep{huber19721972, huber1981wiley}.

\paragraph{Setting.}

In this paper, we consider the problem of large-scale supervised learning, where we observe possibly corrupted samples $(X_i, Y_i)_{i=1}^n$ of a random variable $(X, Y) \in \cX \times \cY$ with distribution $P$, where $\cX \subset \R^d $ is the feature space and $\cY \subset \R$ is the set of label values.
We focus on linear methods, where the learning task corresponds to finding an approximation of an optimal parameter
\begin{equation}
    \label{eq:truerisk}
    \theta^\star \in \argmin_{\theta \in \Theta} R(\theta) \quad \text{where} \quad R(\theta):= \E \big[ \ell (X^\top \theta, Y) \big],
\end{equation}
where $\Theta$ is a convex compact subset of $\R^d$ with diameter $\Delta$ containing the origin and $\ell: \R \times \cY \rightarrow \R_+$ is a loss function satisfying the following. We denote $\ell'(z, y):= \partial \ell(z, y) / \partial z$.
\begin{assumption}
\label{assump:lipsmoothloss}
The loss $z \mapsto \ell(z, y)$ is convex for any $y \in \mathcal Y,$ differentiable and $\gamma$-smooth in the sense that $|\ell'(z, y) - \ell'(z', y)| \leq \gamma |z - z'|$
for all $z, z' \in \R$ and $y \in \cY$. Moreover\textup, there exist $q\in[1, 2]$, which we will call the asymptotic polynomial degree, and positive constants $C_{\ell, 1}, C_{\ell, 2}, C_{\ell, 1}'$ and $C_{\ell, 2}'$ such that 
\begin{equation*}
    |\ell(z, y)| \leq C_{\ell, 1} + C_{\ell, 2}|z-y|^{q} \quad \text{and} \quad |\ell'(z, y)| \leq C_{\ell, 1}' + C_{\ell, 2}'|z-y|^{q-1}
\end{equation*}
for all $z\in\R$ and $y\in\cY$.
\end{assumption}
Note that Assumption~\ref{assump:lipsmoothloss} holds for the majority of loss functions used both for regression and classification, such as the square loss $\ell(z, y) = (z - y)^2 / 2$ with $q=2$ or the Huber loss~\citep{huber1964robust} $\ell(z, y) = r_\tau(z - y)$ for $z, y \in \R$ with $\gamma = 1$ and $q=1$, where $r_\tau(u) = \frac 12 u^2 \ind{|u| \leq \tau} + \tau (|u| - \frac 12 \tau) \ind{|u| > \tau}$ with $\tau > 0$ and the logistic loss $\ell(z, y) = \log(1 + e^{-yz})$ for $z \in \R$ and $y \in \{-1, 1\}$ with $\gamma = 1/4$ and $q=1$. We will see shortly that a smaller degree $q$ associated to the loss entails looser requirements on the data distribution.
If $P$ were known, one could approximate $\theta^\star$ using a first-order optimization algorithm such as \emph{gradient descent} (GD), using iterations of the form
\begin{equation}
    \label{eq:gd-known-P}
    \theta_{t+1} \gets \theta_{t} - \eta \grad R(\theta_t) \quad \text{ with } \quad 
    \grad R(\theta) = \E[\ell'(X^\top \theta, Y) X]
\end{equation}
for $t=1, 2, \ldots$ where $\eta > 0$ is a learning rate.

\paragraph{Empirical risk minimization.}

With $P$ unknown, most supervised learning algorithms rely on \emph{empirical risk minimization} (ERM)~\citep{vapnik2013nature, geer2000empirical}, which requires 
(a)~the fact that samples are independent and with the same distribution $P$ and~(b) that $P$ has sub-Gaussian tails, as explained below.
Such assumptions are hardly ever met in practice, and entail implicitly that, for real-world applications, the construction of a training dataset requires involved data preparation, such as outlier detection and removal, data normalization and other issues related to feature engineering~\citep{zheng2018feature, kuhn2019feature}.
An \emph{implicit}\footnote{By \emph{implicit}, we mean defined as the $\argmin$ of some functional, as opposed to the \emph{explicit} iterations of an optimization algorithm: an implicit estimator differs from the exact algorithm applied on the data, while an \emph{explicit} algorithm does not.}
ERM estimator of $\theta^\star$ is a minimizer of the empirical risk $R_n$ given by
\begin{equation}
\label{empricialgrad}
    \wh \theta_n^{\mathtt{erm}} \in \argmin_{\theta \in \Theta} R_n(\theta) \quad \text{where} \quad R_n(\theta):= \frac 1n \sum_{i=1}^n \ell (X_i^\top \theta, Y_i),
\end{equation}
for which one can prove sub-Gaussian deviation bounds under strong hypotheses such as boundedness of $\ell$ or sub-Gaussian concentration~\citep{massart2006risk, lecue2013learning}. 
In the general case, ERM leads to poor estimations of $\theta^\star$ whenever~(a) and/or~(b) are not met, corresponding to situations where (a)~the dataset contains outliers and~(b) the data distribution has heavy tails.
This fact motivated the theory of robust statistics~\citep{huber1964robust, huber2004robust, hampel1971, hampel2011robust, tukey1960}. 
The poor performance of ERM stems from the loose deviation bounds of the empirical mean estimator.
Indeed, as explained by~\cite{catoni2012challenging} for the estimation of the expectation of a real random variable, the Chebyshev inequality provably provides the best concentration bound for the empirical mean estimator in the general case, so that the error is $\Omega(1/\sqrt{n\delta})$ for a confidence $1 - \delta$.
Gradient Descent (GD) combined with ERM leads to an \emph{explicit} algorithm using iterations~\eqref{eq:gd-known-P} with gradients estimated by an average over the samples
\begin{equation}
    \label{eq:gradient-erm}
    \wh \grad^{\mathtt{erm}} R(\theta):= \nabla R_n(\theta) = \frac 1n \sum_{i=1}^n \ell'(X_i^\top \theta, Y_i) X_i,
\end{equation}
which is, as explained above, a poor estimator of $\grad R(\theta)$ beyond (a) and (b).

\paragraph{Robust gradient descent.}

A growing literature about robust GD estimators~\citep{HeavyTails, DBLP:journals/corr/abs-1901-08237, holland2019robust, geoffrey2020robust} suggests to perform GD iterations with $\wh \grad^{\mathtt{erm}} R(\theta)$ replaced by some robust estimator of $\grad R(\theta)$.
An implicit estimator is considered by~\cite{lecue2020robust1}, based on the minimization of a robust estimate of the risk objective using median-of-means.
Robust estimators of $\grad R(\theta)$ can be built using several approaches including geometric median-of-means~\citep{HeavyTails}; robust coordinate-wise estimators~\citep{pmlr-v97-holland19a} based on a modification of~\cite{catoni2012challenging}; coordinate-wise median-of-means or trimmed means~\citep{DBLP:journals/corr/abs-1901-08237} or robust vector means through projection and truncation~\citep{HeavyTails}.
Other works achieve robustness by performing standard training on disjoint subsets of data and aggregating the resulting estimators into a robust one~\citep{minsker2015geometric, brownlees2015empirical}. 
We discuss such alternative methods in more details in Section~\ref{sec:relatedworks} below.

These procedures based on GD require to run \emph{costly} subroutines (at the exception of~\cite{lecue2020robust1, geoffrey2020robust}) that induce a considerable computational overhead compared to the non-robust approach based on ERM.
The aim of this paper is to introduce \emph{robust} and \emph{explicit} learning algorithms, with performance guarantees under weak assumptions on $(X_i, Y_i)_{i=1}^n$, that have a computational cost \emph{comparable} to the non-robust ERM approach.
As explained in Section~\ref{sec:robustcgd} below, the main idea is to combine \emph{coordinate gradient descent} with robust estimators of the partial derivatives $\partial R(\theta) / \partial \theta_j$, that are scalar (univariate) functionals of the unknown distribution $P$.

We denote $|A|$ as the cardinality of a finite set $A$ and use the notation $\setint k = \{ 1, \ldots, k\}$ for any integer $k \in \N \setminus \{ 0 \}$. 
We denote $x^{j}$ as the $j$-th coordinate of a vector $x$.
We will work under the following assumption.

\begin{assumption}
\label{assump:data}
The indices of the training samples $\setint n$ can be divided into two disjoint subsets $\setint n = \mathcal{I} \cup \mathcal{O}$ of \emph{outliers} $\mathcal O$ and \emph{inliers} $\mathcal I$ for which  
we assume the following\textup:~$(a)$ we have $|\mathcal I| > |\mathcal O|;$~$(b)$ the pairs $(X_i, Y_i)_{i \in \mathcal I}$ are i.i.d with distribution $P$ and the outliers $(X_i, Y_i)_{i \in \mathcal O}$ are arbitrary\textup;~$(c)$ there is $\alpha \in (0, 1]$ such that
\begin{equation}
    \label{eq:ass-X-moments}
    \E\big[ |X^j|^{\max(2, q(\alpha + 1))} \big] < +\infty, \quad \E\big[|Y^{q-1} X^j|^{1 + \alpha}\big] < +\infty \quad
    \text{and} \quad \E\big[|Y|^q\big] < +\infty
\end{equation}
for any $j \in \setint d$ where $q \in [1, 2]$ is the loss' asymptotic polynomial degree from Assumption~\ref{assump:lipsmoothloss}.
\end{assumption}

Assumption~\ref{assump:data} is purposely vague about $|\mathcal{I}|$ and $|\mathcal{O}|$ and the value of $\alpha \in (0, 1]$.
Indeed, conditions on $|\mathcal{O}|$ and $\alpha$ will depend on the considered robust estimator of the partial derivatives, as explained in Section~\ref{sec:robust-estimators} below, including theoretical guarantees with $\alpha < 1$ and cases with $\E [Y^2] = +\infty$ (for the Huber loss for instance). The existence of a second moment for $X$ is indispensable for the objective $R(\theta)$ to be Lipschitz-smooth, see Section~\ref{sub:theory-strong-cvx} below.

\medskip
\noindent
\emph{Square loss.}
For the square loss we have $q=2$ and $\E[Y^2] < +\infty$ is required for the risk $R(\theta)$ and its partial derivatives to be well-defined. Note that we have $\E[ | \ell'(X^\top \theta, Y) X^j |^{1+\alpha} ] = \E[ |Y X^j |^{1+\alpha} ]$ for $\theta = 0 \in \Theta$, which makes~\eqref{eq:ass-X-moments} somewhat minimal in order to ensure the existence of the moment we need for the loss derivative for all $\theta \in \Theta$. 

\medskip
\noindent
\emph{Huber loss.}
For the Huber loss, we have $q=1$ and the only requirement on $Y$ is $\E|Y| < +\infty$ and we have $\max(2, q(\alpha + 1)) = 2$ ensuring that $\E[|X^j|^2] < +\infty$, a requirement for the Lipschitz-smoothness of $R(\theta)$, as detailed in Section~\ref{sub:theory-strong-cvx}.

\medskip
\noindent
\emph{Logistic loss.} For the logistic loss we have $|Y| \leq 1$ and $q=1$ so that the only assumption is once again $\E[|X^j|^2] < +\infty$.

\paragraph{Main contributions.} 

We believe that this paper introduces a new interesting compromise between robustness, statistical performance and numerical efficiency for supervised learning with linear methods through the following main contributions:
\begin{itemize}
    \item We introduce a new approach for robust supervised learning with linear methods by combining coordinate gradient descent (CGD) with robust estimators of the partial derivatives used in its iterations (Section~\ref{sec:robustcgd}). 
    We explain that this simple idea turns out to be very effective experimentally (Section~\ref{sec:experiments}), and amenable to an in-depth theoretical analysis (see Section~\ref{sub:theory-strong-cvx} for guarantees under strong convexity and Section~\ref{sec:theory-not-strongly-cvx} without it).
    \item We consider several estimators of the partial derivatives using state-of-the-art robust estimators (Section~\ref{sec:robust-estimators}) and provide theoretical guarantees for CGD combined with each of them. For some robust estimators, our analysis requires only weak moments (allowing $\E[Y^2] = +\infty$ in some cases) together with strong corruption (large $|\mathcal O|$). 
    We provide guarantees for several variants of CGD namely random uniform sampling, importance sampling and deterministic sampling of the coordinates (Section~\ref{sub:theory-strong-cvx}).
    \item We perform extensive numerical experiments, both for regression and classification on several datasets (Section~\ref{sec:experiments}). 
    We compare many combinations of gradient descent, coordinate gradient descent and robust estimators of the gradients and partial derivatives. 
    Some of these combinations correspond to state-of-the-art algorithms~\citep{lecue2020robust2, pmlr-v97-holland19a, HeavyTails}, and we consider also several supplementary baselines such as Huber regression~\citep{owen2007}, classification with the modified Huber loss~\citep{10.1145/1015330.1015332}, Least Absolute Deviation (LAD)~\citep{10.2307/23036355} and RANSAC~\citep{10.1145/358669.358692}.
    Our experiments provide comparisons of the statistical performances and numerical complexities involved in each algorithm, leading to an in-depth comparison of state-of-the-art robust methods for supervised linear learning.
    \item All the algorithms studied and compared in the paper are made easily accessible in a few lines of code through a new \texttt{Python} library called \texttt{linlearn}, open-sourced under the BSD-3 License on \texttt{GitHub} and available here\footnote{\url{https://github.com/linlearn/linlearn}}.
    This library follows the API conventions of \texttt{scikit-learn}~\citep{pedregosa2011scikit-learn}.
\end{itemize}

\section{Robust coordinate gradient descent}
\label{sec:robustcgd}

CGD is well-known for its efficiency and fast convergence properties based on both theoretical and practical studies~\citep{nesterov2012efficiency, Shevade2003ASA, genkin2007large, wuLange} and is the de-facto standard optimization algorithm used in many machine learning libraries.
In this paper, we suggest to use CGD with robust estimators $\wh g_j(\theta)$ of the partial derivatives $g_j(\theta):= \partial R(\theta) / \partial \theta_j \in \R$ of the true risk given by Equation~\eqref{eq:truerisk}, 
several robust estimators $\wh g_j(\theta)$ are described in Section~\ref{sec:robust-estimators} below.

\subsection{Iterations} 
\label{sub:coordinate_gradient_descent}

At iteration $t+1$, given the current iterate $\theta^{(t)}$, CGD proceeds as follows. 
It chooses a coordinate $j_t \in \setint d$ (several sampling mechanisms are possible, as explained below) and the parameter is updated using
\begin{equation}
    \label{eq:iter-strgly-cvx}
    \begin{cases}
    \theta_{j}^{(t+1)} \gets \theta_j^{(t)} - \beta_j \wh g_j(\theta^{(t)}) & \text{ if } j = j_t \\
    \theta_{j}^{(t+1)} \gets \theta_j^{(t)} & \text{ otherwise}
    \end{cases}
\end{equation}
for all $j \in \setint d$, where $\beta_j > 0$ is a step-size for coordinate $j$.
A \emph{single} coordinate is updated at each iteration of CGD, and we will designate $d$ iterations of CGD as a \emph{cycle}.
The CGD procedure is summarized in Algorithm~\ref{alg:robust-cgd} below, where we denote by $\bX \in \R^{n \times d}$ the features matrix with rows $X_1^\top, \ldots, X_n^\top$ and where $\bX_{\bullet}^j \in \R^n$ stands for its $j$-th column.
\begin{algorithm}[!ht]
    \caption{Robust coordinate gradient descent}
    \label{alg:robust-cgd}
    \begin{algorithmic}[1]
      \STATE \textbf{Inputs:} Learning rates $\beta_1, \ldots, \beta_d > 0$; estimators $(\wh g_j(\cdot))_{j=1}^d$ of the the partial derivatives; initial parameter $\theta^{(0)};$  distribution $p = [p_1 \cdots p_d]$ over $\setint d$ and number of iterations $T$.
      \STATE Compute $I^{(0)} \gets \bX \theta^{(0)}$
      \FOR{$t=0, \ldots, T-1$}
      \STATE Sample a coordinate $j_t \in \{1, \ldots, d\}$ with distribution $p$ independently of $j_1, \ldots, j_{t-1}$
      \STATE Compute $\wh g_{j_t}(\theta^{(t)})$ using $I^{(t)}$ and put $D^{(t)} \gets - \beta_{j_t} \wh g_{j_t}(\theta^{(t)})$
      \STATE Update the inner products using $I^{(t+1)} \gets I^{(t)} + \bX_{\bullet}^{j_t} D^{(t)}$ 
      \STATE Apply the update $\theta_{j_t}^{(t+1)} \gets \theta_{j_t}^{(t)} + D^{(t)}$
      \ENDFOR
      \RETURN {The last iterate $\theta^{(T)}$}
    \end{algorithmic}
\end{algorithm}

\noindent

A simple choice for the distribution $p$ is the uniform distribution over $\setint d$, but improved convergence rates can be achieved using importance sampling, as explained in Theorem~\ref{thm:linconv1expect} below, where the choice of the step-sizes $(\beta_j)_{j=1}^d$ is described as well.
The partial derivatives estimators $(\wh g_j(\cdot))_{j=1}^d$ described in Section~\ref{sec:robust-estimators} will determine the statistical error of this explicit learning procedure.
Note that line~6 of Algorithm~\ref{alg:robust-cgd} uses the fact that
\begin{align*}
    I^{(t+1)} = \bX \theta^{(t+1)} &= \sum_{j \neq j_t} \bX_{\bullet}^j \theta_{j}^{(t+1)} + \bX_{\bullet}^{j_t} \theta_{j_t}^{(t+1)} \\ 
    &= \sum_{j \neq j_t} \bX_{\bullet}^j \theta_{j}^{(t)} + \bX_{\bullet}^{j_t} \big(\theta_{j_t}^{(t)} + D^{(t)}\big) = I^{(t)} + \bX_{\bullet}^{j_t} D^{(t)}.
\end{align*}
This computation has complexity $O(n)$, and we will see in Section~\ref{sec:robust-estimators} that the complexity of the considered robust estimators $\wh g_{j_t}(\theta^{(t)})$ at line~5 is also $O(n)$, so that the overall complexity of one iteration of robust CGD is also $O(n)$.
This makes the complexity of one cycle of robust CGD $O(n d)$, which corresponds to the complexity of \emph{one iteration of GD using the non-robust estimator} $\wh \grad^{\mathtt{erm}} R(\theta)$, see Equation~\eqref{eq:gradient-erm}.
A more precise study of these complexities is discussed in Section~\ref{sec:robust-estimators}, see in particular Table~\ref{table:1}.
Moreover, we will see experimentally in Section~\ref{sec:experiments} that our approach is indeed very competitive computationally in terms of the compromise between computations and statistical accuracy, compared to all the considered baselines.

\paragraph{Comparison with robust gradient descent.} 
\label{par:comparison_with_robust_gradient_descent}

Robust estimators of the expectation of a random vector (such as the geometric median by~\cite{minsker2015geometric}) require to solve a $d$-dimensional optimization problem at each iteration step while, in the univariate case, a robust estimator of the expectation can be obtained at a cost comparable to that of an ordinary empirical average. 
Of course, one can combine such univariate estimators into a full gradient: this is the approach considered for instance by~\cite{pmlr-v97-holland19a, holland2019robust, holland2019efficient, DBLP:journals/corr/abs-1901-08237, tu2021variance}, but this approach accumulates errors into the overall estimation of the gradient.
This paper introduces an alternative approach, where univariate estimators of the partial derivatives are used \emph{immediately} to update the current iterate.
We believe that this is the main benefit of using CGD in this context: even if our theoretical analysis hardly explains this, our understanding is that one iteration of CGD is impacted by the estimator error of a \emph{single} partial derivative, that can be corrected straight away in the next iteration, while one iteration of GD is impacted by the accumulated estimation errors of the $d$ partial derivatives, when using $d$ univariate estimators for efficiency, instead of a computationally involved $d$-dimensional estimator (such as geometric median).

\subsection{Theoretical guarantees under strong convexity}
\label{sub:theory-strong-cvx}

In this Section, we provide theoretical guarantees in the form of upper bounds on the risk $R(\theta^{(T)})$ (see Equation~\eqref{eq:truerisk}) for the output $\theta^{(T)}$ of Algorithm~\ref{alg:robust-cgd}.
These upper bounds are generic with respect to the considered robust estimators $\big(\wh g_j(\cdot)\big)_{j=1}^d$ and rely on the following definition.

\begin{definition}
    \label{def:error-vector}
    Let $\delta \in (0, 1)$ be a failure probability. 
    We say that a partial derivatives estimator $\widehat{g}$ has an \emph{error vector} $\epsilon (\delta) \in \R_+^d$ if it satisfies
    \begin{equation}
    \label{eq:uniform-bound-prototype}
        \Proba \Big[ \sup_{\theta\in\Theta}\big| \widehat{g}_j(\theta) - g_j(\theta) \big| \leq \epsilon_j(\delta) \Big] \geq 1 - \delta
    \end{equation}    
    for all $j\in\setint d$.
\end{definition}
In Section~\ref{sec:robust-estimators} below, we specify a value of $\epsilon_j(\delta)$ for each considered robust estimator which will lead to upper bounds on the risk.
Recall that $g_j(\theta) = \partial R(\theta) / \partial \theta_j$ and let us denote as $e_j$ the $j$-th canonical basis vector of $\R^d$.
We need the following extra assumptions on the optimization problem itself.
\begin{assumption}
    \label{ass:mintheta-and-smoothness}
    There exists $\theta^\star \in\Theta$ satisfying the stationary gradient condition $\grad R(\theta^\star) = 0$.
    Moreover, we assume that there are Lipschitz constants $L_j > 0$ such that
    \begin{equation*}
        \big|g_j(\theta + h e_j) - g_j(\theta)\big| \leq L_j |h|
    \end{equation*}
    for any $j \in \setint d$, $h \in \R$ and $\theta \in \Theta$ such that $\theta + h e_j\in\Theta$. We also consider $L > 0$ such that
    \begin{equation*}
        \big\| g(\theta + h) - g(\theta) \big\| \leq L \| h\|
    \end{equation*}
    for any $h \in \Theta$ and $\theta \in \Theta$ such that $\theta+h\in\Theta$. We denote $L_{\max}:= \max_{j \in \setint d} L_j$ and $L_{\min}:= \min_{j \in \setint d} L_j$.
\end{assumption}

Under Assumptions~\ref{assump:lipsmoothloss} and~\ref{assump:data}, we know that the Lipschitz constants $(L_j)_{j \in \setint d}$ and $L$ do exist. 
Indeed, the Hessian matrix of the risk $R(\theta)$ is given by
\begin{equation*}
    \grad^2 R(\theta) = \E \big[ \ell '' (\theta^\top X, Y)X X^\top \big],
\end{equation*}
where  $\ell''(z, y):= \partial^2 \ell(z, y) / \partial z^2$, so that 
\begin{equation}\label{eq:lipschitz_constants}
    L_j = \sup_{\theta \in \Theta} \E \big[ \ell''(\theta^\top X, Y) (X^j)^2\big] \quad \text{and} \quad
    L = \sup_{\theta \in \Theta} \big\| \grad^2 R(\theta) \big\|_{\mathrm{op}},
\end{equation}
where $\big\| H \big\|_{\mathrm{op}}$ stands for the operator norm of a matrix $H$.
Assumption~\ref{assump:lipsmoothloss} entails $L_j \leq \gamma \E \big[(X^j)^2\big]$, which is finite because of Equation~\eqref{eq:ass-X-moments} from Assumption~\ref{assump:data}.
In order to derive \emph{linear} convergence rates for CGD, it is standard to require strong convexity~\citep{nesterov2012efficiency, wright2015coordinate}.
Here, we require strong convexity on the risk $R(\theta)$ itself, as described in the following.

\begin{assumption}
    \label{ass:strongconvexity}
    We assume that the risk $R$ given by Equation~\eqref{eq:truerisk} is $\lambda$-strongly convex\textup, namely that
    \begin{equation}
        \label{eq:strongconvexity}
        R(\theta_2) \geq R(\theta_1) + \langle \nabla R(\theta_1), \theta_2 - \theta_1 \rangle
        + \frac{\lambda}{2}\|\theta_2 - \theta_1\|^2
    \end{equation}
    for any $\theta_1, \theta_2\in \Theta$.
\end{assumption}

Assumption~\ref{ass:strongconvexity} is satisfied whenever $\lambda_{\min}\big(\grad^2 R(\theta)\big) \geq \lambda$ for any $\theta \in \Theta$, where $\lambda_{\min}(H)$ stands for the smallest eigenvalue of a symmetric matrix $H$.
For the least-squares loss, this translates into the condition $\lambda_{\min}\big(\E [X X^\top] \big) \geq \lambda$.
Note that one can always make the risk $\lambda$-strongly convex by considering ridge penalization, namely by replacing $R(\theta)$ by $R(\theta) + \frac{\lambda}{2}\| \theta \|_2^2$, but we provide also guarantees without this Assumption in Section~\ref{sec:theory-not-strongly-cvx} below.
The following Theorem provides an upper bound over the risk of Algorithm~\ref{alg:robust-cgd} whenever the estimators $\wh g_j(\cdot)$ have an error vector $\epsilon(\delta)$, as defined in Definition~\ref{def:error-vector}.
We introduce for short $R^\star = R(\theta^\star) = \min_{\theta \in \Theta} R(\theta)$.

\begin{theorem}
    \label{thm:linconv1expect}
    Grant Assumptions~\ref{assump:lipsmoothloss}\textup,~\ref{ass:mintheta-and-smoothness} and~\ref{ass:strongconvexity}. 
    Let $\theta^{(T)}$ be the output of Algorithm~\ref{alg:robust-cgd} with step-sizes $\beta_j = 1 /  L_j,$ an initial iterate $\theta^{(0)},$ uniform coordinates sampling $p_j = 1 / d$ and estimators of the partial derivatives with error vector $\epsilon(\cdot)$. 
    Then\textup, we have 
    \begin{equation}
        \label{eq:thm1-uniform}
        \E \big[ R(\theta^{(T)}) \big] - R^\star \leq 
        \big(R(\theta^{(0)}) - R^\star\big) \Big( 1 - \frac{\lambda}{L_{\max}d} \Big)^T 
        + \frac{L_{\max}}{2\lambda L_{\min} } \big\| \epsilon ( \delta ) \big\|_2^2
    \end{equation}
    with probability at least $1 - \delta$, where the expectation is w.r.t. the sampling of the coordinates.
    Now\textup, if Algorithm~\ref{alg:robust-cgd} is run as before\textup, but with an \emph{importance sampling} distribution $p_j = L_j / \sum_{k \in  \setint d} L_{k},$ we have
    \begin{equation}
        \label{eq:thm1-importance-sampling}
        \E \big[ R(\theta^{(T)}) ] - R^\star \leq \big(R(\theta^{(0)}) - R^\star\big)
        \Big( 1 - \frac{\lambda}{\sum_{j \in  \setint d} L_{j}} \Big)^T 
        + \frac{1}{2\lambda } \big\| \epsilon ( \delta ) \big\|_2^2
    \end{equation}
     with probability at least $1 - \delta$.
\end{theorem}

The proof of Theorem~\ref{thm:linconv1expect} is given in Appendix~\ref{sec:proofs}.
It adapts standard arguments for the analysis of CGD~\citep{nesterov2012efficiency,wright2015coordinate} with inexact estimators of the partial derivatives.
The statistical error $\| \epsilon( \delta ) \|_2^2$ is studied in Section~\ref{sec:robust-estimators} for each considered robust estimator of the partial derivatives.
Both~\eqref{eq:thm1-uniform} and~\eqref{eq:thm1-importance-sampling} are upper bounds on the excess risk with exponentially vanishing optimization errors (called \emph{linear} rate in optimization) and a constant statistical error.
The optimization error term of~\eqref{eq:thm1-importance-sampling}, given by
\begin{equation*}
    \big(R(\theta^{(0)}) - R^\star\big) \Big( 1 - \frac{\lambda}{\sum_{j \in  \setint d} L_{j}} \Big)^T,
\end{equation*}
goes to $0$ exponentially fast as the number of iterations $T$ increases, with a contraction constant better than that of~\eqref{eq:thm1-uniform} since $\sum_{j \in  \setint d} L_{j} \leq d L_{\max}$.
This can be understood from the fact that importance sampling better exploits the knowledge of the Lipschitz constants $L_j$.
Also, note that $T$ is the number of iterations of CGD, so that $T = C d$ where $C$ is the number of CGD cycles.
Therefore, defining $ L':= \frac 1d \sum_{j \in  \setint d} L_j $, we have
\begin{equation*}
    \Big( 1 - \frac{\lambda}{d L'} \Big)^{C d} \leq \Big( 1 - \frac{\lambda}{L'} \Big)^{C},
\end{equation*}
for $d \geq 1$, which leads to a linear rate at least similar to the one of GD~\citep{bubeck2015convex}.

Theorem~\ref{thm:linconv1expect} proves an upper bound on the excess risk $R(\theta^{(T)}) - R^\star$ of the iterates of robust CGD directly, without using an intermediate upper bound on $\|\theta^{(T)} - \theta^\star\|_2^2$.
This differs from the approaches used by~\cite{HeavyTails, pmlr-v97-holland19a} that consider robust GD (while we introduce robust CGD here) to bound the excess risk of the iterates.
This allows us to obtain a better contraction factor for the optimization error and a better constant in front of the statistical error.
Note that we can derive also an upper bound on $\|\theta^{(T)} - \theta^\star\|_2^2$, see Theorem~\ref{thm:linconvparam} in Appendix~\ref{sec:proofs}.

Note that the iterations considered in Algorithm~\ref{alg:robust-cgd} do not perform a projection in $\Theta$.
Indeed, one can show that $\|\thetat - \theta^\star \|$ is also subject to a contraction and is therefore decreasing w.r.t. $t$.
Thus, if $\theta^{(0)} = 0$, iterates $\theta^{(t)}$ naturally belong to the $\ell_2$ ball of radius $2\|\theta^\star\|$.

\paragraph{Step-sizes.}

The step-sizes $\beta_j = 1 / L_j$ are unknown, since they are functionals of the unknown distribution $P$. 
So, we provide, in Appendix~\ref{sub:unknown-step-sizes}, theoretical guarantees similar to that of Theorem~\ref{thm:linconv1expect} using step-sizes $\wh \beta_j = 1 / \wh L_j$, where $\wh L_j$ is a robust estimator of the upper bound  $\overline L_j:= \gamma \E\big[(X^j)^2\big] \geq L_j$ of the Lipschitz constant $L_j$.

\paragraph{A deterministic result.} 

The previous Theorem~\ref{thm:linconv1expect} provides upper bounds on the expectation of the excess risk with respect to the sampling of the coordinates used in CGD.
In Theorem~\ref{thm:linconvdeterministic} below, we provide an upper bound similar to the one from Theorem~\ref{thm:linconv1expect}, but with a fully deterministic variant of CGD, where we replace line~4 of Algorithm~\ref{alg:robust-cgd} with a deterministic cycling through the coordinates.

\begin{theorem}
    \label{thm:linconvdeterministic}
    Grant Assumptions~\ref{assump:lipsmoothloss}\textup, \ref{ass:mintheta-and-smoothness} and~\ref{ass:strongconvexity}. 
    Let $\theta^{(T)}$ be the output of Algorithm~\ref{alg:robust-cgd} with step-sizes $\beta_j = 1 /  L_j,$ an initial iterate $\theta^{(0)},$ deterministic cycling over $ \setint d$ such that
    \begin{equation*}
        \{ j_{td + 1}, j_{td + 2}, \dots, j_{(t + 1)d - 1} \} =  \setint d   
    \end{equation*}
    for any $t$ and estimators of the partial derivatives with error vector $\epsilon(\cdot)$.
    Then\textup, we have 
    \begin{equation*}
        R(\theta^{(T)}) - R^\star \leq \big(R(\theta^{(0)}) - R^\star\big) \big(1 - 2 \lambda \kappa\big)^T 
        + \frac{3}{8\lambda \kappa L_{\min} } \big\| \epsilon ( \delta ) \big\|_2^2
    \end{equation*}
    with probability at least $1 - \delta,$ where we introduced the constant
    \begin{equation*}
        \kappa = \frac{1}{8L_{\max}(1 + d(L_{\max}/L_{\min}))}.
    \end{equation*}
\end{theorem}

The proof of Theorem~\ref{thm:linconvdeterministic} is given in Appendix~\ref{sec:proofs} and uses arguments from~\cite{doi:10.1137/120887679} and~\cite{li2017faster}.
It provides an extra guarantee on the convergence of CGD, for a very general choice of coordinates cycling, at the cost of degraded constants compared to Theorem~\ref{thm:linconv1expect}, both for the optimization and statistical error terms.

\section{Robust estimators of the partial derivatives} 
\label{sec:robust-estimators}

We consider three estimators of the partial derivatives
\begin{equation*}
    g_j(\theta) = \frac{\partial R(\theta)}{\partial \theta_j} = \E\big[ \ell'(X^\top \theta, Y) X^j \big]
\end{equation*}
that can be used within Algorithm~\ref{alg:robust-cgd}: Median-of-Means in Section~\ref{sub:median-of-means}, Trimmed mean in Section~\ref{sub:trimmed-mean} and an estimator that we will call ``Catoni-Holland'' in Section~\ref{sub:catoni-holland}.
We provide, for each estimator, a concentration inequality for the estimation of $g_j(\theta)$ for fixed $\theta$ under a weak moments assumption (Lemmas~\ref{lem:basicMOM},~\ref{lem:basicTMean} and~\ref{lem:basicCH}). 
We derive also uniform versions of the bounds in each case (Propositions~\ref{prop:uniformMOM},~\ref{prop:RademacherMOM},~\ref{prop:uniformTM} and~\ref{prop:uniformCH}) which define the error vectors to be plugged into Theorems~\ref{thm:linconv1expect} and~\ref{thm:linconvdeterministic}.
We also discuss in details the numerical complexity of each estimator and explain that they all are, in their own way, an interpolation between the empirical mean and the median.
We wrap up these results in Table~\ref{table:1} below.

\begin{table}[H]
\centering
\begin{tabular}{l|cccc}
 & \makecell{Optimal \\ deviation bound} & \makecell{Robustness \\to outliers} & \makecell{Numerical \\ complexity} & \makecell{Hyper- \\ parameter} \\
\hline
$\mathtt{ERM}$ & No & None & $O(n)$ & None \\[5pt]
$\mathtt{MOM}$ & Yes & Yes for $|\mathcal{O}| < K/2$ & $O(n + K)$ & $K \in \setint n$ \\[5pt]
$\mathtt{CH}$ & Yes & None &  $O(n)$ & Scale $s$ \\[5pt]
$\mathtt{TM}$ & Yes & Yes for $|\mathcal{O}| < n / 8$ & $O(n)$ & Proportion $\epsilon \in [0, 1/2)$ \\
\end{tabular}
\caption{Properties of some robust estimators, where $\mathtt{ERM}=$ Empirical Risk Minimizer (ordinary mean), $\mathtt{MOM}=$ Median-of-Means, $\mathtt{CH}=$ Catoni-Holland and $\mathtt{TM}=$ Trimmed Mean. We recall that $n=$ sample size and $|\mathcal O|=$ number of outliers. The parameters of each estimators are: the number of blocks $K$ in $\mathtt{MOM}$, a scale parameter $s > 0$ in $\mathtt{CH}$ and a proportion of samples $\epsilon$ in $\mathtt{TM}$.}
\label{table:1}
\end{table}
The deviation bound optimality in Table~\ref{table:1} is meant in terms of the dependence, up to a constant, on the sample size $n$, required confidence $\delta\in (0,1)$ and distribution variance\footnote{or more generally the centered moment of order $1+\alpha$ for $\alpha \in (0,1]$, see below.}. An estimator's deviation bound is deemed optimal if it fits the lower bounds given by Theorems~1 and~3 in~\cite{lugosi2019mean}.
Let us introduce the centered moment of order $1 + \alpha$ of the partial derivatives and its maximum over $\Theta$, given by
\begin{equation}
    \label{eq:partial-derivative-moment}
    m_{\alpha, j}(\theta):= \E\Big[ \big| \ell'(X^\top\theta, Y) X^j - \E[\ell'(X^\top\theta, Y) X^j] \big|^{1+\alpha} \Big] \quad \text{and} \quad M_{\alpha, j} = \sup_{\theta \in \Theta} m_{\alpha, j}(\theta)
\end{equation}
for $\alpha \in (0, 1]$.
Note that $m_{1, j}(\theta) = \V\big[\ell'(X^\top\theta, Y) X^j\big]$ and we know that $m_{\alpha, j}(\theta)$ exists, as explained in the next Lemma.

\begin{lemma}
    \label{lem:partial-deriv-moments}
    Under Assumptions~\ref{assump:lipsmoothloss} and~\ref{assump:data} the risk $R(\theta)$ is well defined for all $\theta \in \Theta$ and we have
    \begin{equation*}
        \E\big[ \big|\ell'(X^\top \theta, Y) X^j \big|^{1 + \alpha} \big] < +\infty
    \end{equation*}
    for any $j \in  \setint d$ and $\theta \in \Theta$.
\end{lemma}

The proof of Lemma~\ref{lem:partial-deriv-moments} involves simple algebra and is provided in Appendix~\ref{sec:proofs}.
Let us introduce
\begin{equation}
    \label{eq:sample-partial-derivative}
    g^i_j(\theta):= \ell'(X_i^\top \theta, Y_i) X_{i}^j,
\end{equation}
the sample $i\in\setint{n}$ partial derivative for coordinate $j\in\setint{d}$.

\subsection{Median-of-Means}
\label{sub:median-of-means}

The Median-Of-Means ($\mathtt{MOM}$) estimator is the median
\begin{equation}
    \label{eq:mom-estimator}
    \wh g_j^{\mathtt{MOM}}(\theta):= \median \big(\wh g_j^{(1)}(\theta), \ldots, \wh g_j^{(K)}(\theta)\big)
\end{equation}
of the block-wise empirical means
\begin{equation}
    \label{eq:mom-block-mean}
    \wh g_j^{(k)}(\theta):= \frac{1}{|B_k|} \sum_{i \in B_k} g_j^i(\theta)
\end{equation}
within blocks $B_1, \ldots, B_K$ of roughly equal size that form a partition of $\setint n$ and that are sampled uniformly at random.
This estimator depends on the choice of the number $K$ of blocks used to compute it, 
which can be understood as an ``interpolation'' parameter between the ordinary mean ($K=1$) and the median ($K=n$).
It is robust to heavy-tailed data and a limited number of outliers as explained in the following lemma.

\begin{lemma}
    \label{lem:basicMOM}
    Grant Assumptions~\ref{assump:lipsmoothloss} and~\ref{assump:data} with $\alpha \in (0, 1]$. If $|\mathcal{O}| \leq K / 12,$ we have\textup:
    \begin{equation*}
    \P\bigg[ \big| \wh g_j^{\mathtt{MOM}}(\theta) - g(\theta)_j \big| > (24 m_{\alpha, j} (\theta))^{1/(1+\alpha)} \Big( \frac Kn \Big)^{\alpha / (1+\alpha)} \bigg] \leq e^{-K/18}
    \end{equation*}
    for any fixed $j \in  \setint d$ and $\theta \in \Theta$.
    If we fix a confidence level $\delta \in (0, 1)$ and choose $K:= \lceil 18 \log (1 / \delta) \rceil ,$ we have
    \begin{align}
        \nonumber
        \big| \wh g_j^{\mathtt{MOM}}(\theta) - g(\theta)_j \big| 
        &\leq c_\alpha m_{\alpha, j} (\theta)^{1 / (1+\alpha)} \Big( \frac{\log(1/\delta)}{n} \Big)^{\alpha / (1 + \alpha)} \\
        &\leq c_\alpha M_{\alpha, j}^{1 / (1+\alpha)} \Big( \frac{\log(1/\delta)}{n} \Big)^{\alpha / (1 + \alpha)}
    \end{align}
    with a probability larger than $1 - \delta$, where $c_\alpha:= 2^{(3 + \alpha) / (1 + \alpha)} 3^{(1 + 2 \alpha)/ (1 + \alpha)}.$
\end{lemma}

The proof of Lemma~\ref{lem:basicMOM} is given in Appendix~\ref{sec:proofs} and it adapts simple arguments from~\cite{lugosi2019mean} and~\cite{lecue2020robust1}.
Compared to~\cite{lugosi2019mean}, it provides additional robustness with respect to $|\mathcal O| \geq 1$ outliers and compared to~\cite{lecue2020robust1} it provides guarantees with weak moments $\alpha < 1$.
An inspection of the proof of Lemma~\ref{lem:basicMOM} shows that it holds also under the assumption $|\cO| \leq (1 - \varepsilon) K / 2$ for any $\varepsilon \in (0, 1)$ with an increased constant $c_\alpha = 8 \times 3^{1 / (1 + \alpha)} / \varepsilon^{(1 + 2\alpha) / (1 + \alpha)}$.
This concentration bound is optimal under the $(1+\alpha)$-moment assumption (see Theorems~1 and~3 in~\cite{lugosi2019mean}) and is sub-Gaussian when $\alpha=1$ (finite variance).
The next proposition provides a \emph{uniform} deviation bound over $\Theta$ for $\wh g_j^{\mathtt{MOM}}(\theta)$.

\begin{proposition}
    \label{prop:uniformMOM}
    Grant Assumptions~\ref{assump:lipsmoothloss} and~\ref{assump:data} with $\alpha \in (0, 1]$ and $|\cO|\leq K/12$. 
    We have
    \begin{equation*}
        \Proba \Big[ \sup_{\theta\in\Theta}\big| \widehat{g}_j^{\mathtt{MOM}}(\theta) - g_j(\theta) \big| \leq \epsilon_j^{\mathtt{MOM}}(\delta) \Big] \geq 1 - \delta \quad
    \end{equation*}
    for any $j \in \setint{d},$ with
    \begin{align*}
        \epsilon_j^{\mathtt{MOM}}(\delta) &:= c_{\alpha} \Big( M_{j, \alpha} + \frac{m_{L,\alpha}}{n^{\alpha}} \Big)^{1/(1+\alpha)} \Big( \frac{\log(d/\delta) + d\log(3\Delta n^{\alpha/(1+\alpha)} /2)}{n} \Big)^{\alpha/(1+\alpha)} \\
        &\quad + (\overline{L} + L_{j})\Big(\frac{1}{n}\Big)^{\alpha/(1+\alpha)}
    \end{align*}
     where $\overline{L} = \gamma \E \|X\|^2,$ $m_{L,\alpha} = \E |\gamma \|X\|^2 - \overline{L} |^{1+\alpha}$ and $c_{\alpha} = 2^{(3+2\alpha)/(1+\alpha)}3^{(1+3\alpha)/(1+\alpha)}$.
\end{proposition}

The proof of Proposition~\ref{prop:uniformMOM} is given in Appendix~\ref{sec:proofs} and uses methods similar to Lemma~\ref{lem:basicMOM} with an $\varepsilon$-net argument.
This defines the error vector $\epsilon^{\mathtt{MOM}}(\delta)$ of the $\mathtt{MOM}$ estimator of the partial derivatives in the sense of Definition~\ref{def:error-vector}, 
that can be combined directly with the convergence results from Theorems~\ref{thm:linconv1expect} and~\ref{thm:linconvdeterministic} from Section~\ref{sec:robustcgd}.
Since the optimization error decreases exponentially w.r.t. the number of iterations $T$ in these theorems, while the estimator error $\|\epsilon (\delta) \|_2$ is fixed, one only needs $T = O(\|\epsilon (\delta) \|_2)$ to make both terms of the same order.

\paragraph{About uniform bounds.}

What is necessary to obtain a control of the excess risk of robust CGD is a control of the noise terms $|\widehat{g}_j(\theta^{(t)}) - g_j(\theta^{(t)})|$, where both iterates $\theta^{(t)}$ and estimators $\widehat{g}_j(\cdot)$ of the partial derivatives depend on the same data.
This forbids the direct use of a deviation such a the one from Lemma~\ref{lem:basicMOM} (and Lemmas~\ref{lem:basicTMean} and~\ref{lem:basicCH} below) where $\theta$ must be deterministic.
We use in this paper an approach based on  uniform deviation bounds (Propositions~\ref{prop:uniformMOM},~\ref{prop:uniformTM} and~\ref{prop:uniformCH}) in order to bypass this problem, similarly to~\cite{holland2019efficient} and many other papers using empirical process theory. 
This is of course pessimistic, since $\theta^{(t)}$ goes to $\theta^\star$ as $t$ increases.
Another approach considered in~\cite{HeavyTails} is to split data into segments of size $n/T$ and to compute the gradient estimator using a segment independent of the ones used to compute the current iterate.
This approach departs strongly from what is actually done in practice, and leads to controls on the excess risk expressed with $\widetilde{\delta} = \delta/T$ and $\widetilde{n} = n/T$ instead of $\delta$ and $n$, hence a deterioration of the control of the excess risk.
Our approach based on uniform deviations also suffers from a deterioration, due to the use of an $\varepsilon$-net argument, observed in Proposition~\ref{prop:uniformMOM} through the extra $d^{\alpha/(1+\alpha)}$ factor when compared to Lemma~\ref{lem:basicMOM}.
Avoiding such deteriorations is an open difficult problem, either using uniform bounds or data splitting.

\medskip
\noindent
In addition to Proposition~\ref{prop:uniformMOM}, we propose another uniform deviation bound for $\wh g_j^{\mathtt{MOM}}(\theta)$ using the Rademacher complexity, which is a 
fundamental tool in statistical learning theory and empirical process theory~\citep{ledoux1991probability,koltchinskii2006local,bartlett2005local}.
Let us introduce 
\begin{equation*}
    \mathcal{R}_j(\Theta) = \E \Big[\sup_{\theta\in\Theta} \sum_{i\in\mathcal{I}} \varepsilon_i g_j^i(\theta) \Big]
\end{equation*}
for $j\in\setint{d}$, where $(\varepsilon_i)_{i \in \cI}$ are i.i.d Rademacher variables and where we recall that $\cI$ contains the inliers indices (see Assumption~\ref{assump:data}).
\begin{proposition}
    \label{prop:RademacherMOM}
    Grant Assumptions~\ref{assump:lipsmoothloss} and~\ref{assump:data} with $\alpha \in (0, 1]$. If $|\mathcal{O}| \leq K / 12,$ we have
    \begin{equation*}
        \Proba \bigg[\sup_{\theta \in\Theta} \big| \wh g_j^{\mathtt{MOM}}(\theta) - g_j(\theta)\big| \geq \max \Big(\Big(\frac{36M_{\alpha, j}}{(n/K)^\alpha}\Big)^{1/(1+\alpha)} , \frac{64 \mathcal{R}_j(\Theta)}{n}\Big)\bigg] \leq e^{-K/18}
    \end{equation*}
    for any $j \in  \setint d$.
    If we fix a confidence level $\delta \in (0, 1)$ and choose $K:= \lceil 18 \log (1 / \delta) \rceil ,$ we have
    \begin{equation}
        \label{eq:RademacherMOM}
        \sup_{\theta\in\Theta} \big| \wh g_j^{\mathtt{MOM}}(\theta) - g(\theta)_j \big| 
        \leq \max \Big(c_\alpha M_{\alpha, j}^{1/(1+\alpha)}\Big(\frac{\log(d/\delta)}{n}\Big)^{\alpha/(1+\alpha)} , \frac{64 \mathcal{R}_j(\Theta)}{n}\Big)        
    \end{equation}
     with a probability larger than $1 - \delta$ for all $j\in\setint{d},$ where $c_\alpha:= 2^{(2 + \alpha) / (1 + \alpha)} 3^{2}.$ Moreover\textup, if $\mu_{X,j}^{2(1+\alpha)}:= \E[(X^j)^{2(1+\alpha)}] < +\infty$ for all $j\in\setint{d}$ we have
    \begin{equation*}
        \mathcal{R}_j(\Theta) \leq \gamma \Delta C_{\alpha} \Big(n \mu_{X,j}^{1+\alpha} \sum_{k\in\setint{d}} \mu_{X,k}^{1+\alpha} \Big)^{1/(1+\alpha)} = O((nd)^{1/(1+\alpha)}),
    \end{equation*}
    where $C_{\alpha}$ is a constant depending only on $\alpha$.
\end{proposition}
The proof of Proposition~\ref{prop:RademacherMOM} is given in Appendix~\ref{sec:proofs} and borrows arguments from~\cite{lecue2020robust1, boucheron2013concentration}. 
For $\alpha=1$, the bound~\eqref{eq:RademacherMOM} leads to a $O(\sqrt{n d})$ bound  similar to that of Theorem~2 from~\cite{lecue2020robust1}, although we consider here a different quantity (Rademacher complexity of the partial derivatives, towards the study of the \emph{explicit} robust CGD algorithm, while \emph{implicit} algorithms are studied therein).
Note also that we do not prove similar uniform bounds using the Rademacher complexity for the $\mathtt{TM}$ and $\mathtt{CH}$ algorithms considered below, an interesting open question.

\paragraph{Comparison with~\cite{HeavyTails, pmlr-v97-holland19a}.} 

A first distinction of our results compared to~\cite{HeavyTails, pmlr-v97-holland19a} is the use and theoretical study of robust CGD instead of robust GD.
A second distinction is that we work under $1+\alpha$ moments on the partial derivatives of the risk, while~\cite{HeavyTails, pmlr-v97-holland19a} require $\alpha=1$.
Our setting is similar but more general than the one laid out in~\cite{pmlr-v97-holland19a} since the latter does not consider the presence of outliers. Theorem 5 from~\cite{pmlr-v97-holland19a} states linear convergence of the optimization error thanks to strong convexity similarly to our Theorem~\ref{thm:linconv1expect}. 
Their management of the statistical error is quite similar and leads to the same rate. 
However, our bound involves the sum of the coordinatewise moments of the gradient thanks to Proposition~\ref{prop:uniformMOM}, an improvement over the bound from~\cite{pmlr-v97-holland19a} which is only stated in terms of a uniform bound on the coordinate variances.
Another reference point is the heavy-tailed setting of~\cite{HeavyTails}, which deals with heavy-tails independently from the problem of corruption and requires $\alpha=1$.
More importantly, the approach considered in~\cite{HeavyTails} relies on data-splitting, which departs significantly from what is done in practice, while we do not perform data-spitting but use uniform bounds, as discussed above.

\paragraph{Complexity of $\wh g_j^{\mathtt{MOM}}(\theta)$.}

The computation of $\wh g_j^{\mathtt{MOM}}(\theta)$ requires (a)~to sample a permutation of $\setint n$ to sample the blocks $B_1, \ldots, B_K$, (b)~to compute averages within the blocks and (c)~to compute the median of $K$ numbers.
Sampling a permutation of $\setint n$ has complexity $O(n)$ using the Fischer-Yates algorithm~\citep{knuth1997seminumerical}, and so does the computation of the averages, so that $(a)$ and $(b)$ have complexity $O(n)$.
The computation of the median of $K$ numbers can be done using the quickselect algorithm~\citep{10.1145/366622.366647} with $O(K)$ average complexity, leading to a complexity $O(n + K) = O(n)$ since $K < n$.

\subsection{Trimmed Mean estimator}
\label{sub:trimmed-mean}

The idea of the Trimmed Mean ($\mathtt{TM}$) estimator is to exclude a proportion of data in the tails of their distribution to achieve robustness. 
We are aware of two variants: (1)~one in which samples in the tails are removed, the remaining samples being used to compute an empirical mean and (2)~another variant in which samples in the tails are clipped but not removed from the empirical mean.
Variant~(1) is robust to $\eta$-corruption\footnote{We call ``$\eta$-corruption'' the context where the outlier set $\cO$ in Assumption~\ref{assump:data} satisfies $|\cO| = \eta n$ with $\eta \in[0, 1/2)$} whenever the data distribution is sub-exponential~\citep{DBLP:journals/corr/abs-1901-08237} or sub-Gaussian~\citep{diakonikolas2019efficient, diakonikolas2019robust,diakonikolas2019sever}.
Variant~(2), also known as \emph{Winsorized mean}, enjoys a sub-Gaussian deviation~\citep{lugosi2019mean} for heavy-tailed distributions. 
 Both robustness properties are shown simultaneously (sub-Gaussian deviations under a heavy-tails assumption and $\eta$-corruption) in~\cite{lugosi2021robust} (see Theorem~1 therein).
 We consider below variant~(2), which proceeds as follows.

First, the $\mathtt{TM}$ estimator splits $\setint n = \setint{n/2} \cup \setint{n/2}^\complement$ where $\setint{n/2}^\complement = \setint n \setminus \setint{n/2}$, assuming without loss of generality that $n$ is even, and it computes the sample derivatives $g^i_j(\theta)$ given by~\eqref{eq:sample-partial-derivative} for all $i \in \setint n$.
Then, given a proportion $\epsilon \in [0, 1/2)$, it computes the $\epsilon$ and $1 - \epsilon$ quantiles of $(g^i_j(\theta))_{i \in \setint{n/2}}$ given by
\begin{equation*}
    q_\epsilon:= g^{([\epsilon n / 2])}_j(\theta) \quad \text{and} \quad  q_{1 - \epsilon}:= g^{([(1 - \epsilon) n / 2])}_j(\theta),
\end{equation*}
where $g^{(1)}_j(\theta) \leq \cdots \leq g^{(n/2)}_j(\theta)$ is the order statistics of $(g^i_j(\theta))_{i \in \setint{n/2}}$ and where $[x]$ is the lower integer part of $x \in \N$.
Finally, the estimator is computed as 
\begin{equation}
    \label{eq:tmean-estimator}
    \wh g_j^{\mathtt{TM}}(\theta) = \frac{2}{n} \sum_{i \in \setint{n/2}^\complement}
    q_\epsilon \vee g^i_j(\theta) \wedge q_{1 - \epsilon},
\end{equation}
where $a \wedge b:= \min(a, b)$ and $a \vee b:= \max(a, b)$, namely it is the average of the partial derivatives from samples in $\setint{n/2}^\complement$ clipped in the interval $[q_\epsilon, q_{1 - \epsilon}]$.
Note that $\wh g_j^{\mathtt{TM}}(\theta)$ is also some form of ``interpolation'' between the average and the median through $\epsilon$: it is the average of the partial derivatives for $\epsilon=0$ and their median for $\epsilon=1/2$.
As explained in the next lemma, the $\mathtt{TM}$ estimator is robust both to a proportion of corrupted samples and heavy-tailed data.

\begin{lemma}
\label{lem:basicTMean}
Grant Assumptions~\ref{assump:lipsmoothloss} and~\ref{assump:data} with $\alpha \in (0, 1]$
and assume that $|\cO| \leq \eta n$ with $\eta < 1/8$.
If we fix a confidence level $\delta \in (0, 1)$ and choose $\epsilon = 8\eta + 12 \log(4/\delta) / n,$  we have
\begin{align*}
    | \wh g_j^{\mathtt{TM}}(\theta) - g_j(\theta) | &\leq 7 m_{\alpha, j}(\theta)^{1 / (1 + \alpha)} \Big(4 \eta + \frac{6 \log (4/\delta)}{n} \Big)^{\alpha / (1 + \alpha)} \\
    &\leq 7 M_{\alpha, j}^{1 / (1 + \alpha)} \Big(4 \eta + \frac{6 \log (4/\delta)}{n} \Big)^{\alpha / (1 + \alpha)}
\end{align*}
with a probability larger than $1 - \delta$.
\end{lemma}

The proof of Lemma~\ref{lem:basicTMean} is given in Appendix~\ref{sec:proofs} and extends  Theorem~1 from~\cite{lugosi2021robust} to $\alpha \in (0, 1]$ instead of $\alpha = 1$ only.
It shows that the $\mathtt{TM}$ estimator has the remarkable quality of being simultaneously robust to heavy-tailed and a \emph{fraction} of corrupted data, as opposed to $\mathtt{MOM}$ which is only robust to a limited \emph{number} of outliers.
Note that for the computation of the $\mathtt{TM}$ estimator, the splitting $\setint n = \setint{n/2} \cup \setint{n/2}^\complement$ is a technical theoretical requirement used to induce independence between $q_\epsilon, q_{1 - \epsilon}$ and the sample partial derivatives $(g^i_j(\theta))_{i \in \setint{n/2}^\complement}$ involved in the average~\eqref{eq:tmean-estimator}.
Our implementation does not use this splitting.

\paragraph{Comparison with~\cite{HeavyTails}.}

A comparison between Lemma~\ref{lem:basicTMean} and the results by~\cite{HeavyTails} pertaining to the corrupted setting is relevant here. 
We first point out that corruption in~\cite{HeavyTails} is modeled as receiving data from the ``$\eta$-contaminated'' distribution $(1 - \eta)P + \eta Q$ with $Q$ an arbitrary distribution.
On the other hand, Lemma~\ref{lem:basicTMean} considers the more general $\eta$-corrupted setting where an $\eta$-proportion of the data is replaced by arbitrary outliers \emph{after} sampling. 
In this case, Lemma~\ref{lem:basicTMean} results in a statistical error with a dependence of order $\sqrt{\eta d}$ in the corruption (on the vector euclidean norm). 
On the other hand, Lemma~1 in~\cite{HeavyTails} yields a better dependence of order $\sqrt{\eta \log d}$ in the corresponding case. Keep in mind, however, that Algorithm 2 from~\cite{HeavyTails} which achieves this rate requires recursive SVD decompositions to compute a robust gradient making it computationally heavy and impractical for moderately high dimension.
Additionally, the relevant results in~\cite{HeavyTails} require a stronger moment assumption on the gradient and impose additional constraints on the corruption rate $\eta$.
We also mention Algorithm~5 from~\cite{HeavyTails} which yields an even better dependence on the dimension (see their Lemma~2), although it involves a computationally costly procedure as well. 
Besides, knowledge of the trace and operator norm of the covariance matrix of the estimated vector is required which makes the algorithm more difficult to use in practice.

\begin{proposition}
    \label{prop:uniformTM}
    Grant Assumptions~\ref{assump:lipsmoothloss} and~\ref{assump:data} with $\alpha \in (0, 1]$ and $|\cO|\leq \eta n$.
    We have
    \begin{equation*}
        \Proba \Big[ \sup_{\theta\in\Theta}\big| \widehat{g}_j^{\mathtt{TM}}(\theta) - g_j(\theta) \big| \leq \epsilon_j^{\mathtt{TM}}(\delta) \Big] \geq 1 - \delta \quad
    \end{equation*}
    for any $j\in\setint{d}$ with
    \begin{align*}
    \epsilon_j^{\mathtt{TM}}(\delta) &:= 28 \Big(M_{j, \alpha} + \frac{m_{L,\alpha}}{n^{\alpha(1+\alpha)}}\Big)^{1/(1+\alpha)} 
    \Big(2\eta +3 \frac{\log(4d/\delta) + d\log(3\Delta n^{\alpha/(1+\alpha)}/2)}{n}\Big)^{\alpha/(1+\alpha)} \\
     & \quad + \frac{\overline{L} + L_{j}}{n^{\alpha/(1+\alpha)}}
    \end{align*}    
    where $\overline{L}$ and $m_{L,\alpha}$ are as in Proposition~\ref{prop:uniformMOM}.
\end{proposition}

The proof of Proposition~\ref{prop:uniformTM} is given in Appendix~\ref{sec:proofs} and uses an $\varepsilon$-net argument to obtain a uniform bound.
The error vector $\epsilon^{\mathtt{TM}}(\delta)$ can be plugged into Theorem~\ref{thm:linconv1expect} for example.
Similarly to $\mathtt{MOM}$, the resulting statistical error has optimal dependence on the $(1+\alpha)$-moments of the partial derivatives~\eqref{eq:partial-derivative-moment}. 

\paragraph{Complexity of $\wh g_j^{\mathtt{TM}}(\theta)$.}

The most demanding part for the computation of $\wh g_j^{\mathtt{TM}}(\theta)$ is the computation of $q_\epsilon$ and $q_{1 - \epsilon}$.
A naive idea is to sort all $n$ values at an average cost $O(n \log n)$ with quicksort for example~\citep{10.1145/366622.366647} and to simply retrieve the desired order statistics afterwards.
Of course, better approaches are possible, including the median-of-medians algorithm (not to be confused with $\mathtt{MOM}$), which remarkably manages to keep the cost of finding an order statistic with complexity $O(n)$ even in the worst case (see for instance Chapter~9 of~\cite{cormen2009introduction}). 
However, the constant hidden in the previous big-O notations seriously impact performances in real-world implementations: we compared several implementations experimentally and concluded that a variant of the quickselect algorithm~\citep{10.1145/366622.366647} was the fastest for this problem.

\subsection{Catoni-Holland estimator}
\label{sub:catoni-holland}

This estimator is a variation of the robust mean estimator by~\cite{catoni2012challenging} introduced by~\cite{pmlr-v97-holland19a} for robust statistical learning, hence the name ``Catoni-Holland'', that we will denote $\wh g_j^{\mathtt{CH}}(\theta)$.
It is defined as an M-estimator which consists in solving
\begin{equation}
    \label{eq:meanMestimator}
    \sum_{i=1}^n \psi \Big(\frac{g_j^i(\theta) - \zeta}{\wh s_j(\theta)} \Big) = 0
\end{equation}
with respect to $\zeta$, where $\psi$ is an uneven function satisfying $\psi(0) = 0$, $\psi(x) \sim x$ when $x \sim 0$ and $\psi(x) = o(x)$ when $x \to +\infty$ and where $\wh s_j(\theta) > 0$ is a scale estimator. 
An approximate solution can be found using the fixed-point iterations
\begin{equation*}
    \zeta_{k+1} = \zeta_k + \frac{\wh s_j(\theta)}{n} \sum_{i=1}^n \psi\Big( \frac{g_j^i(\theta) - \zeta_k}{\wh s_j(\theta)} \Big),
\end{equation*}
which can easily be shown to converge to the desired value thanks to the monotonicity and Lipschitz-property of $\psi$. 
Following~\cite{pmlr-v97-holland19a}, we use the function $\psi(x) = 2 \arctan(\exp(x)) - \pi/2$, while functions satisfying $-\log(1 - x +x^2/2) \leq \psi(x) \leq \log(1 + x + x^2 / 2)$ are considered in~\cite{catoni2012challenging}.
As explained in~\cite{pmlr-v97-holland19a}, the scale estimator is given by 
\begin{equation}
    \label{eq:ch-scale-estimator}
    \wh s_j(\theta):= \wh \sigma_j(\theta) \sqrt{ \frac{n}{2 \log(4 / \delta)} },
\end{equation}
for a confidence level $\delta \in (0, 1)$, where $\wh \sigma_j(\theta)$ is an estimator of the standard deviation of the partial derivative $\sigma_j(\theta):= m_{1, j}(\theta)^{1/2} = \V[\ell'(X^\top \theta, Y) X^j]^{1/2}$, see~\eqref{eq:partial-derivative-moment}.
The estimator $\wh \sigma_j(\theta)$ is defined through another M-estimator solution to
\begin{equation}
    \label{eq:scaleMestimator}
    \sum_{i=1}^n \chi \Big( \frac{g_j^i(\theta) - \bar g_j(\theta)}{\sigma} \Big) = 0
\end{equation}
with respect to $\sigma$, where $\bar g_j(\theta) = \frac 1n \sum_{i=1}^n g_j^i(\theta)$ and $\chi$ is an even function satisfying $\chi(0) < 0$ and $\chi(x) > 0$ as $x \to +\infty$.
We use the same function as in~\cite{pmlr-v97-holland19a} given by $\chi(u) = u^2 / (1 + u^2) - c$ where $c$ is such that $\E \chi(Z) = 0$ for $Z$ a standard Gaussian random variable.
To compute $\wh \sigma_j(\theta)$ we use also fixed-point iterations
\begin{equation}
    \label{eq:iterscaleMestimator}
    \sigma_{k+1} = \sigma_{k} \Big(1 - \frac{\chi(0)}{n}\sum_{i=1}^n \chi \Big( \frac{g_j^i(\theta) - \bar g_j(\theta)}{\sigma_{k}} \Big) \Big).
\end{equation}
We refer to the supplementary material of~\cite{pmlr-v97-holland19a} for further details on this procedure.

The $\mathtt{CH}$ estimator can be understood, once again, as an interpolation between the average and the median of the partial derivatives.
Indeed, whenever $s$ is large, the function $\psi(\cdot / s)$ is close to the $\sign$ function, which, if used in~\eqref{eq:meanMestimator}, leads to an $M$-estimator corresponding to the median~\citep{van2000asymptotic}.
For $s$ small, $\psi(\cdot / s)$ is close to the identity, so that minimizing~\eqref{eq:meanMestimator} leads to an ordinary average.
As explained in the next lemma, this estimator is robust to heavy-tailed data (with $\alpha=1$).

\begin{lemma}
    \label{lem:basicCH}
    Grant Assumptions~\ref{assump:lipsmoothloss} and~\ref{assump:data} with $\alpha = 1$ and assume that $\cO = \emptyset$ \textup(no outliers\textup).
    For some failure probability $\delta > 0,$ assume that we have\textup, with probability at least $ 1 - \delta/2,$ that $\sigma_j(\theta) / C' \leq \wh \sigma_j(\theta) \leq C' \sigma_j(\theta)$ for some constant $C' > 1$.
    Then\textup, we have
    \begin{equation*}
        |\widehat{g}^{\mathtt{CH}}_j(\theta) - g_j(\theta)| 
        \leq C' \sigma_j(\theta) \sqrt{\frac{8\log(4/\delta)}{n}} 
        \leq C' \Sigma_j \sqrt{\frac{8\log(4/\delta)}{n}} 
    \end{equation*}
    with probability at least $1 - \delta$, where $\Sigma_j = M_{1, j} = \sup_{\theta \in \Theta} \sigma_j(\theta)$.
\end{lemma}
The proof of Lemma~\ref{lem:basicCH} is given in Appendix~\ref{sec:proofs} and is an almost direct application of the deviation bound from~\cite{pmlr-v97-holland19a}. 
If $C' \approx 1$, the deviation bound of $\widehat{g}^{\mathtt{CH}}_j(\cdot)$ is better than the ones given in Lemmas~\ref{lem:basicMOM} and~\ref{lem:basicTMean} with $\alpha = 1$.
This stems from the fact that the analysis of Catoni's estimator~\citep{catoni2012challenging} results in a deviation with the best possible constant~\citep{devroye2016sub}.
However, contrary to $\mathtt{MOM}$ and $\mathtt{TM}$, an estimator of the scale is necessary: it makes $\mathtt{CH}$ computationally much more demanding (see Figure~\ref{fig:estimator_runtimes} below), since it requires to perform two fixed-point iterations to approximate both $\wh \sigma_j(\theta)$ and $\wh g_j^{\mathtt{CH}}(\theta)$ and it requires Assumption~\ref{assump:data} with $\alpha = 1$ so that $\sigma_j(\theta) < +\infty$.
Moreover, there is no guaranteed robustness to outliers, a fact confirmed by the numerical experiments performed in Section~\ref{sec:experiments} below.

\begin{proposition}
    \label{prop:uniformCH}
    Grant Assumptions~\ref{assump:lipsmoothloss} and~\ref{assump:data} with $\alpha = 1$ and $\cO = \emptyset$. 
    Denote $\overline{L} = \E [\gamma \|X\|^2 ],$ $\sigma_L^2 = \V [\gamma \|X\|^2 ]$ and assume that for all $\theta, \widetilde \theta \in \Theta$ such that $\|\theta - \widetilde \theta\| \leq 1/\sqrt{n}$ we have 
    \begin{equation*}
        \frac 1 2 \sigma_j^2(\widetilde \theta) \leq \sigma^2_j(\theta ) \leq 2 \sigma^2_j(\widetilde \theta) \quad \text{ and }\quad \frac{\sigma_j(\theta)}{ \sigma_L} \geq \frac{1}{\sqrt{n}}.
    \end{equation*}
    Furthermore\textup, assume that for all $\theta \in \Theta,$ the variance estimator $\wh \sigma_j(\theta)$ defined by~\eqref{eq:scaleMestimator} satisfies $\sigma_j(\theta) / C' \leq \wh \sigma_j(\theta) \leq C' \sigma_j(\theta)$ for some constant $C' > 1$ with probability at least $ 1 - \delta/2$.
    Then\textup, we have
    \begin{equation*}
        \Proba \Big[ \sup_{\theta\in\Theta}\big| \widehat{g}_j^{\mathtt{CH}}(\theta) - g_j(\theta) \big| \leq \epsilon_j^{\mathtt{CH}}(\delta) \Big] \geq 1 - \delta \quad
    \end{equation*}    
    for any $j\in\setint{d}$ with
    \begin{equation*}
        \epsilon_j^{\mathtt{CH}}(\delta):= 4C' \Big( 2 \Sigma_j + \frac{\sigma_L}{\sqrt{n}} \Big) 
        \sqrt{\frac{\log(4d/\delta) + d\log(3\Delta \sqrt{n}/2)}{n}} + \frac{\overline{L} + L_{j}}{\sqrt{n}} 
    \end{equation*}
    where $\overline{L}$ is as in Proposition~\ref{prop:uniformMOM}.
\end{proposition}

The proof of Proposition~\ref{prop:uniformCH} is given in Appendix~\ref{sec:proofs}.
It uses again an $\varepsilon$-net argument combined with a careful control of the variations of $\widehat{g}_j^{\mathtt{CH}}(\theta)$ with respect to $\theta$.
Compared with~\cite{pmlr-v97-holland19a}, we make a different use of the $\mathtt{CH}$ estimator: while it is used therein to estimate the whole gradient $\grad R(\theta)$ during the robust GD iterations, we use it here to estimate the partial derivatives $g_j(\theta)$ during  iterations of robust CGD.
The numerical experiments from Section~\ref{sec:experiments} confirm, in particular, that our approach leads to a considerable speedup and improved statistical performances when compared to~\cite{pmlr-v97-holland19a}.

The statements of Lemma~\ref{lem:basicCH} and Proposition~\ref{prop:uniformCH} require $\alpha=1$, while a very recent extension of Catoni's bound~\citep{chen2021generalized} is available for $\alpha\in(0,1)$.
However, the necessity to estimate the centered $(1+\alpha)$-moment subsists (standard-deviation for $\alpha=1$).
Although iteration~\eqref{eq:iterscaleMestimator} may be adapted to this case, theoretical guarantees for it do lack.
Note that even for $\alpha=1$, the statements of Lemma~\ref{lem:basicCH} and Proposition~\ref{prop:uniformCH} require assumptions on $\sigma_j^2(\theta)$ and $\wh \sigma_j(\theta)$: an extension to $\alpha \in (0, 1]$ would lead to a set of even more intricate assumptions.

\paragraph{Complexity of $\wh g_j^{\mathtt{CH}}(\theta)$.}

It is not straightforward to analyze the complexity of this estimator, since it involves fixed-point iterations with a number of iterations that can vary from one run to the other.
However, each iteration has complexity $O(n)$ and we observe empirically that the number of iterations is of constant order (usually smaller than $10$) independently from the required confidence.
Therefore, the overall complexity remains in $O(n)$ as demonstrated also by Figure~\ref{fig:estimator_runtimes} below.
The latter also shows that the numerical complexity of $\mathtt{CH}$ is larger than that of $\mathtt{MOM}$ and $\mathtt{TM}$, which later impacts the overall training time.

\subsection{A comparison of the numerical complexities}

As explained above, all the considered estimators of the partial derivatives have a numerical complexity $O(n)$.
However, they perform different computations and have very different running times in practice.
So, in order to compare their actual computational complexities we perform the following experiment.
We consider an increasing sample size $n$ between $10^2$ and $10^6$ on a logarithmic scale and run all the estimators:  $\mathtt{MOM}$, $\mathtt{TM}$, $\mathtt{CH}$ and $\mathtt{ERM}$, which is the average of the per-sample partial derivatives $g_j^i(\theta)$.
We fix their parameters so as to obtain deviation bounds with confidence $1 - \delta = 99\%$: this corresponds to 82 blocks for $\mathtt{MOM}$, $\epsilon=72/n$ for $\mathtt{TM}$ and $\delta=0.01$ for $\mathtt{CH}$, but the conclusion is similar with different combinations of parameters.
We use random samples with student $t(2.1)$ distribution (a finite variance distribution but with heavy tails, although run times do not differ by much when using different distributions).
This leads to the display proposed in Figure~\ref{fig:estimator_runtimes}, where we display the averaged timings over 100 repetitions (together with standard-deviations).

\begin{figure}[!ht]
    \centering
    \includegraphics[width=0.6\textwidth]{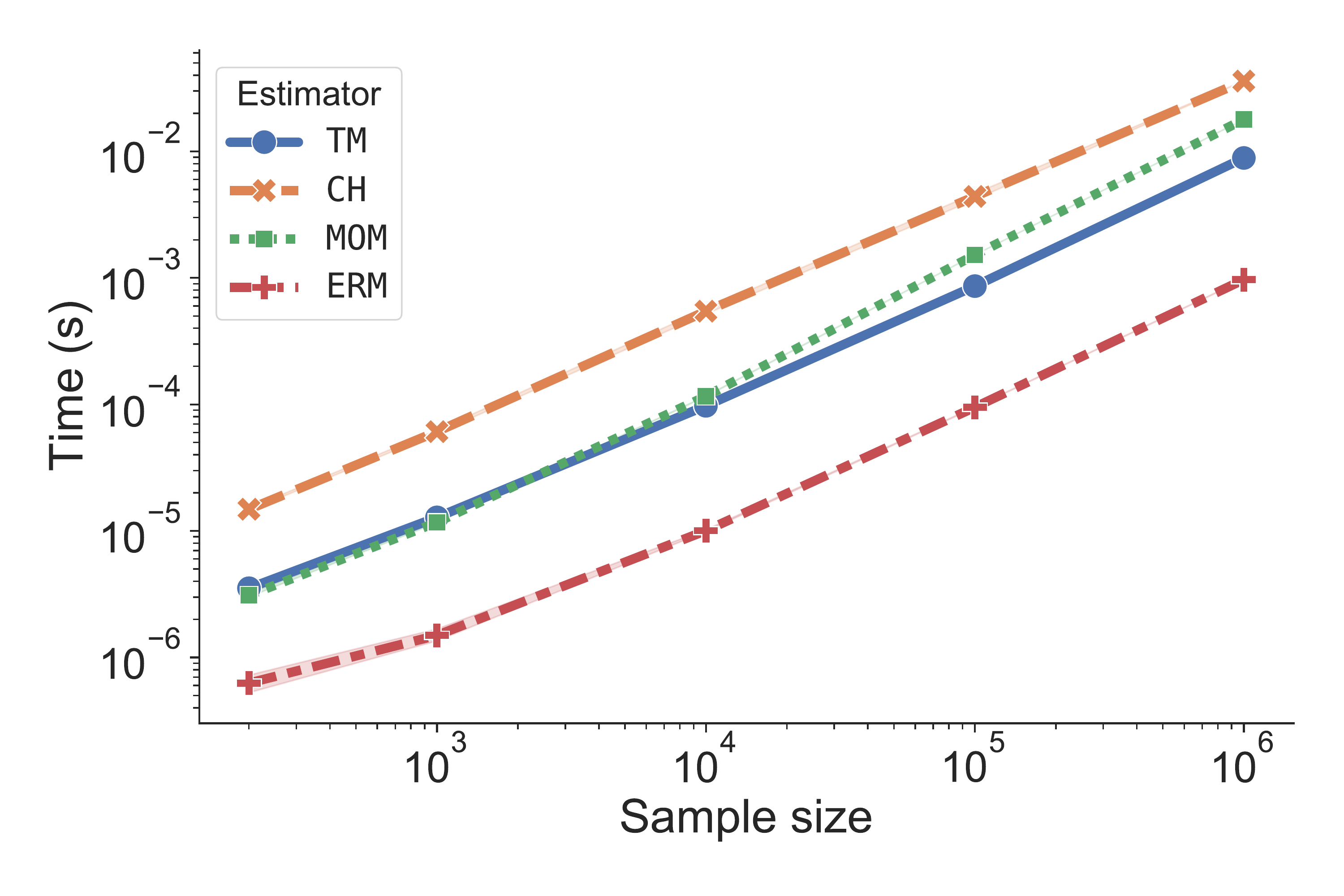}
    \caption{\small Average running time ($y$-axis) of all the considered estimators against an  increasing sample size ($x$-axis). The run times increase with a similar slope (on a logarithmic scale), confirming $O(n)$ complexities, but differ significantly: $\mathtt{ERM}$ is of course the fastest, followed by $\mathtt{TM}$ and $\mathtt{MOM}$ (both are close) and finally $\mathtt{CH}$, which is the slowest.}
    \label{fig:estimator_runtimes}
\end{figure}

We observe that the run times of the estimators increase with a similar slope (on a logarithmic scale) against the sample size, confirming the $O(n)$ complexities. 
However, their timings differ significantly. 
$\mathtt{MOM}$ and $\mathtt{TM}$ share similar timings ($\mathtt{TM}$ becomes faster than $\mathtt{MOM}$ for large samples) and are about 10 times slower than $\mathtt{ERM}$. $\mathtt{CH}$ is the slowest of all and is roughly 50 times slower than $\mathtt{ERM}$.
This is of course related to the fact that $\mathtt{CH}$ requires to perform the fixed-point iterations each of which roughly costing $\Theta(n)$. 
In all cases, the estimators' complexities remain in $O(n)$ so that the complexity of a single iteration of robust CGD (see Algorithm~\ref{alg:robust-cgd}) using either of them is $O(n)$, which is identical to the complexity of a non-robust ERM-based CGD. 
This means that Algorithm~\ref{alg:robust-cgd} achieves robustness at a limited cost, where the computational difference lies only in the constants in front of the big O notations.

\section{Related works}
\label{sec:relatedworks}

Robust statistics have received a longstanding interest and started in the 60s with the pioneering works of~\cite{tukey1960} and~\cite{huber1964robust}. 
Since then, several works pursued the development of robust statistical methods including 
non-convex $M$-estimators~\citep{huber2004robust}, $\ell_1$ tournaments~\citep{Devroye1987NonparametricDE,10.1214/aos/1176350820} and methods based on depth functions~\citep{chen2017robust, gao2020robust, mizera2002depth}, the latter being difficult to use in practice because of their numerical complexity.

A renewal of interest has manifested recently, related, on the one hand, to the increasing need for algorithms able to learn from large non-curated datasets and on the other hand, to the development of robust mean estimators with good theoretical guarantees under weak moment assumptions, including Median-of-Means (MOM)~\citep{nemirovskij1983problem, alon1999space, JERRUM1986169} and Catoni's estimator~\citep{catoni2012challenging}.
Under adversarial corruption~\citep{charikar2017learning}, several statistical learning problems related to robustness are studied, such as parameter estimation~\citep{lai2016agnostic, pmlr-v108-prasad20a, minsker2018sub, diakonikolas2019robust,lugosi2021robust}, regression~\citep{klivans2018efficient, liu2020high, cherapanamjeri2020optimal, bhatia2017consistent}, classification~\citep{lecue2020robust1, klivans2009learning, liu2015classification}, PCA~\citep{li2017robust, candes2011robust,paul2021robust} and most recently online learning~\citep{pmlr-v134-vanerven21a}.

In the heavy-tailed setting, a robust learning approach introduced in~\cite{brownlees2015empirical} proposes to optimize a robust estimator of the risk based on Catoni's mean estimator~\citep{catoni2012challenging} resulting in an implicit estimator for which they show near-optimal guarantees under weak assumptions on the data.
However, the new risk may not be convex (even if the considered loss is), so that its minimization may be expensive and lead to an estimator unrelated to the one theoretically studied, potentially making the associated guarantees inapplicable.
More recently, an explicit variant was proposed in~\cite{Zhang20181regressionWH} which applies Catoni's influence function to each term of the sum defining the empirical risk for linear regression. 
The associated optimum enjoys a sub-Gaussian bound on the excess risk, albeit with a slow rate since the $\ell_1$ loss was used. 
A follow-up extended this result under weaker distribution assumptions~\citep{chen2021generalized}. 
The main drawback of this approach is that the unconventional use of the influence function introduces a considerable amount of bias which appears in the excess risk bounds.

Another approach proposed in~\cite{minsker2015geometric, hsu2016loss} aims at obtaining a robust estimator by computing standard ERMs on disjoint subsets of the data and aggregating them using a multidimensional MOM. 
This approach has recently been used as well in~\cite{pmlr-v130-holland21a} with various aggregation strategies in order to perform robust distributed learning. 
Although the previous works use easily implementable aggregation procedures, the associated deviation bounds are sub-optimal (see for instance~\cite{lugosi2019mean}).
Moreover, dividing the data into multiple subsets makes the method impractical for small sample sizes and may introduce bias coming from the choice of such a subdivision. 

In the setting where an $\eta$-proportion of the data consist of arbitrary outliers, a robust meta-algorithm is introduced in~\cite{diakonikolas2019sever}, which repeatedly trains a given base learner and filters outliers based on an eccentricity score. The method reaches the target $\sigma\sqrt{\eta}$ error rate with $\sigma$ the gradient standard deviation, although the requirement of multiple training rounds may be computationally expensive.

More recently, robust solutions to classification problems were proposed in~\cite{lecue2020robust1} by using MOM to estimate the risk and computing gradients on trustworthy data subsets in order to perform descent.
A variant was also proposed by the same authors in~\cite{lecue2020robust2} where a pair of parameters is alternately optimized for a min-max objective.
The resulting algorithm is efficient numerically, though it requires a vanishing step-size to converge due to the variance coming from gradient estimation. 
Moreover, the provided theoretical guarantees concern the optimum of the formulated problem but not the optimization algorithm put to use.

Several recent papers~\citep{HeavyTails, holland2019efficient, holland2019robust, pmlr-v97-holland19a, chen2017distributed} perform a form of robust gradient descent, where learning is guided by various robust estimators of the true gradient $\grad R(\theta)$. 
Two robust gradient estimation algorithms are proposed in~\cite{HeavyTails}. 
The first one is a vector analog of MOM where the scalar median is replaced by the geometric median
\begin{equation}
    \label{eq:gmom}
    \gmed(g_1, \dots, g_K):= \argmin_{g \in \R^d} \sum_{j=1}^K \|g - g_j\|_2,
\end{equation} 
which can be computed using the algorithm given in~\cite{Vardi1423}. 
This vector mean estimator enjoys improved concentration properties over the standard mean as shown in~\cite{minsker2015geometric} although these remain sub-optimal (see also~\cite{lugosi2019mean}). 
A line of works~\citep{lugosi2019sub, Hopkins2018MeanEW, pmlr-v99-cherapanamjeri19b, Depersin2019RobustSE, lugosi2021robust, lei2020fast} specifically addresses the issue of devising efficient procedures with optimal deviation bounds.

Supervised learning with robustness to heavy-tails and a limited number of outliers is thus achieved but at a possibly high computational cost. 
The second algorithm called ``Huber gradient estimator'' is intended for Huber's $\epsilon$-contamination setting. 
It uses recursive SVD decompositions followed by projections and truncations in order to filter out corruption.
The method proves to be robust to data corruption but its computational cost becomes prohibitive as soon as the data has moderately large dimensionality.

\section{Theoretical guarantee without strong convexity}
\label{sec:theory-not-strongly-cvx}

In this section we provide an upper bound similar to that of Theorem~\ref{thm:linconv1expect}, but without the strong convexity condition from Assumption~\ref{ass:strongconvexity}.
As explained in Theorem~\ref{thm:not-strgly-cvx} below, without strong convexity, the optimization error shrinks at a slower sub-linear rate when compared to Theorem~\ref{thm:linconv1expect} (a well-known fact, see~\cite{bubeck2015convex}).
In order to ensure that robust CGD, which uses ``noisy'' partial derivatives, remains a descent algorithm, we assume that the parameter set can be written as a product $\Theta = \prod_{j\in \setint{d}}\Theta_j$ and replace the iterations~\eqref{eq:iter-strgly-cvx} (corresponding to Line~5 in Algorithm~\ref{alg:robust-cgd}) by
\begin{equation}
    \label{eq:iter-not-strgly-cvx}
    \begin{cases}
    \theta^{(t+1)}_j \gets \proj_{\Theta_j}\big( \theta_j^{(t)} - \beta_j \tau_{\epsilon_{j}}\big(\wh g_j(\theta^{(t)})\big)\big) & \text{ if } j = j_t \\
    \theta^{(t+1)}_j \gets \theta_j^{(t)} & \text{ otherwise},
    \end{cases}
\end{equation}
where $\proj_{\Theta_j}$ is the projection onto $\Theta_j$ and $\tau_\epsilon$ is the soft-thresholding operator given by $\tau_{\epsilon}(x) = \sign(x)(|x| - \epsilon)_+$ with $(x)_+ = \max(x, 0)$.
In Theorem~\ref{thm:not-strgly-cvx} below we use $\epsilon_j = \epsilon_j(\delta)$, the $j$-th coordinate of the error vector from Definition~\ref{def:error-vector}, which is instantiated for each robust estimator in Section~\ref{sec:robust-estimators}.
Since it depends on the moment $m_{\alpha, j}$, it is not observable, so we propose in Lemma~\ref{lem:mom-for-grad-moment} from Appendix~\ref{sub:observable-upper-bound-moment} an observable upper bound deviation for it based on $\mathtt{MOM}$.

This use of soft-thresholding of the partial derivatives can be understood as a form of partial derivatives (or gradient) clipping.
However, note that it is rather a theoretical artifact than something to use in practice (we never use $\tau_\epsilon$ in our numerical experiments from Section~\ref{sec:experiments} below).
Indeed, the operator $\tau_\epsilon$ naturally appears for the following simple reason: consider a convex $L$-smooth scalar function $f: \R \to \R$ with derivative  $g(x):= f'(x)$.
An iteration of gradient descent from $x_0$ uses an increment $\delta$ that minimizes the right-hand side of the following inequality:
\begin{equation*}
  f(x_0 + \delta) \leq Q(\delta, x_0):= f(x_0) + \delta g(x_0) + \frac{L}{2}\delta^2,
\end{equation*}
namely $\argmin_{\delta} Q(\delta, x_0) = -g(x_0) / L$ leading to the iterate $x_0 - g(x_0) / L$ with ensured improvement of the objective. 
In our context, $g(x)$ is unknown and we use an estimator $\wh g(x)$ satisfying $|\wh g(x) - g(x)| \leq \epsilon$ with a large probability. 
Taking this uncertainty into account leads to the upper bound
\begin{equation*}
  f(x_0 + \delta) \leq \widetilde{Q}(\delta, x_0):= f(x_0) + \delta \hg(x_0) + \frac{L}{2}\delta^2 + \epsilon|\delta|,
\end{equation*}
and, after projection onto the parameter set, to the iteration~\eqref{eq:iter-not-strgly-cvx} since $\argmin_\delta \widetilde Q(\delta, x_0) = x_0 - \tau_{\epsilon} (\wh g(x_0)) / L$, with guaranteed decrease of the objective. 

The clipping of partial derivatives is unnecessary in the strongly convex case since each iteration translates into a contraction of the excess risk, so that the degradations caused by the gradient errors remain controlled (see the proof of Theorem~\ref{thm:linconv1expect}). 
No such contraction can be established without strong convexity, and clipping prevents gradient errors to accumulate uncontrollably.

\begin{theorem}
    \label{thm:not-strgly-cvx}
    Grant Assumptions~\ref{assump:lipsmoothloss} and~\ref{ass:mintheta-and-smoothness} with $\Theta = \prod_{j\in\setint{d}}\Theta_j$. 
    Let $\theta^{(T)}$ be the output of Algorithm~\ref{alg:robust-cgd} where we replace iterations~\eqref{eq:iter-strgly-cvx} by~\eqref{eq:iter-not-strgly-cvx} with step-sizes $\beta_j = 1 /  L_j,$ an initial iterate $\theta^{(0)}\in\Theta,$ uniform coordinates sampling $p_j = 1 / d$ and estimators of the partial derivatives with error vector $\epsilon(\cdot)$. 
    Then\textup, we have with probability at least $1-\delta$
    \begin{align*}
        \E \big[ R(\theta^{(T)}) ] - R^\star &\leq \frac{d}{T + 1} \bigg( \sum_{j\in\setint{d}} \frac{L_{j}}{2} \big( \theta^{(0)}_j - \theta^\star_j\big)^2 + R(\theta^{(0)}) \bigg) +
        \frac{2 \| \epsilon ( \delta ) \|_2}{T + 1} \sum_{t=0}^T \| \theta^{(t)} - \theta^\star \|_2,
    \end{align*}
    where the expectation is w.r.t the sampling of the coordinates. 
    Moreover\textup, we have 
    \begin{equation*}
        \| \thetat - \theta^\star \big\|_2 \leq \| \theta^{(t-1)} - \theta^\star \|_2
    \end{equation*}
    with the same probability\textup, for all $t\in \setint{T}$.
\end{theorem}
The proof of Theorem~\ref{thm:not-strgly-cvx} is given in Appendix~\ref{sec:proofs} and is based on the proof of Theorem~5 from~\cite{nesterov2012efficiency} and Theorem~1 from~\cite{shalev2011stochastic} while managing noisy partial derivatives. 
The optimization error term vanishes at a sublinear $1 / T$ rate and is initially of order $R(\theta^{(0)})$ plus the potential $\Phi(\theta) = \sum_{j=1}^d L_j(\theta_j - \theta^\star_j)^2/2$ which is instrumental in the proof.
Notice that $\| \epsilon ( \delta ) \|_2$ appears without the square which translates into ``slow'' $1/\sqrt{n}$ rates instead of ``fast'' $1/n$ rates stated achieved by the bounds from Section~\ref{sec:robustcgd}.
This degradation is an unavoidable consequence of the loss of strong convexity of the risk~\citep{srebro2010optimistic}.

\section{Numerical Experiments}
\label{sec:experiments}

The theoretical results given in Sections~\ref{sec:robustcgd},~\ref{sec:robust-estimators} and~\ref{sec:theory-not-strongly-cvx} can be applied to a wide range of linear methods for supervised learning, with guaranteed robustness both with respect to heavy-tailed data and outliers. 
We perform below experiments that confirm these robustness properties for several tasks (regression, binary classification and multi-class classification) on several datasets including a comparison with many baselines including the state-of-the-art. 

\subsection{Algorithms}

The algorithms introduced in this paper are compared with several baselines among the following large set of algorithms.
For all algorithms, we use, unless specified otherwise, the least-squares loss for regression, and the logistic loss for classification (both for binary and multiclass problems, using the multiclass logistic loss).
The algorithms studied and compared below can be used easily in a few lines of \texttt{Python} code with our library called \texttt{linlearn}, open-sourced under the BSD-3 License on \texttt{GitHub} and available here: \url{https://github.com/linlearn/linlearn}.
This library follows the API conventions of \texttt{scikit-learn}~\citep{pedregosa2011scikit-learn}.
    
\paragraph{CGD algorithms: $\mathtt{MOM}$, $\mathtt{CH}$, $\mathtt{TM}$ and $\mathtt{CGD\ ERM}$.}

The $\mathtt{MOM}$, $\mathtt{CH}$ and $\mathtt{TM}$ algorithms are the different variants of robust CGD (Algorithm~\ref{alg:robust-cgd}) introduced in this paper, respectively based on median-of-means, trimmed mean and Catoni-Holland estimators of the partial derivatives introduced in Section~\ref{sec:robust-estimators}. 
We include also $\mathtt{CGD\ ERM}$ which is CGD using a non-robust estimation of the partial derivatives based on a mean.

\paragraph{GD algorithms: $\mathtt{ERM}$, $\mathtt{LLM}$, $\mathtt{HG}$, $\mathtt{GMOM}$, $\mathtt{CH\ GD}$ and  $\mathtt{Oracle}$.}

These are all GD algorithms using different estimators of the gradients.
$\mathtt{ERM}$ uses a non-robust gradient based on a simple mean. 
$\mathtt{LLM}$ corresponds to Algorithm~1 from~\cite{lecue2020robust1}. 
It uses a MOM estimation of the risk and performs GD using gradients computed as the mean of the gradients from the block corresponding to the median of the risk.
$\mathtt{HG}$ is Algorithm 2 from~\cite{HeavyTails}, called Huber Gradient Estimator, which uses recursive SVD decompositions and truncations to compute a robust gradient.
$\mathtt{GMOM}$ is Algorithm~3 from~\cite{HeavyTails}, which estimates gradients using a geometric MOM (based on the geometric median).
$\mathtt{CH\ GD}$ is the robust GD algorithm from~\cite{pmlr-v97-holland19a}, which uses gradients computed as coordinate-wise $\mathtt{CH}$ estimators.
We consider also $\mathtt{Oracle}$,  which is GD performed with ``oracle'' gradients, namely the gradient of the unobserved true risk (only available for linear regression experiments using  simulated data).

\paragraph{Extra algorithms: $\mathtt{RANSAC}$, $\mathtt{HUBER}$ and $\mathtt{LAD}$.}

We consider also the following extra algorithms.
For regression, we consider $\mathtt{RANSAC}$~\citep{10.1145/358669.358692}, using the implementation available in the scikit-learn library~\citep{pedregosa2011scikit-learn}.
$\mathtt{HUBER}$ stands for ERM learning with the modified Huber loss~\citep{10.1145/1015330.1015332} for classification and Huber loss~\citep{owen2007} for regression.
$\mathtt{LAD}$ is ERM learning using the least absolute deviation loss~\citep{10.2307/23036355}, namely regression using the mean absolute error instead of least-squares.

\subsection{Regression on simulated datasets}
\label{sub:regression-simulated}

We consider the following simulation setting for linear regression with the square loss.
We generate features $X \in \R^d$ with $d=5$ with a non-isotropic Gaussian distribution with covariance matrix $\Sigma$ and labels  $Y = X^\top \theta^\star + \xi$ for a fixed $\theta^\star \in \R^d$ and simulated noise $\xi$.
Since all distributions are known in such simulated data, we can compute the true risk and true gradients (used in $\mathtt{Oracle}$).

We consider the following simulation settings: (a) $\xi$ is centered Gaussian; (b) $\xi$ is Student with $\nu=2.1$ degrees of freedom (heavy-tailed noise).
We consider then settings (c), (d), (e) and (f) where $\xi$ is the same in (b) but 1\% of the data is replaced by outliers as follows. 
For case $(c)$, $X \in \R^5$ is replaced by a constant equal to $\lambda_{\max}(\Sigma)$ (largest eigenvalue of $\Sigma$) and labels are replaced by $2 y_{\max}$ with $y_{\max} = \max_{i\in \mathcal I} |y_i|$; for (d) we do the same as (c) and additionally multiply labels by $-1$ with probability $1/2$; for (e) we sample $X = 10\lambda_{\max}(\Sigma) v + Z$ where $v \in \R^5$ is a fixed unit vector chosen at random and $Z$ is a standard Gaussian vector and labels are i.i.d. Bernoulli random variables; finally for (f) we sample $X = 10\lambda_{\max}(\Sigma) V$ where $V$ is uniform on the unit sphere and labels $y = y_{\max} \times (\varepsilon + U)$ where $\varepsilon$ is a Rademacher variable and $U$ is uniform in $[-1/5, 1/5]$.

For this experiment on simulated datasets, we fix the parameters of the robust estimators of the partial derivatives using the confidence level $\delta=0.01$ and the number of outliers for $\mathtt{MOM}$ and $\mathtt{TM}$. 
We report, for all considered simulation settings (a)-(f), the average over 30 repetitions of the excess risk for the square loss ($y$-axis) of all considered algorithms along their iterations ($x$-axis, corresponding to cycles for CGD and iterations for GD) in Figure~\ref{fig:lin_reg}.
\begin{figure}[!ht]
     \centering
     \includegraphics[width=0.8\textwidth]{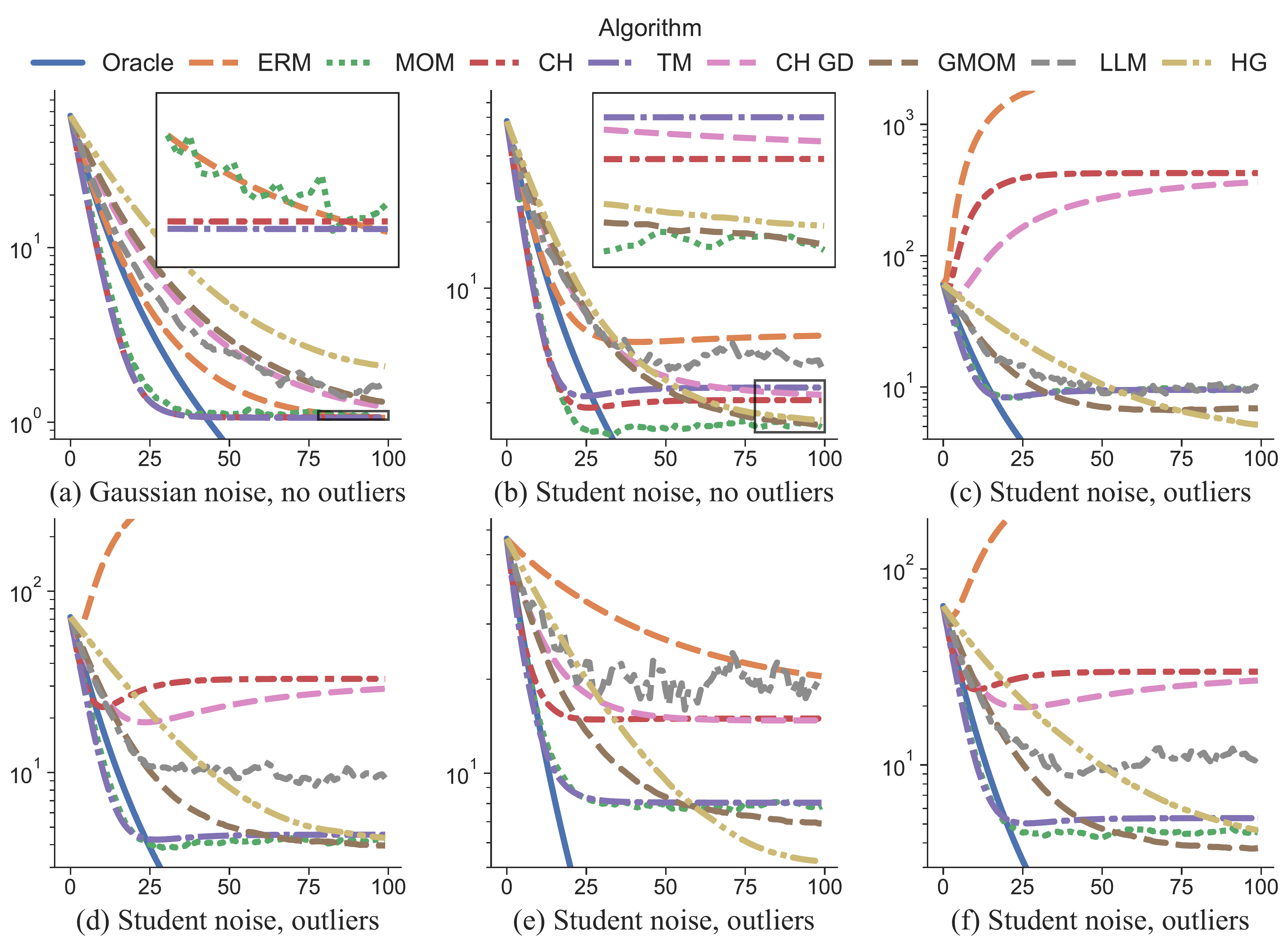}
        \caption{Excess-risk for the square loss ($y$-axis) against iterations ($x$-axis) for all the considered algorithms in the simulation settings~(a)-(c) (top row) and (d)-(f) (bottom row). We zoom-in the last iterations for simulation settings~(a) and~(b) to improve readability.}
        \label{fig:lin_reg}
\end{figure}
We observe that CGD-based algorithms generally converge faster than GD-based ones, independently of the quality of the optimum found.
For~(a) with Gaussian noise and no outliers, the final performance of all algorithms is roughly similar to that of $\mathtt{ERM}$ (as expected since the sample mean has optimal deviation guarantees for sub-Gaussian distributions) except for $\mathtt{LLM}$ and $\mathtt{HG}$ that converge slowly. 
For~(b) with heavy-tailed noise and no outliers, $\mathtt{ERM}$ clearly degrades when compared to robust methods, with $\mathtt{HG}$ reaching the best result and $\mathtt{LLM}$ the worst.
For settings~(c)-(f) with heavy-tailed noise and outliers, we observe different behaviours. 
We observe that $\mathtt{ERM}$ and $\mathtt{CH}$ are the most sensitive to outliers, especially in setting~(c) and~(e) (single-direction corruption of the gradients), where $\mathtt{GMOM}$ and $\mathtt{HG}$ (based on robust gradient estimation) perform best while $\mathtt{MOM}$ and $\mathtt{TM}$ are close competitors.
For settings~(d) and~(f) where gradients can be corrupted in multiple directions, the performance difference between $\mathtt{GMOM}$/$\mathtt{HG}$ and $\mathtt{MOM}$/$\mathtt{TM}$ is small.
A remarkable property to keep in mind is that $\mathtt{MOM}$/$\mathtt{TM}$ always converge faster.
Finally, while we observe that $\mathtt{LLM}$ is robust to heavy tails and outliers, its use of a median mini-batch and vanishing descent steps makes it unstable and prevents it from converging to a good minimum, compared to other algorithms.

\subsection{Classification on several datasets}
\label{sub:classification-datasets}

We consider classification tasks (binary and multiclass) on several datasets from the UCI Machine Learning Repository~\citep{Dua:2019}, see Appendix~\ref{sec:exp_details} for more details.
We use the logistic loss for binary classification and the multiclass logistic loss for multiclass problems. 
For $k$-class problems with $k > 2$, the iterates are $d \times k$ matrices and CGD is performed block-wise along the class axis. 
In this case, a CGD cycle performs again $d$ iterations (one for each feature coordinate) and each iteration updates the $k$ corresponding model weights (a form of block coordinate gradient descent, see~\cite{blondel2013block} for arguments in favor of this approach).

For each considered dataset, we corrupt an increasing random fraction of samples with uninformative outliers or a heavy-tailed noise.
Each algorithm is hyper-optimized using cross-validation over an appropriate grid of hyper-parameters. 
See Appendix~\ref{sec:exp_details} for further details about the experiments.
Then, each algorithm is trained again using the full training dataset 10 times over to account for the randomness lying within each method (although most procedures remain very stable across runs) and we finally report in Figure~\ref{fig:classif} the median accuracy obtained on a 15\% test-set ($y$-axis) for each dataset, corruption level ($x$-axis) and algorithm.
\begin{figure}[!ht]
    \centering
    \includegraphics[width=0.8\textwidth]{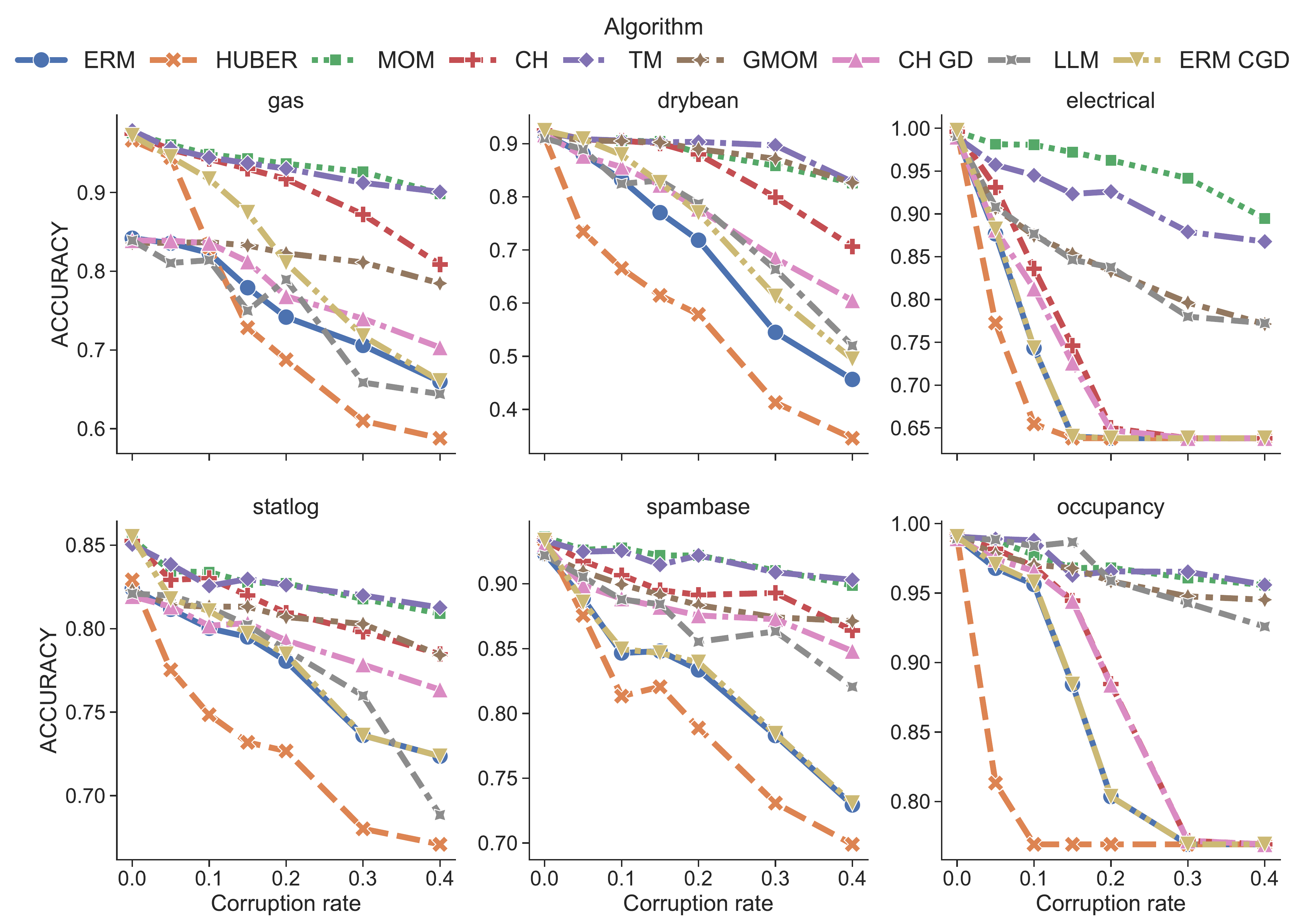}
    \caption{Test accuracy ($y$-axis) against the proportion of corrupted samples ($x$-axis) for six datasets and the considered algorithms.}
    \label{fig:classif}
\end{figure}

We observe that, as expected, the accuracy of each algorithm deteriorates with an increasing proportion of corrupted samples.
We observe that the robust CGD algorithms introduced in this paper are almost always superior, to all the considered baselines, and only suffer from a reasonable decrease in accuracy along the $x$-axis (from $0\%$ to $40\%$ corrupted samples) compared to all baselines.
In particular, $\mathtt{TM}$ and $\mathtt{MOM}$ are generally the best with $\mathtt{GMOM}$ being the closest competitor, and as expected from the theory, $\mathtt{CH}$ is less robust to corruption than both  $\mathtt{TM}$ and $\mathtt{MOM}$. 
Finally, note that the mere use of CGD instead of GD can give a significant advantage in order to find better optima as can be seen for the gas and statlog datasets.

In order to illustrate the computational performance of each method, we report in Figure~\ref{fig:classif_time} the test accuracy ($y$-axis) against the training time ($x$-axis) along iterations of each algorithm for two datasets (rows) and $0\%$, $15\%$ and $30\%$ corruption  (resp. first, middle and last column).
With $0\%$ corruption (first column), most algorithms reach their final accuracy within few iterations for the two considered datasets and our algorithms are somewhat slower than standard methods such as $\mathtt{ERM}$ and $\mathtt{HUBER}$. 
When corruption is present, our robust CGD algorithms reach a better accuracy, and they do so faster than other robust algorithms, such as $\mathtt{CH\ GD}$ and $\mathtt{GMOM}$.
Also, we can observe on this display, once again, the lack of stability of $\mathtt{LLM}$.

\begin{figure}[!ht]
    \centering
    \includegraphics[width=0.8\textwidth]{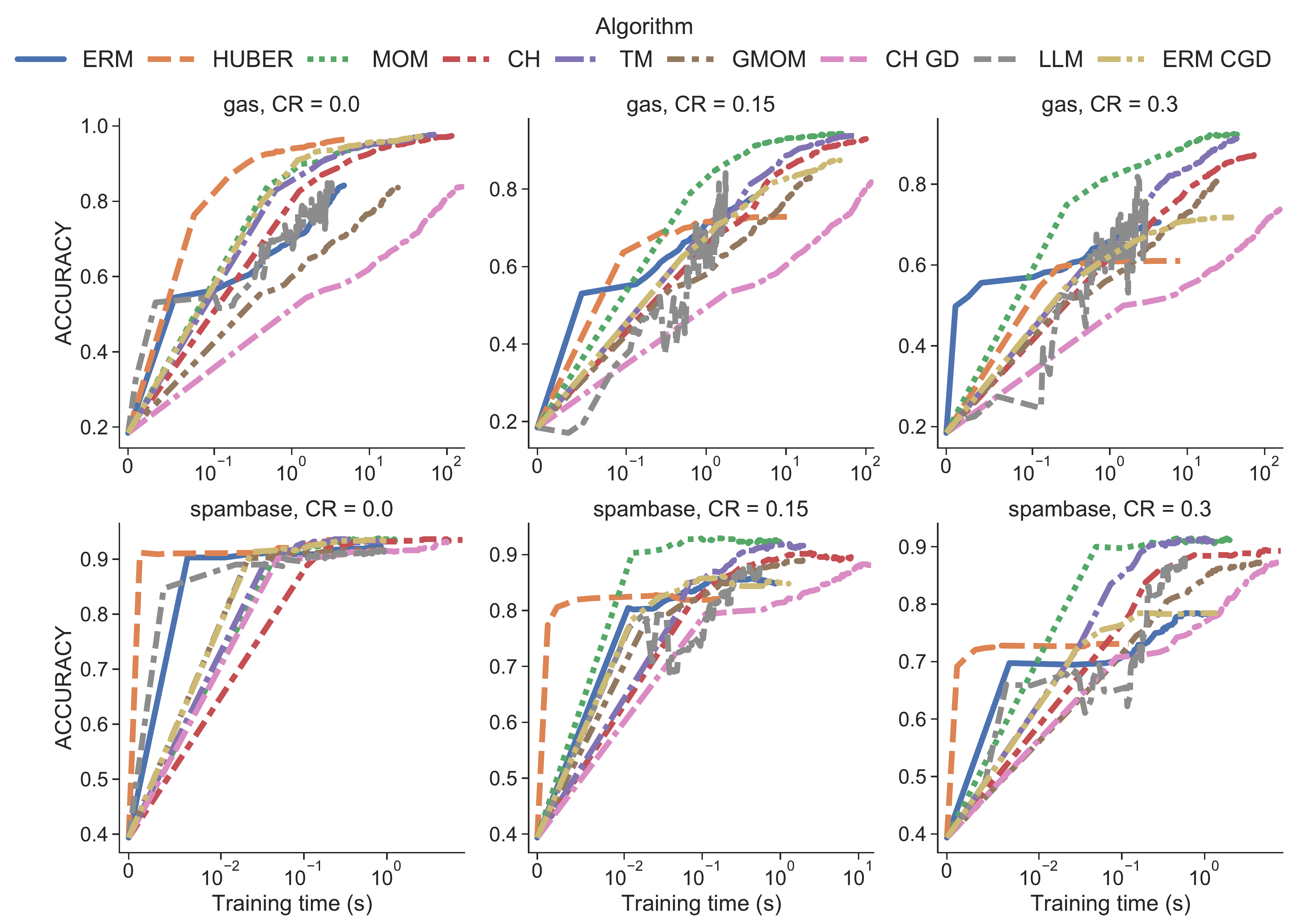}
    \caption{Test accuracy ($y$-axis) against computation time ($x$-axis) along training iterations on two datasets (rows) for $0\%$ corruption (first column), $15\%$ corruption (middle column) and $30\%$ corruption (last column).}
    \label{fig:classif_time}
\end{figure}

\subsection{Regression on several datasets}

We consider the same experimental setting (data corruption, hyper-optimization of algorithms) as in Section~\ref{sub:classification-datasets} but on different datasets from the UCI Machine Learning Database for regression tasks, see Appendix~\ref{sec:exp_details} for details.
We use the square loss for training and use the mean squared error (MSE) as a test metric, excepted for $\mathtt{HUBER}$, $\mathtt{RANSAC}$ and $\mathtt{LAD}$ which proceed differently.
We report the results in Figures~\ref{fig:reg} and~\ref{fig:reg_time}.
Figure~\ref{fig:reg} shows the test MSE ($y$-axis) against the proportion of corrupted samples ($x$-axis) for several datasets and algorithms while Figure~\ref{fig:reg_time} displays the test MSE against the training time analogously to Figure~\ref{fig:classif_time}.
Note that $\mathtt{RANSAC}$, $\mathtt{HUBER}$ and $\mathtt{LAD}$ appear through vertical lines only in Figure~\ref{fig:reg_time} since these use the scikit-learn implementations that do not give access to the training history.
We observe that $\mathtt{TM}$ and $\mathtt{MOM}$ are, once again, clear favorites.
Despite the fact that $\mathtt{HG}$ and $\mathtt{GMOM}$ prove to be very robust and are able to improve $\mathtt{MOM}$ and $\mathtt{TM}$ in certain instances by a small margin, their running times is slower and for some datasets orders of magnitude larger, as observed in Figure~\ref{fig:reg_time}.
This confirms the results observed as well on classification problems, that our robust CGD algorithms ($\mathtt{TM}$ and $\mathtt{MOM}$) offer an excellent compromise between statistical accuracy, robustness and computational effort.
Note also that we observe again the strong sensitivity of $\mathtt{CH}$ to outliers and the unstable performance of $\mathtt{LLM}$.
\begin{figure}[!ht]
    \centering
    \includegraphics[width=0.8\textwidth]{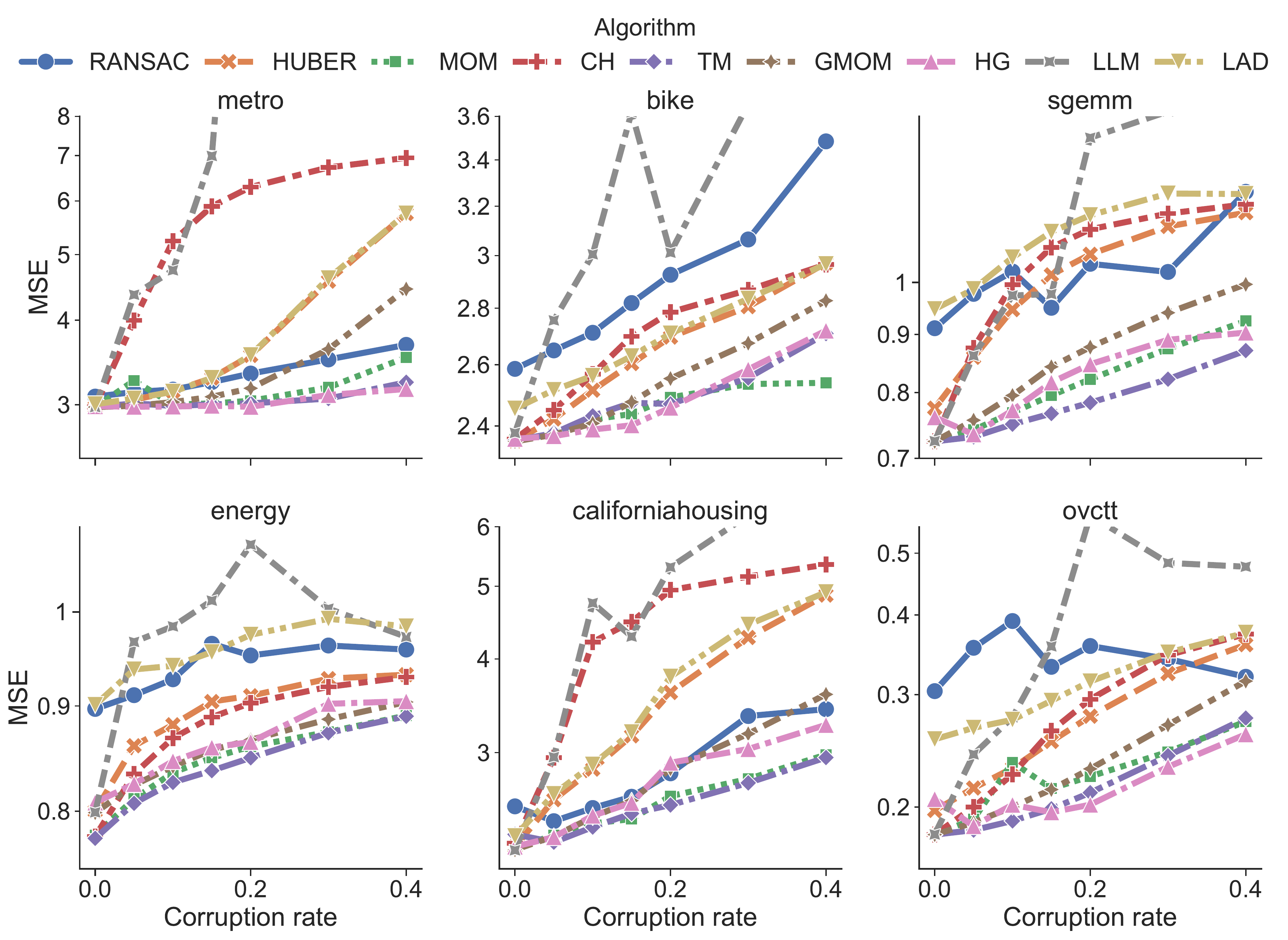}
    \caption{Mean squared error ($y$-axis) against the proportion of corrupted samples ($x$-axis) for six datasets and the considered algorithms.}
    \label{fig:reg}
\end{figure}
\begin{figure}[!ht]
    \centering
    \includegraphics[width=0.8\textwidth]{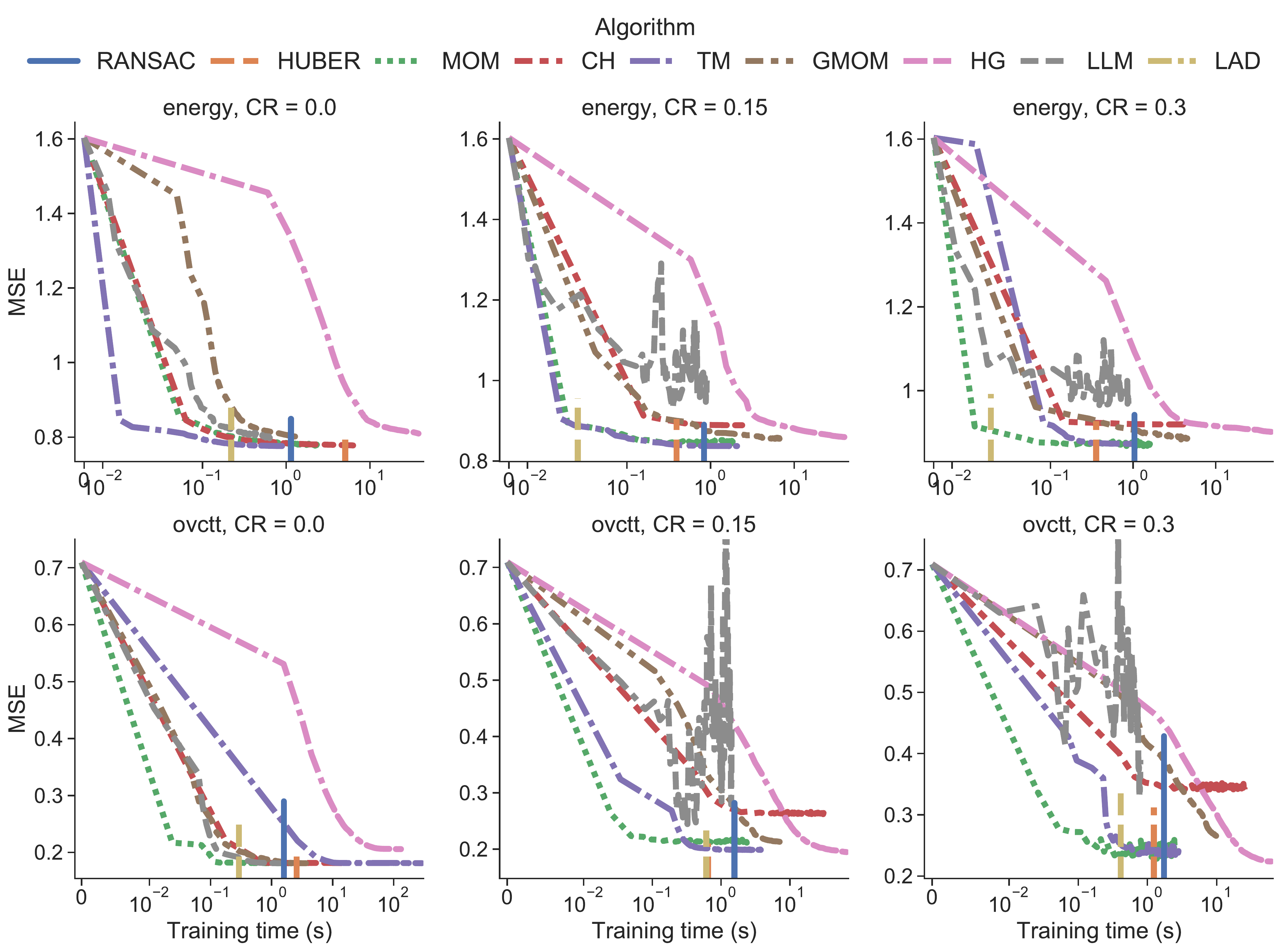}
    \caption{Mean squared error ($y$-axis) against computation time ($x$-axis) along training iterations on two datasets (rows) for $0\%$ corruption (first column), $15\%$ corruption (middle column) and $30\%$ corruption (last column).}
    \label{fig:reg_time}
\end{figure}

\section{Conclusion}

In this paper, we introduce new robust algorithms for supervised learning by combining two ingredients: robust CGD and several robust estimators of the partial derivatives. 
We derive convergence results for several variants of CGD with noisy partial derivatives and prove deviation bounds for all the considered robust estimators of the partial derivatives under somewhat minimal moment assumptions, including cases with infinite variance, and the presence of arbitrary outliers (except for the $\mathtt{CH}$ estimator).
This leads to very robust learning algorithms, with a numerical cost comparable to that of non-robust approaches based on empirical risk minimization, since it lets us bypass the need of a robust \emph{vector mean} estimator and allows to update model weights immediately using a robust estimator of a \emph{single partial derivative} only.
This is substantiated in our numerical experiments, that confirm the fact that our approach offers an excellent compromise between statistical accuracy, robustness and computational effort.
Perspectives include robust learning algorithms in high dimension, achieving sparsity-aware generalization bounds, which is beyond the scope of this paper, since it would require different algorithms based on methods such as mirror descent with an appropriately chosen divergence, see for instance~\cite{shalev2011stochastic, Juditsky2020SparseRB}.


\appendix

\newpage

\section{Supplementary theoretical results and details on experiments}

\subsection{The Lipschitz constants \texorpdfstring{$L_j$}{} are unknown}
\label{sub:unknown-step-sizes}

The step-sizes $(\beta_j)_{j \in \setint d}$ used in Theorems~\ref{thm:linconv1expect} and~\ref{thm:linconvdeterministic} are given by $\beta_j = 1 / L_j$, where the Lipschitz constants $L_j$ are defined by~\eqref{eq:lipschitz_constants}.
This makes them non-observable, since they depend on the unknown distribution of the non-corrupted features $P_{X_i}$ for $i \in \cI$.
We cannot use line-search~\citep{pjm/1102995080} here, since it requires to evaluate the objective $R(\theta)$, which is unknown as well.
In order to provide theoretical guarantees similar to that of Theorem~\ref{thm:linconv1expect} without knowing $(L_j)_{j=1}^d$, we use the following  approach.
First, we use the upper bound
\begin{equation}
    U_j:= \gamma \E\big[(X^j)^2\big] \geq L_j,
\end{equation}
which holds under Assumption~\ref{assump:lipsmoothloss} and estimate $\E[(X^j)^2]$ to build a robust estimator of $U_j$. 
In order to obtain an observable upper bound and to control its deviation with a large probability, we introduce the following condition.
\begin{definition}
    \label{def:lzeta-lxi-condition}
    We say that a real random variable $Z$ satisfies the $L^{\zeta}$-$L^{\xi}$ condition with constant $C \geq 1$ whenever it satisfies
    \begin{equation}
        \big( \E \big[ |Z - \E Z|^{\zeta} \big] \big)^{1/\zeta} \leq C \big( \E \big[ |Z - \E Z|^{\xi} \big] \big)^{1/\xi}.
    \end{equation}
\end{definition}

Using this condition, we can use the $\mathtt{MOM}$ estimator to obtain a high probability upper bound on $\E[(X^j)^2]$ as stated in the following lemma.
\begin{lemma}
    \label{lem:mom-for-Lj}
    Grant Assumption~\ref{assump:data} with $\alpha \in (0, 1]$ and suppose that for all $j\in\setint{d},$ the variable $(X^j)^2$ satisfies the $L^{(1+\alpha)}$-$L^1$ condition with a known constant $C$. 
    For any fixed $j \in  \setint d,$ let $\wh{\sigma}^2_j$ be the $\mathtt{MOM}$ estimator of $\E[(X^j)^2]$ with $K$ blocks. 
    If $|\mathcal{O}| \leq K / 12,$ we have
    \begin{equation*}
        \Proba \Big[ \Big(1 - 12^{1/(1+\alpha)}C\Big( \frac{K}{ n} \Big)^{\alpha/(1+\alpha)} \Big)^{-1}\wh{\sigma}_j^2 \leq \E [(X^j)^2 ] \Big] \leq \exp(-K/18).
    \end{equation*}
    If we fix a confidence level $\delta \in (0, 1)$ and choose $K:= \lceil 18 \log (1 / \delta) \rceil,$ we have
    \begin{equation*}
        \Big( 1 - 216^{1/(1+\alpha)}C\Big( \frac{\log(1/\delta)}{ n} \Big)^{\alpha/(1+\alpha)} \Big)^{-1} \wh{\sigma}_j^2 > \E[ (X^j)^2 ]
    \end{equation*}
    with a probability larger than $1 - \delta$.
\end{lemma}
The proof of Lemma~\ref{lem:mom-for-Lj} is given in Appendix~\ref{sec:proofs}. 
Denoting $\wh U_j$ the upper bounds it provides on $\E[(X^j)^2],$ we can readily bound the Lipschitz constants as $L_j \leq \gamma \wh U_j$ which leads to the following statement.

\begin{corollary}
\label{cor:estimate-lipschitz-constants}
    Grant the same assumptions as in Theorem~\ref{thm:linconv1expect} and Proposition~\ref{prop:uniformMOM}. Suppose additionally that for all $j\in\setint{d}$, the variable $(X^j)^2$ satisfies the $L^{(1+\alpha)}$-$L^1$ condition with a known constant $C$ and fix $\delta \in (0, 1)$. 
    Let $\theta^{(T)}$ be the output of Algorithm~\ref{alg:robust-cgd} with step-sizes $\wh \beta_j = 1 / \overline L_j$ where $\overline L_j:= \gamma \wh U_j$ and $\wh U_j$ are the upper bounds from Lemma~\ref{lem:mom-for-Lj} with confidence $\delta/2d,$ an initial iterate $\theta^{(0)},$ importance sampling distribution $p_j = \overline L_j / \sum_{k \in  \setint d} \overline L_{k}$ and estimators of the partial derivatives with error vector $\epsilon(\cdot)$.
    Then\textup, we have 
        \begin{equation}
        \E \big[ R(\theta^{(T)}) ] - R^\star \leq (R(\theta^{(0)}) - R^\star)
        \Big( 1 - \frac{\lambda}{\sum_{j \in  \setint d} \overline L_{j}} \Big)^T 
        + \frac{1}{2\lambda } \big\| \epsilon ( \delta/2 ) \big\|_2^2
    \end{equation}
     with probability at least $1 - \delta$.
\end{corollary}

The proof of Corollary~\ref{cor:estimate-lipschitz-constants} is given in Appendix~\ref{sec:proofs}.
It is a direct consequence of Theorem~\ref{thm:linconv1expect} and Lemma~\ref{lem:mom-for-Lj} and shows that an upper bound similar to that of Theorem~\ref{thm:linconv1expect} can be achieved with \emph{observable} step-sizes.
One may argue that the $L^{(1+\alpha)}$-$L^1$ condition simply bypasses the difficulty of deriving an observable upper bound by arbitrarily assuming that a ratio of moments is observed. 
However, we point out that a hypothesis of this nature is indispensable to obtain bounds such as the one above (alternatively, consider a real random variable with an infinitesimal mass drifting towards infinity). 
In fact, the $L^{(1+\alpha)}$-$L^1$ condition is much weaker than the requirement of boundedness (with known range) common to most known empirical bounds~\citep{Maurer2009EmpiricalBB, AUDIBERT20091876, mnih2008empirical}.

\subsection{Observable upper bound for the moment  $m_{\alpha, j}$}
\label{sub:observable-upper-bound-moment}

Since the moment $m_{\alpha, j}$, it is not observable, so we propose in Lemma~\ref{lem:mom-for-grad-moment} below an observable upper bound deviation for it based on $\mathtt{MOM}$.
Let us introduce now a robust estimator $\wh{m}^{\mathtt{MOM}}_{\alpha, j}(\theta)$ of the unknown moment $m_{\alpha, j}(\theta)$ using the following ``two-step'' $\mathtt{MOM}$ procedure.
First, we compute $\wh{g}^{\mathtt{MOM}}_j(\theta)$, the $\mathtt{MOM}$ estimator of $g_j(\theta)$ with $K$ blocks given by~\eqref{eq:mom-estimator}.
Then, we compute again a $\mathtt{MOM}$ estimator on $| g^i_j(\theta) - \wh{g}^{\mathtt{MOM}}_j(\theta) |^{1+\alpha}$ for $i\in \setint{n}$, namely
\begin{equation}
    \label{eq:moment-mom-estimator}
    \wh m_{\alpha, j}^{\mathtt{MOM}}(\theta):= \median \big( \wh m_{\alpha, j}^{(1)}(\theta), \ldots, \wh m_{\alpha, j}^{(K)}(\theta) \big),
\end{equation}
where
\begin{equation*}
    \wh m_{\alpha, j}^{(k)}(\theta):= \frac{1}{|B_k|} \sum_{i \in B_k} \big| g^i_j(\theta) - \wh{g}^{\mathtt{MOM}}_j(\theta) \big|^{1+\alpha},
\end{equation*}
using uniformly sampled blocks $B_1, \ldots, B_K$ of equal size that form a partition of $\setint n$.

\begin{lemma}
    \label{lem:mom-for-grad-moment}
    Grant Assumptions~\ref{assump:lipsmoothloss} and~\ref{assump:data} with $\alpha \in (0, 1]$ and suppose that for all $j\in\setint{d}$ and $\theta \in \Theta$ the partial derivatives $\ell'(X^\top \theta, Y)X^j$  satisfy the $L^{(1+\alpha)^2}$-$L^{(1+\alpha)}$ condition with known constant $C$ for any $j\in\setint{d}$ \textup(see Definition~\ref{def:lzeta-lxi-condition}\textup).
    Then\textup, if $|\mathcal{O}| \leq K / 12,$ we have
    \begin{equation*}
        \P \big[ \wh{m}^{\mathtt{MOM}}_{\alpha, j}(\theta) \leq  (1 - \kappa)m_{\alpha, j}(\theta) \big] \leq 2\exp(-K/18)
    \end{equation*}
    where $\kappa = \epsilon + 24 (1 + \alpha) \big(\frac{(1 + \epsilon) K}{n} \big)^{\alpha/(1+\alpha)}$ and $\epsilon = (24 (1 + C^{(1+\alpha)^2} ))^{1/(1+\alpha)} \big(\frac Kn)^{\alpha / (1 + \alpha)}$.
\end{lemma}

The proof of Lemma~\ref{lem:mom-for-grad-moment} is given in Appendix~\ref{sec:proofs}.

\subsection{Experimental details}
\label{sec:exp_details}

We provide in this section supplementary information about the numerical experiments conducted in Section~\ref{sec:experiments}.

\subsubsection{Datasets}

The main characteristics of the datasets used from the UCI repository are given in Table~\ref{tab:dataset-descrip} and their direct URLs are given in Table~\ref{tab:data-source}.

\begin{table}[!ht]
    \centering
    \small
\begin{tabular}{lcccccc}
    \hline
    Dataset & \# Samples & \# Features & \# Categorical & \# Classes \\
    \hline
    statlog & 6,435 & 36 & 0 & 6 \\
    spambase & 4,601 & 57 & 0 & 2 \\
    electrical & 10,000 & 13 & 0 & 2 \\
    occupancy~\citep{candanedo2016accurate} & 20,560 & 5 & 0 & 2 \\
    gas~\citep{vergara2012chemical} & 13,910 & 128 & 0 & 6 \\
    drybean~\citep{KOKLU2020105507} & 13,611 & 16 & 0 & 7 \\
    \hline
    energy~\citep{CANDANEDO201781} & 19,735 & 27 & 0 & - \\
    bike~\citep{fanaee2014event} & 17,379 & 10 & 5 & - \\
    metro & 48,204 & 6 & 1 & - \\
    sgemm~\citep{ballester2019sobol} & 241,600 & 14 & 0 & - \\
    ovctt & 68,784 & 20 & 2 & - \\
    californiahousing & 20,640 & 8 & 0 & - \\
    \hline
\end{tabular}
\caption{Main characteristics of the datasets used in experiments, including number of samples, number of features, number of categorical features and number of classes.}
\label{tab:dataset-descrip}
\end{table}

\begin{table}[!ht]
    \centering
    \footnotesize
    \begin{sideways}
    \begin{tabular}{ll}
    \toprule
        Dataset & URL \\
        \midrule
        statlog & \url{https://archive.ics.uci.edu/ml/datasets/Statlog+\%28Landsat+Satellite\%29} \\
        spambase & \url{https://archive.ics.uci.edu/ml/datasets/spambase} \\
        electrical & \url{https://archive.ics.uci.edu/ml/datasets/Electrical+Grid+Stability+Simulated+Data+} \\
        occupancy & \url{https://archivehttps://archive.ics.uci.edu/ml/datasets/Occupancy+Detection+} \\
        gas & \url{https://archive.ics.uci.edu/ml/datasets/Gas+Sensor+Array+Drift+Dataset} \\
        drybean & \url{https://archive.ics.uci.edu/ml/datasets/Dry+Bean+Dataset} \\
        \hline
        energy & \url{https://archive.ics.uci.edu/ml/datasets/Appliances+energy+prediction} \\
        bike & \url{https://archive.ics.uci.edu/ml/datasets/Bike+Sharing+Dataset} \\
        metro & \url{https://archive.ics.uci.edu/ml/datasets/Metro+Interstate+Traffic+Volume} \\
        sgemm & \url{https://archive.ics.uci.edu/ml/datasets/SGEMM+GPU+kernel+performance} \\
        ovctt & \url{https://archive.ics.uci.edu/ml/datasets/Online+Video+Characteristics+and+Transcoding+Time+Dataset} \\
        californiahousing & loaded from \texttt{scikitlearn.datasets} \\
        \bottomrule
    \end{tabular}
    \end{sideways}
    \caption{The URLs of all the datasets used in the paper, giving direct download links and supplementary details.}
    \label{tab:data-source}
\end{table}

\subsubsection{Data corruption}

For a given corruption rate $\eta$, we obtain a corrupted version of a dataset by replacing an $\eta$-fraction of its samples with uninformative elements. 
For a dataset of size $n$ we choose $\cO \subset \setint{n}$ which satisfies $|\cO| = \eta n$ up to integer rounding. 
The corruption is applied prior to any preprocessing except in the regression case where label scaling is applied before. 
The affected subset is chosen uniformly at random. 
Since many datasets contain both continuous and categorical data features, we distinguish two different corruption mechanisms which we apply depending on their nature. 
The labels are corrupted as continuous or categorical values when the task is respectively regression or classification.
Denote $\widetilde{\bX} \in \R^{n\times (d+1)}$ the data matrix with the vector of labels added to its columns. 
Let $\widetilde{J} \subset \setint{d+1}$ denote the index of continuous columns, we compute $\wh \mu_j$ and $\wh \sigma_j$ their empirical means and standard deviations respectively for $j \in \widetilde J$. 
We also sample a random unit vector $u$ of size $|\widetilde{J}|$.
\begin{itemize}
    \item For categorical feature columns, for each corrupted index $i \in \cO$, we replace $\bX_{i,j}$ with a uniformly sampled value among $\{\bX_{\bullet ,j}\}$ i.e. among the possible modalities of the categorical feature in question.
    \item For continuous features, for each corrupted index $i \in \cO$, we replace $\bX_{i, \widetilde{J}}$ with equal probability with one of the following possibilities:
    \begin{itemize}
        \item a vector $\xi$ sampled coordinatewise according to $\xi_j = r_j + 5 \wh\sigma_j \nu $ where $r_j$ is a value randomly picked in the column $\bX_{\bullet ,j}$ and $\nu$ is a sample from the Student distribution with $2.1$ degrees of freedom.
        \item a vector $\xi$ sampled coordinatewise according to $\xi_j = \wh \mu_j + 5\wh\sigma_j u_j + z $ where $z$ is a standard gaussian.
        \item a vector $\xi$ sampled according to $\xi = \wh \mu + 5\wh \sigma \otimes w $ where $w$ is a uniformly sampled unit vector.
    \end{itemize}
\end{itemize}

\subsection{Preprocessing}

We apply a minimal amount of preprocessing to the data before applying the considered learning algorithms. More precisely, categorical features are one-hot encoded while centering and standard scaling is applied to the continuous features.

\subsection{Parameter hyper-optimization}

We use the \texttt{hyperopt} library to find optimal hyper-parameters for all algorithms. 
For each dataset, the available samples are split into training, validation and test sets with proportions $70\%, 15\%, 15\%$. 
Whenever corruption is applied, it is restricted to the training set. 
We run 50 rounds of hyper-parameter optimization which are trained on the training set and evaluated on the validation set. 
Then, we report results on the test set for all hyper-optimized algorithms. 
For each algorithm, the hyper-parameters are tried out using the following sampling mechanism (the one we specify to \texttt{hyperopt}):
\begin{itemize}
    \item $\mathtt{MOM}$, $\mathtt{GMOM}$, $\mathtt{LLM}$: we optimize the number of blocks $K$ used for the median-of-means computations. This is done through a \texttt{block\_size} $=K/n$ hyper-parameter chosen with log-uniform distribution over $[10^{-5}, 0.2]$
    \item $\mathtt{CH}$ and $\mathtt{CH\: GD}$:  we optimize the confidence $\delta$ used to define the $\mathtt{CH}$ estimator's scale parameter (see Equation~\eqref{eq:ch-scale-estimator}) chosen with log-uniform distribution over $[e^{-10}, 1]$
    \item $\mathtt{TM}$, $\mathtt{HG}$: we optimize the percentage used for trimming uniformly in $[10^{-5}, 0.3]$
    \item $\mathtt{RANSAC}$: we optimize the value of the \texttt{min\_samples} parameter in the scikit-learn implementation, chosen as $4 + m$ with $m$ an integer chosen uniformly in $\setint{100}$
    \item $\mathtt{HUBER}$: we optimize the \texttt{epsilon} parameter in the scikit-learn implementation chosen uniformly in $[1.0, 2.5]$
\end{itemize}

\section{Proofs}
\label{sec:proofs}

\subsection{Proof of Theorem~\ref{thm:linconv1expect}}

This proof follows, with minor modifications, the proof of Theorem~1 from~\cite{wright2015coordinate}.
Using Definition~\ref{def:error-vector} , we obtain
\begin{equation}
    \label{eq:approx}
    \P [ \mathcal E ] \geq 1 - \delta \quad \text{where} \quad \mathcal E:= \big\{ \forall j \in  \setint d, \quad \forall t \in [T], \quad 
    \big| \widehat{g}_j(\thetat) - g_j(\thetat)\big| \leq \epsilon_j(\delta ) \big\}.
\end{equation}
Let us recall that $e_j$ stands for the $j$-th canonical basis of $\R^d$ and that, as described in Algorithm~\ref{alg:robust-cgd}, we have
\begin{equation*}
    \theta^{(t+1)} = \theta^{(t)} - \beta_{j_t} \widehat{g}_t e_{j_t},
\end{equation*}
where we use the notations $\widehat{g}_t = \hg_{j_t}(\theta^{(t)})$ and $g_t = g_{j_t}(\theta^{(t)})$ and where we recall that $j_1, \ldots, j_t$ is a i.i.d sequence with distribution $p$.
We introduce also $\epsilon_j:= \epsilon_j(\delta)$.
Using Assumption~\ref{ass:mintheta-and-smoothness}, we obtain
\begin{align}
    R(\theta^{(t+1)}) &= R\big(\theta^{(t)} - \beta_{j_t} \widehat{g}_t e_{j_t}\big) \nonumber \\
    &\leq R(\theta^{(t)}) - \big\langle g(\theta^{(t)}), \beta_{j_t} \widehat{g}_t e_{j_t} \big\rangle + \frac{L_{j_t}}{2}\beta_{j_t}^2\widehat{g}_t^2 \nonumber \\
    &= R(\theta^{(t)}) - \beta_{j_t} g_t^2 - \beta_{j_t} g_t ( \widehat{g}_t - g_t ) + \frac{L_{j_t}\beta_{j_t}^2}{2}\big(g_t^2 + (\widehat{g}_t - g_t)^2 + 2g_t(\widehat{g}_t - g_t)\big) \nonumber \\
    &= R(\theta^{(t)}) - \beta_{j_t} g_t (1 - L_{j_t}\beta_{j_t})( \widehat{g}_t - g_t ) - 
    \beta_{j_t} \Big( 1 - \frac{L_{j_t}\beta_{j_t}}{2} \Big) g_t^2  + \frac{L_{j_t}\beta_{j_t}^2}{2}(\widehat{g}_t - g_t)^2 \nonumber \\
    &= R(\theta^{(t)}) - \frac{1}{2 L_{j_t}}g_t^2  + \frac{1}{2L_{j_t}}(\widehat{g}_t - g_t)^2 \nonumber \\ 
    &\leq R(\theta^{(t)}) - \frac{1}{2 L_{j_t}}g_t^2  + \frac{\epsilon_{j_t}^2}{2L_{j_t}} \label{eq:descentinequality}
\end{align}
on the event $\cE$, where we used the choice $\beta_{j_t} = 1/L_{j_t}$ and the fact that $| \widehat{g}_t - g_t | \leq \epsilon_{j_t}$ on $\mathcal E$.

Since $j_1, \ldots, j_t$ is a i.i.d sequence with distribution $p$, we have for any $(j_1, \ldots, j_{t-1})$-measurable and integrable function $\varphi$ that
\begin{equation*}
    \E_{t-1}\big[\varphi(j_t)\big] = \sum_{j \in  \setint d} \varphi(j) p_j,
\end{equation*}
where we denote for short the conditional expectation $\E_{t-1}[\cdot] = \E_{t-1}[\cdot | j_1, \ldots, j_{t-1}]$.
So, taking $\E_{t-1}[\cdot]$ on both sides of~\eqref{eq:descentinequality} leads, whenever $p_j = L_j / \sum_{k=1}^d L_k$, to
\begin{equation*}
    \E_{t-1}\big[ R(\theta^{(t+1)})\big] \leq R(\theta^{(t)}) - \frac{1}{2 \sum_k L_k}\big\|g(\theta^{(t)})\big\|^2 + \frac{1}{2\sum_k L_k} \Xi,
\end{equation*}
where we introduced $\Xi:= \|\epsilon(\delta)\|_2^2$, while it leads to
\begin{equation*}
    \E_{t-1} \big[R(\theta^{(t+1)})\big] \leq R(\theta^{(t)}) 
    - \frac{1}{2 L_{\max}d} \big\|g(\theta^{(t)})\big\|^2 + \frac{1}{2d L_{\min}} \Xi
\end{equation*}
whenever $p_j = 1 / d$, simply using $L_{\min} \leq L_j \leq L_{\max}$.
In order to treat both cases simultaneously, consider $\bar L = \sum_{k=1} L_k$ and $\bar \epsilon = \Xi / (2 \sum_k L_k)$ whenever $p_j = L_j / \sum_{k=1}^d L_k$ and $\bar L = d L_{\max}$ and $\bar \epsilon / (2 d L_{\min})$ whenever $p_j = 1 / d$ and continue from the inequality
\begin{equation*}
    \E_{t-1}\big[ R(\theta^{(t+1)})\big] \leq R(\theta^{(t)}) - \frac{1}{2 \bar L} \big\|g(\theta^{(t)})\big\|^2 + \bar \epsilon.    
\end{equation*}
Introducing $\phi_t:= \E \big[R(\theta^{(t)})\big] - R^\star$ and taking the expectation w.r.t. all $j_1, \dots, j_t$ we obtain
\begin{equation}
    \label{eq:descent}
    \phi_{t+1} \leq \phi_{t} - \frac{1}{2 \bar L} \E \big\| g(\theta^{(t)}) \big\|^2 + \bar \epsilon.
\end{equation}
Using Inequality~\eqref{eq:strongconvexity} with $\theta_1 = \theta^{(t)}$ gives
\begin{equation*}
    R(\theta_2) \geq R(\theta^{(t)}) + 
    \big \langle \nabla R(\theta^{(t)}), \theta_2 - \theta^{(t)} \big\rangle 
    + \frac{\lambda}{2}\big\|\theta_2 - \theta^{(t)}\big\|^2
\end{equation*}
for any $\theta_2 \in \R^d$, so that by minimizing both sides with respect to $\theta_2$ leads to
\begin{equation*}
    R^\star \geq R(\theta^{(t)}) - \frac{1}{2\lambda}\big\|g(\theta^{(t)})\big\|^2
\end{equation*}
namely
\begin{equation*}
    \phi_t \leq \frac{1}{2\lambda} \E \big\|g(\theta^{(t)})\big\|^2,
\end{equation*}
by taking the expectation on both sides.
Together with~\eqref{eq:descent} this leads to the following approximate contraction property:
\begin{equation*}
    \phi_{t+1} \leq \phi_{t} \Big(1 - \frac{\lambda}{\bar L} \Big) + \bar \epsilon,
\end{equation*}
and by iterating $t=1, \ldots, T$ to
\begin{equation*}
    \phi_T \leq \phi_0 \Big( 1 - \frac{\lambda}{\bar {L}} \Big)^T + \frac{\bar {\epsilon}\bar {L}}{\lambda},
\end{equation*}
which allows to conclude the Proof of Theorem~\ref{thm:linconv1expect}. $\hfill \square$

\subsection{Proof of Theorem~\ref{thm:linconvdeterministic}}

This proof reuses ideas from~\cite{li2017faster} and~\cite{doi:10.1137/120887679} and adapts them to our context where the gradient coordinates are replaced with high confidence approximations.
Without loss of generality, we initially assume that the coordinates are cycled upon in the natural order. We condition on the event~(\ref{eq:approx}) which holds with probability $\geq 1 - \delta$ as in the proof of Theorem~\ref{thm:linconv1expect} and denote $\epsilon_j = \epsilon_j(\delta)$ and $\epsilon_{Euc} = \|\epsilon(\delta)\|$.

Let the iterations be denoted as $\theta^{(t)}$ for $t=0,\dots, T$ and $\theta^{(t)}_{i+1} = \theta^{(t)}_{i} - \beta_{i+1} \widehat{g}(\theta^{(t)}_{i})_{i+1} e_{i+1}$ for $i=0, \dots, d-1$ with $\beta_i = 1/L_i$, $\theta^{(t)}_0 = \theta^{(t)}$ and $\theta^{(t)}_d = \theta^{(t+1)}$. With these notations we have 
\begin{equation*}
    R(\theta^{(t)}) - R(\theta^{(t+1)}) = \sum_{i=0}^{d-1} R(\theta^{(t)}_i) - R(\theta^{(t)}_{i+1}).
\end{equation*}
Similarly to (\ref{eq:descentinequality}) in the proof of Theorem~\ref{thm:linconv1expect} we find:
\begin{equation*}
    R(\theta^{(t)}_i) - R(\theta^{(t)}_{i+1}) \geq \frac{1}{2L_{i+1}}\big( g(\theta^{(t)}_{i})_{i+1}^2 - \epsilon_{i+1}^2\big),
\end{equation*}
leading to
\begin{equation}\label{eq:decr}
    R(\theta^{(t)}) - R(\theta^{(t+1)}) \geq \sum_{i=0}^{d-1} \frac{1}{2L_{i+1}} g(\theta^{(t)}_{i})_{i+1}^2 - \frac{1}{2 L_{\min}} \sum_{i=0}^{d-1}\epsilon_{i+1}^2 .
\end{equation}
The following aims to find a relationship between $\sum_{i=0}^{d-1} \frac{1}{2L_{i+1}} g(\theta^{(t)}_{i})_{i+1}^2$ and $\big\|g(\thetat)\|_2^2$ which we do by comparing coordinates. For the first step in a cycle we have $g(\thetat)_1 = g(\thetat_0)_1$ because $\thetat = \thetat_0$. Let $j \in \{1,\dots, d-1\}$, by the Mean Value Theorem, there exists $\gamma^{(t)}_j \in \R^d$ such that we have:
\begin{align*}
    g(\thetat)_{j+1} &= g(\thetat)_{j+1} - g(\thetat_j)_{j+1} + g(\thetat_j)_{j+1} \\
    &= \big(\nabla g_{j+1}(\gamma^{(t)}_j)\big)^\top \big(\thetat - \thetat_j\big) + g(\thetat_j)_{j+1} \\
    &= \bigg[\frac{\partial R(\gamma^{(t)}_j)}{\partial_{j+1}\partial_1}, \dots, \frac{\partial R(\gamma^{(t)}_j)}{\partial_{j+1}\partial_j}, 0, \dots, 0\bigg] \big[ (\thetat - \thetat_j)_1, \dots, (\thetat - \thetat_j)_j, 0, \dots, 0 \big]^\top \\ & \quad + g(\thetat_j)_{j+1} \\
    &= [H_{j+1, 1}, \dots, H_{j+1, j}, 0, \dots, 0] \bigg[ \frac{\hg_1(\thetat_0)}{L_1}, \dots, \frac{\hg_j(\thetat_{j-1})}{L_j}, 0, \dots, 0 \bigg]^\top + g(\thetat_j)_{j+1} \\
    &= [H_{j+1, 1}, \dots, H_{j+1, j}, 0, \dots, 0] \bigg[ \frac{g_1(\thetat_0) + \delta^{(t)}_{1}}{L_1}, \dots, \frac{g_j(\thetat_{j-1}) + \delta^{(t)}_{j}}{L_j}, 0, \dots, 0 \bigg]^\top \\ & \quad + g(\thetat_j)_{j+1} \\
    &= \underbrace{\bigg[\frac{H_{j+1, 1}}{\sqrt{L_1}}, \dots, \frac{H_{j+1, j}}{\sqrt{L_j}}, \sqrt{L_{j+1}}, 0, \dots, 0\bigg]}_{\widetilde{h}^\top_{j+1}} \underbrace{\bigg[ \frac{g_1(\thetat_0)}{\sqrt{L_1}}, \dots, \frac{g_d(\thetat_{d-1})}{\sqrt{L_d}} \bigg]^\top}_{\widetilde{g}_t} \\ & \quad + \underbrace{[H_{j+1, 1}, \dots, H_{j+1, j}, 0, \dots, 0]}_{h^\top_{j+1}} \bigg[ \frac{\delta^{(t)}_{1}}{L_1}, \dots, \frac{\delta^{(t)}_{d}}{L_d}\bigg]^\top \\
    &= \widetilde{h}_{j+1} \widetilde{g}_t + h_{j+1} A^{-1}\delta^{(t)},
\end{align*}
where we introduced the following quantities: $A \in \R^d$ equal to $A = \diag(L_j)_{j=1}^d$, the vector $\delta^{(t)} \in \R^d$ is such that $\delta^{(t)}_j = \hg(\thetat_{j-1})_{j} - g(\thetat_{j-1})_{j}$ which satisfies $|\delta^{(t)}_{j}| \leq \epsilon_j$, the matrix $H = (h_1, \dots, h_d)^\top $ and $\widetilde{H} = A^{1/2} + H A^{-1/2} = ( \widetilde{h}_1, \dots , \widetilde{h}_d)^\top$. In the case $j=0$ the vector $h_{j+1} = h_1$ is simply zero. This allows us to obtain the following estimation:
\begin{align}
    \big\|g(\thetat)\big\|^2 &= \sum_{j=1}^d g(\thetat)_j^2 = \sum_{j=1}^d (\widetilde{h}_j^\top \widetilde{g}_t + h_j^\top A^{-1} \delta^{(t)})^2 \nonumber  \\ 
    &\leq \sum_{j=1}^d 2(\widetilde{h}_j^\top \widetilde{g}_t)^2 + 2(h_j^\top A^{-1} \delta^{(t)})^2 = 2\big\|\widetilde{H} \widetilde{g}_t\big\|^2 + 2\big\|H A^{-1} \delta^{(t)}\big\|^2 \nonumber \\
    &\leq 2\big\|\widetilde{H}\big\|^2 \big\|\widetilde{g}_t\big\|^2 + \frac{2}{L_{\min}^2}\|H\|^2 \epsilon_{Euc}^2 \nonumber \\
    &= 2\|\widetilde{H}\|^2 \sum_{i=0}^{d-1} \frac{1}{L_{i+1}} g(\theta^{(t)}_{i})_{i+1}^2 + \frac{2}{L_{\min}^2}\|H\|^2 \epsilon_{Euc}^2 \label{grad_norm_ineq}.
\end{align}
We can bound the spectral norm $\|\widetilde{H}\|$ as follows:
\begin{equation*}
    \|{\widetilde{H}\|^2 = \|A^{1/2} + H A^{-1/2}\|^2 \leq 2\|A^{1/2}}^2 + 2\|H A^{-1/2}\|^2 \leq 2\Big(L_{\max} + \frac{\|H\|^2}{L_{\min}}\Big).
\end{equation*}
For $\|H\|$, we use the coordinate-wise Lipschitz-smoothness in order to find
\[\|H\|^2 \leq \|H\|_F^2 = \sum_{j=1}^d \|h_j\|^2 \leq \sum_{j=1}^d \big\|\nabla g_j(\gamma^{(t)}_{j-1})\big\|^2 \leq \sum_{j=1}^d L_j^2 \leq d L_{\max}^2. \]
Combining the previous inequality with (\ref{eq:decr}) and (\ref{grad_norm_ineq}), we find:
\begin{align*}
    R(\theta^{(t)}) - R(\theta^{(t+1)}) &\geq \frac{1}{8L_{\max}(1 + d\frac{L_{\max}}{L_{\min}})}\big\|g(\theta^{(t)})\big\|^2 - \frac{\epsilon_{Euc}^2}{2} \Big((\frac{1}{L_{\min}} + \frac{d\Big(\frac{L_{\max}}{L_{\min}}\Big)^2}{2L_{\max}(1 + d\frac{L_{\max}}{L_{\min}})}\Big) \\
    &\geq \frac{1}{8L_{\max}(1 + d\frac{L_{\max}}{L_{\min}})}\big\|g(\theta^{(t)})\big\|^2 - \frac{\epsilon_{Euc}^2}{2} \Big(\frac{1}{L_{\min}} + \frac{1}{2L_{\min}}\frac{dL_{\max}/L_{\min}}{1 + d\frac{L_{\max}}{L_{\min}}}\Big) \\
    &\geq \underbrace{\frac{1}{8L_{\max}(1 + d\frac{L_{\max}}{L_{\min}})}}_{=: \kappa}\big\|g(\theta^{(t)})\big\|^2 - \frac{3}{4L_{\min}} \epsilon_{Euc}^2,
\end{align*}
where the last step uses that $\frac{dL_{\max}/L_{\min}}{1 + d\frac{L_{\max}}{L_{\min}}} \leq 1$. Using $\lambda$-strong convexity by choosing $\theta_1 = \theta^{(t)}$ in inequality (\ref{eq:strongconvexity}) and minimizing both sides w.r.t. $\theta_2$ we obtain:
\[R(\theta^{(t)}) - R^\star \leq \frac{1}{2\lambda}\|g(\theta^{(t)})\|^2, \]
which combined with the previous inequality yields the contraction inequality:
\[
R(\theta^{(t+1)}) - R^\star \leq (R(\theta^{(t)}) - R^\star)(1 - 2\lambda \kappa) + \frac{3}{4L_{\min}} \epsilon_{Euc}^2,
\]
and after $T$ iterations we have:
\[
R(\theta^{(T)}) - R^\star \leq (R(\theta^{(0)}) - R^\star)(1 - 2\lambda \kappa)^T + \frac{3\epsilon_{Euc}^2}{8 L_{\min}\lambda \kappa},
\]
which concludes the proof of Theorem~\ref{thm:linconvdeterministic}.
To see that the proof still holds for any choice of coordinates satisfying the conditions in the main claim, notice that the computations leading up to Inequality~(\ref{grad_norm_ineq}) work all the same if one were to apply a permutation to the coordinates beforehand.

\subsection{Convergence of the parameter error}

We state and prove a result about the linear convergence of the parameter under strong convexity.

\begin{theorem}
    \label{thm:linconvparam}
    Grant Assumptions~\ref{assump:lipsmoothloss}\textup, \ref{ass:mintheta-and-smoothness} and~\ref{ass:strongconvexity}. 
    Let $\theta^{(T)}$ be the output of Algorithm~\ref{alg:robust-cgd} with constant step-size $\beta = \frac{2}{\lambda + L},$ an initial iterate $\theta^{(0)},$ uniform coordinates sampling $p_j = 1 / d$ and estimators of the partial derivatives with error vector $\epsilon(\cdot)$. 
    Then\textup, we have 
    \begin{equation}
        \label{eq:thm4-uniform}
        \E \big\|\theta^{(T)} - \theta^\star \big\|_2 \leq 
        \big\|\theta^{(0)} - \theta^\star\big\|_2 \Big(1 - \frac{2\beta \lambda L}{d(\lambda + L)}\Big)^T 
        + \frac{\sqrt{d}(\lambda + L)}{\lambda L } \big\| \epsilon ( \delta ) \big\|_2
    \end{equation}
    with probability at least $1 - \delta$, where the expectation is w.r.t. the sampling of the coordinates.
\end{theorem}

\begin{proof}
As in the proof of Theorem~\ref{thm:linconv1expect}, let $(\widehat{g}_j(\theta))_{j=1}^d$ be the estimators used and introduce the notations 
\begin{equation*}
    \widehat{g}_t = \hg_{j_t}(\theta^{(t)}) \quad \text{ and } \quad g_t = g_{j_t}(\theta^{(t)}).
\end{equation*}
We also condition on the event~(\ref{eq:approx}) which holds with probability $1 - \delta$ and use the notations $\epsilon_{Euc} = \|\epsilon(\delta)\|_2$ and $\epsilon_j = \epsilon_j(\delta)$. We denote $\|\cdot\|_{L_2}$ the $L_2$-norm w.r.t. the distribution over $j_t$ i.e. for a random variable $\xi$ we have $\|\xi\|_{L_2} = \sqrt{\E_{j_t} \|\xi\|^2}$. We compute:
\begin{equation}
\label{ineq:th2}
    \big\|\theta^{(t+1)} - \theta^\star\big\|_{L_2} = \big\|\theta^{(t)} - \beta_{j_t} \widehat{g}_t e_{j_t} - \theta^\star\big\|_{L_2} \leq \big\|\theta^{(t)} - \beta_{j_t} g_t e_{j_t} - \theta^\star\big\|_{L_2} + \big\|\beta_{j_t}(\widehat{g}_t - g_t)\big\|_{L_2}.
\end{equation}
We first treat the first term of~\eqref{ineq:th2}, in the case of uniform sampling with equal step-sizes $\beta_j = \beta$ we have:
\begin{equation*}
    \big\|\theta^{(t)} - \beta g_t e_{j_t} - \theta^\star\big\|^2 = \big\|\theta^{(t)} - \theta^\star\big\|^2 + \beta^2 g_t^2 - 2\beta \big\langle g_t e_{j_t},\theta^{(t)} - \theta^\star\big\rangle.
\end{equation*}
By taking the expectation w.r.t. the random coordinate $j_t$ we find:
\begin{align*} &\big\|\theta^{(t)} - \beta g_t e_{j_t} - \theta^\star\big\|_{L_2}^2 = \E \big\|\theta^{(t)} - \beta g_t e_{j_t} - \theta^\star\big\|^2 \\
&=\E\big\|\theta^{(t)} - \theta^\star\big\|^2 +\frac{\beta^2}{d} \E\big\|g(\theta^{(t)})\big\|^2 - 2\frac{\beta}{d} \E\big\langle g(\theta^{(t)}), \theta^{(t)} - \theta^\star\big\rangle \\
&=\E\|\theta^{(t)} - \theta^\star\|^2 +\Big(\frac{\beta}{d}\Big)^2 \E\|g(\theta^{(t)})\|^2 - 2\frac{\beta}{d} \E\big\langle g(\theta^{(t)}), \theta^{(t)} - \theta^\star\big\rangle + \frac{\beta^2}{d} \E\big\|g(\theta^{(t)})\big\|^2\Big(1 - \frac{1}{d}\Big) \\
    &\leq \E\big\|\theta^{(t)} - \theta^\star\big\|^2\Big(1 - \frac{2\beta \lambda L}{d(\lambda + L)}\Big) + \frac{\beta}{d} \Big(\frac{\beta}{d} - \frac{2}{\lambda + L}\Big) \E\big\|g(\theta^{(t)})\big\|^2 + \frac{\beta^2}{d} \E\big\|g(\theta^{(t)})\big\|^2\Big(1 - \frac{1}{d}\Big) \\
    &= \E\big\|\theta^{(t)} - \theta^\star\big\|^2\Big(1 - \frac{2\beta \lambda L}{d(\lambda + L)}\Big) + \frac{\beta}{d} \Big(\beta - \frac{2}{\lambda + L}\Big)\E\big\|g(\theta^{(t)})\big\|^2 \\
    &\leq \E\big\|\theta^{(t)} - \theta^\star\big\|^2\underbrace{\Big(1 - \frac{2\beta \lambda L}{d(\lambda + L)}\Big)}_{=:\kappa^2}.
\end{align*}
The first inequality is obtained by applying inequality (2.1.15) from~\cite{Nesterov} (see also~\cite{bubeck2015convex} Lemma 3.11) and the second one is due to the choice of $\beta$. We can bound the second term as follows:
\[
\big\|\widehat{g}_t - g_t\big\|^2_{L_2} = \E_{j_t}\big|\widehat{g}_t - g_t\big|^2 = \frac{1}{d} \sum_{j=1}^d\big|\widehat{g}_j(\thetat) - g_j(\thetat)\big|^2 \leq \frac{\epsilon^2_{Euc}}{d}.
\]
Combining the latter with the former bound, we obtain the approximate contraction:
\[ 
\big\|\theta^{(t+1)} - \theta^\star\big\|_{L_2} \leq \kappa \big\|\theta^{(t)} - \theta^\star\big\|_{L_2} + \frac{\beta \epsilon_{Euc}}{\sqrt{d}}.
\]
By iterating this argument on $T$ rounds we find that:
\[ 
\big\|\theta^{(T)} - \theta^\star\big\|_{L_2} \leq \kappa^T\big\|\theta^{(0)} - \theta^\star\big\|_{L_2} + \frac{\beta \epsilon_{Euc}}{\sqrt{d}(1 - \kappa)}.
\]
Finally, the following inequality yields the result in the case of uniform sampling:
\[
\frac{1}{1-\kappa} \leq \frac{1 + \sqrt{1 - \frac{2\beta \lambda L}{d(\lambda + L)}}}{\frac{2\beta \lambda L}{d(\lambda + L)}} \leq \frac{d(\lambda + L)}{\beta \lambda L}.
\qedhere
\]
\end{proof}

\subsection{Proof of Lemma~\ref{lem:partial-deriv-moments}}

Let $\theta \in \Theta$, using Assumption~\ref{assump:lipsmoothloss} we have:
\begin{equation*}
    |\ell(\theta^{\top} X, Y)| \leq C_{\ell, 1} + C_{\ell, 2}|\theta^{\top} X- Y|^{q} \leq C_{\ell, 1} + 2^{q-1} C_{\ell, 2}(|\theta^{\top} X|^{q} + |Y|^{q}).
\end{equation*}
Taking the expectation and using Assumption~\ref{assump:data} shows that the risk $R(\theta)$ is well defined (recall that $q\leq 2$). Next, since $1\leq q \leq 2$, simple algebra gives
\begin{align*}
    \big| \ell' (\theta^{\top} X, Y) X_j \big|^{1 + \alpha} 
    &\leq \big| \big(C_{\ell,1}' + C_{\ell,2}'|\theta^\top X - Y|^{q-1}\big)X^j \big|^{1 + \alpha}  \\
    &\leq 2^\alpha \big(\big|C_{\ell,1}' X^j\big|^{1+\alpha}  + (C_{\ell,2}' (|(\theta^\top X)^{q-1} X^j| + |Y^{q-1} X^j|))^{1 + \alpha}\big)  \\
    &\leq 2^\alpha \Big(\big|C_{\ell,1}' X^j\big|^{1+\alpha}  + \Big(C_{\ell,2}' \Big(\sum_{k=1}^d|\theta_k|^{q-1} |(X^k)^{q-1} X^j| + |Y^{q-1} X^j|\Big)\Big)^{1 + \alpha}\Big)  \\
    &\leq 2^\alpha \Big(\big|C_{\ell,1}' X^j\big|^{1+\alpha}\\
    &+ 2^\alpha (C_{\ell,2}')^{1+\alpha} \Big(d^\alpha \sum_{k=1}^d|\theta_k|^{(q-1)(1+\alpha)} |(X^k)^{q-1} X^j|^{1+\alpha} + |Y^{q-1} X^j|^{1 + \alpha}\Big)\Big).
\end{align*}
Given Assumption~\ref{assump:data}, it is straightforward that $\E|X^j\big|^{1+\alpha} < \infty$ and $\E |Y^{q-1} X^j|^{1 + \alpha}< \infty$. Moreover, using a Hölder inequality with exponents $a = \frac{q(1+\alpha)}{(q-1)(1+\alpha)}$ and $b = q$ (the case $q=1$ is trivial) we find:
\begin{align*}
    \E\big|(X^k)^{q-1} X^j\big|^{1+\alpha}\leq \big(\E \big|X^k\big|^{q(1+\alpha)}\big)^{1/a} \big(\E \big|X^j\big|^{q(1+\alpha)}\big)^{1/b},
\end{align*}
which is finite under Assumption~\ref{assump:data}. 
This concludes the proof of Lemma~\ref{lem:partial-deriv-moments}.

\subsection{Proof of Lemma~\ref{lem:basicMOM}}

This proof follows a standard argument from~\cite{lugosi2019mean,geoffrey2020robust} in which we use a Lemma from~\cite{bubeck2013bandits} in order to control the $(1+\alpha)$-moment of the block means instead of their variance.
Indeed, we know from Lemma~\ref{lem:partial-deriv-moments} that under Assumptions~\ref{assump:lipsmoothloss} and~\ref{assump:data}, the gradient coordinates have finite $(1+\alpha)$-moments, namely
$\E[| \ell'(X^\top\theta, Y) X_j |^{1 + \alpha} ] < +\infty$ for any $j \in  \setint d$.
Recall that $(\wh g_j^{(k)}(\theta))_{k \in \setint{K}}$ stands for the block-wise empirical mean given by Equation~\eqref{eq:mom-block-mean} and introduce the set of non-corrupted block indices given by $\cK = \{ k \in \setint{K} \;: \; B_k \cap \mathcal{O} = \emptyset \}$. We will initially assume that the number of outliers satisfies $|\cO| \leq (1 - \varepsilon) K / 2$ for some $0 < \varepsilon < 1$. 
Note that since samples are i.i.d in $B_k$ for $k \in \cK$, we have $\E\big[\wh g_j^{(k)}(\theta)\big] = g_j(\theta)$.
We use the following Lemma from~\cite{bubeck2013bandits}.

\begin{lemma}[Lemma~3 from~\cite{bubeck2013bandits}]
    \label{lem:mean-deviation-weak-moments}
    Let $Z, Z_1, \ldots, Z_n$ be a i.i.d sequence with $m_\alpha = \E[| Z - \E Z|^{1 + \alpha}] < +\infty$ for some $\alpha \in (0, 1]$ and put $\bar Z_n = \frac 1n \sum_{i \in \setint n} Z_i$.
    Then\textup, we have
    \begin{equation*}
        \bar Z_n \leq \E Z + \Big( \frac{3 m_\alpha}{\delta n^{\alpha}} \Big)^{1 / (1 + \alpha)}
    \end{equation*}
    for any $\delta \in (0, 1),$ with a probability $1 - \delta$.
\end{lemma}
Lemma~\ref{lem:mean-deviation-weak-moments} entails that
\begin{equation*}
    \big| \wh g_j^{(k)}(\theta) - g_j(\theta)\big| \leq \Big( \frac{3 m_{j, \alpha}(\theta)}{\delta' (n / K)^\alpha} \Big)^{1 / (1 + \alpha)} =: \eta_{j, \alpha, \delta'}(\theta)
\end{equation*}
with probability larger than $1 - 2 \delta'$, for each $k \in \cK$, since we have $n / K$ samples in block $B_k$.
Now, recalling that $\wh g_j(\theta)$ is the median (see~\eqref{eq:mom-estimator}), we can upper bound its failure probability as follows:
\begin{align*}
    \P \Big[ \big| \wh g_j^{\mathtt{MOM}}(\theta) - g_j(\theta) \big| \geq \eta_{j, \alpha, \delta'}(\theta) \Big] 
    &\leq \P \bigg[ \sum_{k \in \setint{K}} \ind{}\Big\{\big| \wh g_j^{(k)}(\theta) - g_j(\theta)\big| \geq \eta_{j, \alpha, \delta'}(\theta)\Big\} > K / 2 \bigg] \\
    &\leq \P \bigg[ \sum_{k \in \cK} \ind{}\Big\{\big| \wh g_j^{(k)}(\theta) - g_j(\theta)\big| \geq \eta_{j, \alpha, \delta'}(\theta)\Big\} > K / 2 - |\cO| \bigg],
\end{align*}
since at most $|\cO|$ blocks contain one outlier.
Since the blocks $B_k$ are disjoint and contain i.i.d samples for $k \in \cK$, we know that
\begin{equation*}
    \sum_{k \in \cK} \ind{}\Big\{\big| \wh g_j^{(k)}(\theta) - g_j(\theta)\big| \geq \eta_{j, \alpha, \delta'}(\theta)\Big\}  
\end{equation*}
follows a binomial distribution $\text{Bin}(|\cK|, p)$ with $p \leq 2\delta'$.
Using the fact that $\text{Bin}(|\cK|, p)$ is stochastically dominated by  $\text{Bin}(|\cK|, 2\delta')$ and that $\E[\text{Bin}(|\cK|, 2\delta')] = 2\delta'|\cK|$, we obtain, if $S \sim \text{Bin}(|\cK|, 2\delta')$, that
\begin{align*}
    \P \Big[ \big| \wh g_j^{\mathtt{MOM}}(\theta) - g_j(\theta) \big| \geq \eta_{j, \alpha, \delta'}(\theta) \Big] &\leq \P \big[ S > K/2 - |\cO|  \big] \\
    &= \P \big[ S - \E S > K/2 - |\cO| - 2\delta'|\cK| \big] \\
    &\leq \P \big[ S - \E S > K (\varepsilon - 4\delta') / 2 \big] \\
    &\leq \exp \big(-K (\varepsilon - 4\delta')^2 / 2 \big),
\end{align*}
where we used the fact that $|\cO| \leq (1 - \varepsilon) K / 2$ and $|\cK| \leq K$ for the second inequality and the Hoeffding inequality for the last.
This concludes the proof of Lemma~\ref{lem:basicMOM} for the choice $\varepsilon = 5/6$ and $\delta' = 1/8$.

\subsection{Proof of Proposition~\ref{prop:uniformMOM}}

\paragraph{Step 1.}

First, we fix $\theta \in \Theta$ and try to bound $\big|\wh{g}_j^{\mathtt{MOM}}(\theta) - g_j(\theta)\big|$ in terms of quantities only depending on $\widetilde{\theta}$ which is the closest point to $\theta$ in an $\varepsilon$-net. 
Recall that $\Delta$ is the diameter of the parameter set $\Theta$ and let $\varepsilon > 0$ be a positive number. There exists an $\varepsilon$-net covering $\Theta$ with cardinality no more than $(3\Delta/2\varepsilon)^d$ i.e. a set $N_{\varepsilon}$ such that for all $\theta \in \Theta$ there exists $\widetilde{\theta} \in N_{\varepsilon}$ such that $\|\widetilde{\theta} - \theta\| \leq \varepsilon$. Consider a fixed $\theta \in \Theta$ and $j\in \setint{d}$, we wish to bound the quantity $\big|\wh{g}_j^{\mathtt{MOM}}(\theta) - g_j(\theta)\big|$. Using the $\varepsilon$-net $N_{\varepsilon}$, there exists $\widetilde{\theta}$ such that $\|\widetilde{\theta} - \theta\| \leq \varepsilon$ which we can use as follows:
\begin{align}
    \big|\wh{g}_j^{\mathtt{MOM}}(\theta) - g_j(\theta)\big| &\leq \big|\wh{g}_j^{\mathtt{MOM}}(\theta) - g_j(\widetilde{\theta})\big| + \big|g_j(\widetilde{\theta}) - g_j(\theta)\big| \nonumber \\
    &\leq \big|\wh{g}_j^{\mathtt{MOM}}(\theta) - g_j(\widetilde{\theta})\big| + L_j \varepsilon,
    \label{eq:momstab1}
\end{align}
where we used the gradient's coordinate Lipschitz constant to bound the second term. We now focus on the second term. Introducing the notation $g_j^i(\theta) = \ell'(\theta^\top X_i, Y_i)X_i^j$, we have 
\begin{equation*}
    g_j^i(\theta) = \ell'(\widetilde{\theta}^{\top} X_i, Y_i)X_i^j + \underbrace{(\ell'(\theta^{ \top} X_i, Y_i) - \ell'(\widetilde{\theta}^{\top} X_i, Y_i))X_i^j}_{=:\Delta_i}.
\end{equation*}
Let $(B_k)_{k\in\setint{K}}$ be the blocks used to compute the $\mathtt{MOM}$ estimator and associated block means $\wh g_j^{(k)}(\theta)$ and $\wh g_j^{(k)}(\widetilde{\theta})$. 
Notice that the $\mathtt{MOM}$ estimator is \emph{monotonous} non decreasing w.r.t. to each of the entries $g_j^i(\theta)$ when the others are fixed. Without loss of generality, assume that $\wh g_j^{\mathtt{MOM}}(\theta) - g_j(\widetilde{\theta}) \geq 0$ then we have:
\begin{equation}\label{eq:momstab2}
    \big|\wh g_j^{\mathtt{MOM}}(\theta) - g_j(\widetilde{\theta})\big| \leq \big|\widecheck{g}_j^{\mathtt{MOM}}(\widetilde{\theta}) - g_j(\widetilde{\theta})\big|,
\end{equation}
where $\widecheck{g}_j^{\mathtt{MOM}}(\widetilde{\theta})$ is the $\mathtt{MOM}$ estimator obtained using the entries $\ell'\big(\widetilde{\theta}^{\top} X_i, Y_i\big)X_i^j + \varepsilon \gamma \|X_i\|^2 = g_j^i(\widetilde{\theta}) + \varepsilon \gamma \|X_i\|^2$ instead of $g_j^i(\theta)$. Note that $\widecheck{g}_j^{\mathtt{MOM}}(\widetilde{\theta})$ no longer depends on $\theta$ except through the fact that $\widetilde{\theta}$ is chosen in $N_{\varepsilon}$ so that $\big\|\widetilde{\theta} - \theta\big\| \leq \varepsilon$. Indeed, using the Lipschitz smoothness of the loss function and a Cauchy-Schwarz inequality we find that:
\begin{equation*}
    |\Delta_i| \leq \gamma \|\theta - \widetilde{\theta}\|\cdot\|X_i\| \cdot |X_i^j| \leq \varepsilon \gamma \|X_i\|^2.
\end{equation*}

\paragraph{Step 2.}

We now use the concentration property of $\mathtt{MOM}$ to bound the quantity which is in terms of $\widetilde{\theta}$.
The samples $(g_j^i(\widetilde{\theta}) + \varepsilon \gamma \|X_i\|^2)_{i \in \setint{n}}$ are independent and distributed according to the random variable $\ell'(\widetilde{\theta}^{\top} X, Y)X^j + \varepsilon \gamma \|X\|^2$. Denote $\overline{L} = \gamma \E\|X\|^2$ and for $k \in \setint{K}$ let $\wh g_j^{(k)}(\widetilde{\theta}) = \frac{K}{n}\sum_{i\in B_k} g_j^i(\widetilde{\theta})$ and $\wh L^{(k)} = \frac{K}{n}\sum_{i\in B_k} \gamma \|X_i\|^2$. We use Lemma~\ref{lem:mean-deviation-weak-moments} for each of these pairs of means to obtain that with probability at least $1 - \delta'/2$:
\begin{equation*}
\big| \wh g_j^{(k)}(\widetilde{\theta}) - g_j(\widetilde{\theta}) \big| \leq \Big( \frac{6 m_{j, \alpha}(\widetilde{\theta})}{\delta' (n / K)^\alpha} \Big)^{1 / (1 + \alpha)} =: \eta_{j, \alpha, \delta'/2}(\widetilde{\theta}),
\end{equation*}
and with probability at least $1 - \delta'/2$
\begin{equation*}
\big| \wh L^{(k)} - \overline{L} \big| \leq \Big( \frac{6 m_{L, \alpha}}{\delta' (n / K)^\alpha} \Big)^{1 / (1 + \alpha)} =: \eta_{L, \alpha, \delta'/2},
\end{equation*}
where $m_{L, \alpha} = \E | \gamma \|X\|^2 - \overline{L} |^{1+\alpha}$. 
Hence for all $k \in \setint{K}$
\begin{align*}
    \Proba &\big( \big|\wh g_j^{(k)}(\widetilde{\theta}) + \varepsilon \wh L^{(k)} - g_j(\widetilde{\theta}) \big| > \eta_{j, \alpha, \delta'/2}(\widetilde{\theta}) +  \varepsilon (\overline{L} + \eta_{L, \alpha, \delta'/2}) \big) \\ \quad &\leq \Proba \big( \big|\wh g_j^{(k)}(\widetilde{\theta}) - g_j(\widetilde{\theta}) \big| > \eta_{j, \alpha, \delta'/2}(\widetilde{\theta}) \big) + \Proba \big( \big|\wh L^{(k)} - \overline{L} \big| > \eta_{L, \alpha, \delta'/2} \big)\\ \quad &\leq \delta'/2 + \delta'/2 = \delta'.
\end{align*}
Now defining the Bernoulli variables
\begin{equation*}
    U_k:= \ind{}\Big\{\big|\wh g_j^{(k)}(\widetilde{\theta}) + \delta \wh L^{(k)} - g_j(\widetilde{\theta}) \big| > \eta_{j, \alpha, \delta'/2}(\widetilde{\theta}) +  \varepsilon \big(\overline{L} + \eta_{L, \alpha, \delta'/2}\big) \Big\},
\end{equation*} we have just seen they have success probability $\leq \delta'$, moreover
\begin{align*}
    \P \Big[ \big| \widecheck g_j^{\mathtt{MOM}}(\widetilde{\theta}) - g_j(\widetilde{\theta}) \big| \geq \eta_{j, \alpha, \delta'/2}(\widetilde{\theta}) +  \varepsilon (\overline{L} + \eta_{L, \alpha, \delta'/2}) \Big] 
    &\leq \P \bigg[ \sum_{k \in \setint{K}} U_k > K / 2 \bigg] \\
    &\leq \P \bigg[ \sum_{k \in \cK} U_k > K / 2 - |\cO| \bigg],
\end{align*}
since at most $|\cO|$ blocks contain one outlier. Since the blocks $B_k$ are disjoint and contain i.i.d samples for $k \in \cK$, we know that $\sum_{k \in \cK} U_k$ follows a binomial distribution $\text{Bin}(|\cK|, p)$ with $p \leq \delta'$.
Using the fact that $\text{Bin}(|\cK|, p)$ is stochastically dominated by  $\text{Bin}(|\cK|, \delta')$ and that $\E[\text{Bin}(|\cK|, \delta')] = \delta'|\cK|$, we obtain, if $S \sim \text{Bin}(|\cK|, \delta')$, that
\begin{align*}
    \P \Big[ \big| \widecheck g_j^{\mathtt{MOM}}(\widetilde{\theta}) - g_j(\widetilde{\theta}) \big| \geq \eta_{j, \alpha, \delta'/2}(\widetilde{\theta}) +  \varepsilon \big(\overline{L} + \eta_{L, \alpha, \delta'/2}\big) \Big] &\leq \P \big[ S > K/2 - |\cO|  \big] \\
    &= \P \big[ S - \E S > K/2 - |\cO| - \delta'|\cK| \big] \\
    &\leq \P \big[ S - \E S > K (\varepsilon' - 2\delta') / 2 \big] \\
    &\leq \exp \big(-K (\varepsilon' - 2\delta')^2 / 2 \big),
\end{align*}
where we used the condition $|\cO| \leq (1 - \varepsilon') K / 2$ and $|\cK| \leq K$ for the second inequality and the Hoeffding inequality for the last. 
To conclude, we choose $\varepsilon' = 5/6$ and $\delta' = 1/4$ and combine \eqref{eq:momstab1}, \eqref{eq:momstab2} and the last inequality in which we take $K = \lceil 18\log(1/\delta) \rceil$ and use a union bound argument to obtain that with probability at least $1 - \delta$ for all $j \in \setint{d}$
\begin{equation}\label{eq:fixedtheta}
    \big| \widecheck g_j^{\mathtt{MOM}}(\widetilde{\theta}) - g_j(\widetilde{\theta}) \big| \leq \big((24 m_{j, \alpha}(\widetilde{\theta}))^{1/(1+\alpha)} + \varepsilon (24 m_{L, \alpha})^{1/(1+\alpha)} \big) \Big( \frac{18\log(d/\delta)}{ n} \Big)^{\alpha / (1 + \alpha)} +  \varepsilon \overline{L}.
\end{equation}

\paragraph{Step 3.}

We use the $\varepsilon$-net to obtain a uniform bound.
For $\theta \in \Theta$ denote $\widetilde{\theta}(\theta) \in N_{\varepsilon}$ the closest point in $N_{\varepsilon}$ satisfying in particular $\|\widetilde{\theta}(\theta) - \theta\| \leq \varepsilon$, we write, following previous arguments
\begin{align*}
    \sup_{\theta \in \Theta} \big| \wh g_j^{\mathtt{MOM}} (\theta) - g_j(\theta) \big| &\leq \sup_{\theta \in \Theta} \big| \wh g_j^{\mathtt{MOM}} (\theta) - g_j(\widetilde{\theta}(\theta)) \big| + \big| g_j (\widetilde{\theta}(\theta)) - g_j(\theta) \big| \\
    &\leq \sup_{\theta \in \Theta} \big| \widecheck g_j^{\mathtt{MOM}} (\widetilde{\theta}(\theta)) - g_j(\widetilde{\theta}(\theta)) \big| + \varepsilon L_j \\
    &= \max_{\widetilde{\theta} \in N_{\varepsilon}} \big| \widecheck g_j^{\mathtt{MOM}} (\widetilde{\theta}) - g_j(\widetilde{\theta}) \big| + \varepsilon L_{j}.
\end{align*}
Here, we make a union bound argument over $\widetilde{\theta} \in N_{\varepsilon}$ for the inequality \eqref{eq:fixedtheta} and choose $\varepsilon = n^{-\alpha/(1+\alpha)}$ to obtain the final result concluding the proof of Proposition~\ref{prop:uniformMOM}.

\subsection{Proof of Proposition~\ref{prop:RademacherMOM}}

This proof reuses arguments from the proof of Theorem 2 in~\cite{lecue2020robust1}. We wish to bound $\big|\wh g_j^{\mathtt{MOM}}(\theta) - g_j(\theta) \big|$ with high probability and uniformly on $\theta\in \Theta$. 
Fix $\theta\in\Theta$ and $j \in \setint{d}$, we have $\wh g_j^{\mathtt{MOM}}(\theta) = \median \big( \wh g_j^{(1)}(\theta), \dots, \wh g_j^{(K)}(\theta) \big)$ with $\wh g_j^{(k)}(\theta) = \frac{K}{n}\sum_{i\in B_k} g^i_j(\theta)$ where the blocks $B_1, \dots, B_K$ constitute a partition of $\setint{n}$. 

Define the function $\phi (t) = (t-1)\ind{1\leq t \leq 2} + \ind{t > 2}$, let $\cK = \{ k\in \setint{K}, \:\: B_k \cap \cO = \emptyset \}$ and $\mathcal{J} = \bigcup_{k \in \cK} B_k$. Thanks to the inequality $\phi(t) \geq \ind{t \geq 2}$, we have:
\begin{align*}
    &\sup_{\theta\in \Theta} \sum_{k=1}^K \ind{}\Big\{\big|\wh g_j^{(k)}(\theta) - g_j(\theta)\big| > x\Big\} \leq \sup_{\theta\in \Theta} \sum_{k\in \cK} \E \big[ \phi\big(2\big|\wh g_j^{(k)}(\theta) - g_j(\theta)\big| / x\big)\big]  + |\cO|\\
    &+ \sup_{\theta\in \Theta} \sum_{k\in \cK} \Big(\phi\big(2\big|\wh g_j^{(k)}(\theta) - g_j(\theta)\big| / x\big) - \E \big[\phi\big(2\big|\wh g_j^{(k)}(\theta) - g_j(\theta)\big| / x\big)\big]\Big) .
\end{align*}
Besides, the inequality $\phi(t) \leq \ind{t \geq 1}$, an application of Markov's inequality and Lemma~\ref{lem:mean-deviation-weak-moments} yield:
\begin{equation*}
    \E \big[\phi\big(2\big|\wh g_j^{(k)}(\theta) - g_j(\theta)\big| / x\big)\big] \leq \Proba \big( \big|\wh g_j^{(k)}(\theta) - g_j(\theta)\big| \geq x/2 \big) \leq \frac{3 m_{\alpha, j}(\theta)}{(x/2)^{1 + \alpha} (n/K)^{\alpha}}. 
\end{equation*}
Therefore, recalling that we defined $M_{\alpha, j}:= \sup_{\theta \in \Theta} m_{\alpha, j}(\theta)$ we have
\begin{align*}
    \sup_{\theta\in \Theta} &\sum_{k=1}^K \ind{}\Big\{\big|\wh g_j^{(k)}(\theta) - g_j(\theta)\big| > x\Big\} \leq K\bigg(\frac{3 M_{\alpha, j}}{(x/2)^{1 + \alpha} (n/K)^{\alpha}} + \frac{|\cO|}{K} \\
    &+ \sup_{\theta\in \Theta} \frac 1K \Big( \sum_{k\in\cK} \phi\big(2\big|\wh g_j^{(k)}(\theta) - g_j(\theta)\big| / x\big) - \E \big[\phi\big(2\big|\wh g_j^{(k)}(\theta) - g_j(\theta)\big| / x\big)\big] \Big)\bigg).
\end{align*}
Now since for all $t$ we have $0\leq \phi(t)\leq 1 $, McDiarmid's inequality says with probability $\geq 1 - \exp(-2y^2 K)$ that:
\begin{align*}
    \sup_{\theta\in \Theta} \frac 1K &\Big( \sum_{k\in\cK} \phi\big(2\big|\wh g_j^{(k)}(\theta) - g_j(\theta)\big| / x\big) - \E \big[\phi\big(2\big|\wh g_j^{(k)}(\theta) - g_j(\theta)\big| / x\big)\big] \Big) \leq \\
    &\E \bigg[ \sup_{\theta\in \Theta} \frac 1K \Big( \sum_{k\in\cK} \phi\big(2\big|\wh g_j^{(k)}(\theta) - g_j(\theta)\big| / x\big) - \E \big[\phi\big(2\big|\wh g_j^{(k)}(\theta) - g_j(\theta)\big| / x\big)\big] \Big) \bigg] + y.
\end{align*}
Using a simple symmetrization argument (see for instance Lemma 11.4 in~\cite{boucheron2013concentration}) we find:
\begin{align*}
    \E \bigg[ \sup_{\theta\in \Theta} \frac 1K &\Big( \sum_{k\in\cK} \phi\big(2\big|\wh g_j^{(k)}(\theta) - g_j(\theta)\big| / x\big) - \E \big[\phi\big(2\big|\wh g_j^{(k)}(\theta) - g_j(\theta)\big| / x\big)\big] \Big) \bigg] \leq \\
    & 2\E \bigg[ \sup_{\theta\in\Theta} \frac 1K \sum_{k\in\cK} \varepsilon_k \phi\big(2\big|\wh g^{(k)}(\theta) - g(\theta)\big| / x\big) \bigg],
\end{align*}
where the $\varepsilon_k$s are independent Rademacher variables. Since $\phi$ is 1-Lipschitz and satisfies $\phi(0)=0$ we can use the contraction principle (see Theorem 11.6 in~\cite{boucheron2013concentration}) followed by another symmetrization step to find \begin{align*}
     2\E \bigg[ \sup_{\theta\in\Theta} \frac 1K &\sum_{k\in\cK} \varepsilon_k \phi\big(2\big|\wh g_j^{(k)}(\theta) - g_j(\theta)\big| / x\big) \bigg] \leq  4\E \bigg[ \sup_{\theta\in\Theta} \frac 1K \sum_{k\in\cK} \varepsilon_k \big|\wh g_j^{(k)}(\theta) - g_j(\theta)\big| / x \bigg]\\
     &\leq \frac{8}{x n} \E \bigg[ \sup_{\theta\in\Theta} \sum_{i\in \mathcal{J}} \varepsilon_i g^{i}_j(\theta) \bigg] \leq \frac{8 \mathcal{R}_j(\Theta)}{x n }.
\end{align*}
Taking $|\cO| \leq (1 - \varepsilon)K/2$, we found that with probability $\geq 1 - \exp(-2y^2 K)$
\begin{align*}
    \sup_{\theta\in\Theta} \sum_{k=1}^K \ind{}\Big\{\big|\wh g_j^{(k)}(\theta) - g_j(\theta)\big| > x\Big\} \leq K\bigg(\frac{3 M_{\alpha, j}}{(x/2)^{1 + \alpha} (n/K)^{\alpha}} + \frac{|\cO|}{K} +  \frac{8\mathcal{R}_j(\Theta)}{x n }\bigg).
\end{align*}
Now by choosing $y = 1/4 - |\cO|/K$ and $x = \max \Big(\Big(\frac{36M_{\alpha, j}}{(n/K)^\alpha}\Big)^{1/(1+\alpha)} , \frac{64 \mathcal{R}_j(\Theta)}{n}\Big)$, we obtain the deviation bound:
\begin{align*}
    \Proba \bigg(\sup_{\theta \in\Theta} \big| \wh g_j^{\mathtt{MOM}}(\theta) - g_j(\theta)\big| \geq \max &\Big(\Big(\frac{36M_{\alpha, j}}{(n/K)^\alpha}\Big)^{1/(1+\alpha)} , \frac{64\mathcal{R}_j(\Theta)}{n}\Big)\bigg) \\
    &\leq\Proba \bigg(\sup_{\theta\in\Theta} \sum_{k=1}^K \ind{}\Big\{\big|\wh g_j^{(k)}(\theta) - g_j(\theta)\big| > x\Big\} > K/2\bigg) \\
    &\leq \exp(-2 (\varepsilon - 1/2)^2 K/4)\\ 
    &\leq \exp(-K/18),
\end{align*}
where the last inequality comes from the choice $\varepsilon = 5/6$. A simple union bound argument lets the previous inequality hold for all $j\in\setint{d}$ with high probability.

Finally, assuming that $X^j$ has finite fourth moment for all $j\in\setint{d}$, we can control the Rademacher complexity. In this part, we assume without loss of generality that $\cI = \setint{n}$, we first write
\begin{align*}
    \mathcal{R}_j(\Theta) &= \E \bigg[ \sup_{\theta \in \Theta} \sum_{i=1}^n \varepsilon_i \ell'(\theta^\top X_i, Y_i)X_i^j \bigg] \\
    &= \E \bigg[\sum_{i=1}^n \varepsilon_i \ell'(0, Y_i)X_i^j + \sup_{\theta \in \Theta} \sum_{i=1}^n \varepsilon_i (\ell'(\theta^\top X_i, Y_i) - \ell'(0, Y_i))X_i^j \bigg].
\end{align*}
Denote $\phi_i(t) = (\ell'(t, Y_i) - \ell'(0, Y_i))X_i^j$ and notice that $\E \big[\sum_{i=1}^n \varepsilon_i \ell'(0, Y_i)X_i^j\big] = 0$. Notice also that $\phi_i(0) = 0$ and $\phi_i$ is $\gamma|X_i^j|$-Lipschitz for all $i$. We use a variant of the contraction principle adapted to our case in which functions with different Lipschitz constants appear. We use Lemma 11.7 from~\cite{boucheron2013concentration} and adapt the proof of their Theorem~11.6 to make the following estimations:
\begin{align*}
    \E \bigg[\sup_{\theta\in\Theta} \sum_{i=1}^n \varepsilon_i \phi_i(\theta^\top X_i) \bigg] &= \E \bigg[ \E \bigg[\sup_{\theta\in\Theta} \sum_{i=1}^{n-1}\varepsilon_i \phi_i(\theta^\top X_i) + \varepsilon_n \phi_n(\theta^\top X_n) \Big| (\varepsilon_i)_{i=1}^{n-1}, (X_i, Y_i)_{i\in\setint{n}} \bigg] \bigg] \\
    &\leq \E \bigg[ \E \bigg[\sup_{\theta\in\Theta} \sum_{i=1}^{n-1}\varepsilon_i \phi_i(\theta^\top X_i) + \varepsilon_n \gamma |X_n^j| \theta^\top X_n \Big| (\varepsilon_i)_{i=1}^{n-1}, (X_i, Y_i)_{i\in\setint{n}} \bigg] \bigg] \\
    &=  \E \bigg[\sup_{\theta\in\Theta} \sum_{i=1}^{n-1}\varepsilon_i \phi_i(\theta^\top X_i) + \varepsilon_n \gamma |X_n^j| \theta^\top X_n  \bigg].
\end{align*}
By iterating the previous argument $n$ times we find:
\begin{align*}
    \E &\bigg[\sup_{\theta\in\Theta} \sum_{i=1}^n \varepsilon_i \phi_i(\theta^\top X_i) \bigg] \leq  \E \bigg[\sup_{\theta\in\Theta} \sum_{i=1}^{n-1}\varepsilon_i \gamma |X_i^j| \theta^\top X_i \bigg].
\end{align*}
Now recalling that the diameter of $\Theta$ is $\Delta$, we use Lemma~\ref{lem:khintchine} below with $p=1$ to bound the previous quantity as:
\begin{align*}
    \E \bigg[\sup_{\theta\in\Theta} \sum_{i=1}^{n}\varepsilon_i \gamma |X_i^j| \theta^\top X_i \bigg] 
    &= \gamma \E \bigg[ \sup_{\theta \in \Theta} \bigg\langle \theta, \sum_{i=1}^{n} \varepsilon_i  X_i|X_i^j| \bigg\rangle \bigg] \\
    &\leq \gamma \Delta \E \bigg[ \E \bigg[ \bigg\|\sum_{i=1}^{n} \varepsilon_i  X_i|X_i^j| \bigg\|_1\Big| (X_i)_{i\in\setint{n}} \bigg]\bigg] \\
    &\leq \gamma \Delta C_{\alpha} \E \bigg[ \sum_{i=1}^{n} \| X_i \|^{1+\alpha} |X_i^j|^{1+\alpha} \bigg]^{1/(1+\alpha)} \\
    &\leq \gamma \Delta C_{\alpha} \bigg(n \E \big[(X^j)^{2(1+\alpha)}\big]^{1/2} \sum_{k\in\setint{d}} \E \big[(X^k)^{2(1+\alpha)}\big]^{1/2}\bigg)^{1/(1+\alpha)},
\end{align*}
where we used a Cauchy-Schwarz inequality in the last step, which concludes the proof of Proposition~\ref{prop:RademacherMOM}. $\qed$

\begin{lemma}[Khintchine inequality variant]\label{lem:khintchine}
Let $\alpha\in(0,1]$ and $(x_i)_{i\in\setint{n}}$ be real numbers with $n \in \N$ and $p>0$ and $(\varepsilon_i)_{i\in\setint{n}}$ be i.i.d Rademacher random variables then we have the inequality:
\begin{equation*}
    \E\bigg[ \bigg|\sum_{i=1}^n \varepsilon_i x_i\bigg|^p\bigg]^{1/p} \leq B_{p,\alpha} \bigg(\sum_{i=1}^n |x_i|^{1+\alpha}\bigg)^{1/(1+\alpha)}
\end{equation*}
with the constant $B_{p,\alpha}:= 2p \Big(\frac{1+\alpha}{\alpha}\Big)^{\alpha p/(1+\alpha) - 1} \Gamma\Big(\frac{\alpha p}{1+\alpha}\Big)$. Moreover, for $p=1$ the constant $B_{1,\alpha}$ is bounded for any $\alpha \geq 0$.
\end{lemma}

\begin{proof}
This proof is a generalization of Lemma 4.1 from~\cite{ledoux1991probability} and uses similar methods.
For all $\lambda > 0$ we have:
\begin{align*}
    \E\exp\Big(\lambda\sum_i \varepsilon_i x_i\Big) &= \prod_i \E\exp(\lambda \varepsilon_i x_i) = \prod_i \cosh(\lambda x_i) \\
    &\leq \prod_i \exp\Big( \frac{|\lambda x_i|^{1+\alpha}}{1+\alpha}\Big) = \exp\Big( \sum_i \frac{|\lambda x_i|^{1+\alpha}}{1+\alpha}\Big),
\end{align*}
where we used the inequality $\cosh(u) \leq \exp\Big( \frac{|u|^{1+\alpha}}{1+\alpha}\Big)$ valid for all $u\in\R$ which can be quickly proven. Since both functions are even, fix $u>0$ and define $f_u(\alpha)=\exp\Big( \frac{|u|^{1+\alpha}}{1+\alpha}\Big) - \cosh(u)$, we can show that $f_u$ is monotonous on $[0,1]$ separately for $u\in (0,\sqrt{e})$ and $(e, +\infty)$ and notice that $f_u(0)$ and $f_u(1)$ are both non-negative for all $u >0$ thanks to the famous inequality $\cosh(u)\leq e^{u^2/2}$. Therefore, the inequality holds for $u\in (0,\sqrt{e})$ and $(e, +\infty)$. Finally, for $u\in (\sqrt{e}, e)$, the function $f_u(\alpha)$ reaches a minimum at $f_u(1/\log(u) - 1) = u^e - \cosh(u)$ and by taking logarithms we have $u^e \geq \cosh(u) \iff \log(1+e^{2u})\leq u + \log(2) + e\log(u)$ but since the derivatives verify $\frac{2}{1+e^{-2u}}\leq2 \leq 1+e/u$ for $u\in(\sqrt{e}, e)$ and $e^{e/2}\geq\cosh(\sqrt{e})$ the desired inequality follows by integration.

By homogeneity, we can focus on the case $\big(\sum_{i=1}^n |x_i|^{1+\alpha}\big)^{1/(1+\alpha)} = 1$, we compute: 
\begin{align*}
    \E\Big|\sum_i \varepsilon_i x_i\Big|^p &= \int_0^{+\infty} \Proba \Big(\Big|\sum_i \varepsilon_i x_i\Big|^p > t\Big)dt \\
    &\leq 2\int_0^{+\infty} \exp\Big(\frac{\lambda^{1+\alpha}}{1+\alpha} -\lambda t^{1/p}\Big) dt \\
    &= 2\int_0^{+\infty} \exp\Big(-\frac{\alpha}{1+\alpha} u^{(1+\alpha)/\alpha}\Big) du^p \\
    &= 2p \Big(\frac{1+\alpha}{\alpha}\Big)^{\alpha p/(1+\alpha) - 1} \Gamma\Big(\frac{\alpha p}{1+\alpha}\Big) = B_{p,\alpha}^p,
\end{align*}
where we used the previous inequality and chose $\lambda = (t^{1/p})^{1/\alpha}$ in the last step. This proves the main inequality. Finally, it is easy to see that $B_{1,\alpha}$ is bounded for high values of $\alpha$ while for $\alpha \sim 0$ it is consequence of the fact that $\Gamma(x) \sim 1/x$ near $0$ and the limit $x^x \to 0$ when $x \to 0^+$.
\end{proof}

\subsection{Proof of Lemma~\ref{lem:basicTMean}}

As previously, Lemma~\ref{lem:partial-deriv-moments} along with Assumptions \ref{assump:lipsmoothloss} and \ref{assump:data} guarantee that the gradient coordinates have finite $(1+\alpha)$-moments. From here, Lemma~\ref{lem:basicTMean} is a direct application of Lemma~\ref{lem:tmeanGeneric} stated and proved below.
In the following lemma, for any sequence $(z_i)_{i=1}^N$ of real numbers, $(z_i^*)_{i=1}^N$ denotes a non-decreasing reordering of it. 

\begin{lemma}
\label{lem:tmeanGeneric}
Let $\widetilde{X}_1, \dots, \widetilde{X}_N, \widetilde{Y}_1, \dots, \widetilde{Y}_N$ denote an $\eta$-corrupted i.i.d sample with rate $\eta$ from a random variable $X$ with expectation $\mu = \E X$ and with finite $1 + \gamma$ centered moment $\E|X - \mu|^{1+\gamma} =M < \infty$ for some $0 < \gamma \leq 1$.
Denote $\widehat{\mu}$ the $\epsilon$-trimmed mean estimator computed as $\widehat{\mu}=\frac{1}{N}\sum_{i=1}^N \phi_{\alpha, \beta}(\widetilde{X}_i)$ with $\phi_{\alpha, \beta}(x) = \max(\alpha, \min(x, \beta))$ and the thresholds $\alpha = \widetilde{Y}^*_{\epsilon N}$ and $\beta = \widetilde{Y}^*_{(1-\epsilon) N}$.
Let $1 > \delta \geq e^{-N}/4,$ taking $\epsilon = 8\eta +12 \frac{\log(4/\delta)}{n},$ we have 
\begin{equation}
    |\widehat{\mu} - \mu| \leq 7 M^{\frac{1}{1 + \gamma}}(\epsilon/2)^{\frac{\gamma}{1 + \gamma}}
\end{equation}
with probability at least $1 - \delta$.
\end{lemma}
\begin{proof}
This proof goes along the lines of the proof of Theorem 1 from~\cite{lugosi2021robust} with the main difference that only the $(1+\gamma)$-moment is used instead of the variance. Denote $X$ the random variable whose expectation $\mu = \E X$ is to be estimated and $\overline{X} = X - \mu$. Let $X_1, \dots, X_N, Y_1, \dots, Y_N$ the original uncorrupted i.i.d. sample from $X$ and let $\widetilde{X}_1, \dots, \widetilde{X}_N, \widetilde{Y}_1, \dots, \widetilde{Y}_N$ denote the corrupted sample with rate $\eta$.
We define the following quantity which will intervene in the proof:
\begin{equation}
    \overline{\mathcal{E}} (\epsilon, X):= \max \Big\{\E \big[ \big|\overline{X} - Q_{\epsilon/2}(\overline{X})\big| \ind{\overline{X} \leq Q_{\epsilon/2}(\overline{X})}\big] , \E \big[ \big|\overline{X} - Q_{1 - \epsilon/2}(\overline{X})\big| \ind{\overline{X} \geq Q_{1 - \epsilon/2}(\overline{X})} \big]\Big\}.
\end{equation}

\paragraph{Step 1.} 
We first derive confidence bounds on the truncation thresholds.
Define the random variable $U = \ind{\overline{X} \geq Q_{1-2\epsilon}(\overline{X})}$. Its standard deviation satisfies $\sigma_U \leq \Proba^{1/2}(\overline{X} \geq Q_{1-2\epsilon}(\overline{X})) = \sqrt{2\epsilon}$. 
By applying Bernstein's inequality we find with probability $\geq 1 - \exp(-\epsilon N/12)$ that:
\[
\big|\big\{i \:: \: Y_i \geq \mu + Q_{1-2\epsilon}(\overline{X})\big\}\big| \geq 3\epsilon N/2,
\]
a similar argument with $U = \ind{\overline{X} > Q_{1 - \epsilon/2}(\overline{X})}$ yields with probability $\geq 1 - \exp(-\epsilon N /12)$ that:
\[
\big|\big\{i \:: \: Y_i \leq \mu + Q_{1-\epsilon/2}(\overline{X})\big\}\big| \geq (1 - (3/4)\epsilon) N,
\]
and similarly with probability $\geq 1 - \exp(-\epsilon N /12)$ we have:
\[
\big|\big\{i \:: \: Y_i \leq \mu + Q_{2\epsilon}(\overline{X})\big|\}\big| \geq 3\epsilon N/2,
\]
and with probability $\geq 1 - \exp(-\epsilon N /12)$:
\[
\big|\big\{i \:: \: Y_i \geq \mu + Q_{\epsilon/2}(\overline{X})\big\}\big| \geq (1 - (3/4)\epsilon) N,
\]
so that with probability $\geq 1 - 4\exp(-\epsilon N /12) \geq 1 - \delta/2$ the four previous inequalities hold simultaneously. 
We call this event $E$ which only depends on the variables $Y_1, \dots, Y_N$. 
Since $\eta \leq \epsilon/8$, if $2\eta N$ samples are corrupted we still have:
\[\big|\big\{i \:: \: \widetilde{Y}_i \geq \mu + Q_{1 - 2\epsilon}(\overline{X})\big|\}\big| \geq ((3/2)\epsilon - 2\eta) N \geq \epsilon N\]
and 
\[\big|\big\{i \:: \: \widetilde{Y}_i \leq \mu + Q_{1 - \epsilon/2}(\overline{X})\big\}\big| \geq (1 - (3/4)\epsilon - 2\eta) N \geq (1 - \epsilon) N\]
consequently, the two following bounds hold
\[Q_{1-2\epsilon}(\overline{X})\leq \widetilde{Y}^*_{(1-\epsilon)N} - \mu \leq Q_{1-\epsilon/2}(\overline{X})\]
and similarly
\[
Q_{\epsilon/2}(\overline{X})\leq \widetilde{Y}^*_{\epsilon N} - \mu \leq Q_{2\epsilon}(\overline{X}).
\]
This provides guarantees on the truncation levels used which are $\alpha = \widetilde{Y}^*_{\epsilon N}$ and $\beta = \widetilde{Y}^*_{(1-\epsilon) N}$.

\paragraph{Step 2.} 

We first bound the deviation $\Big|\frac{1}{N} \sum_{i=1}^N \phi_{\alpha, \beta}(X_i) - \mu \Big|$ in the absence of corruption. W
e write:
\begin{align}
\frac{1}{N} \sum_{i=1}^N \phi_{\alpha, \beta}(X_i) &\leq \frac{1}{N}\sum_{i=1}^N \phi_{\mu + Q_{2\epsilon}(\overline{X}), \mu + Q_{1 - \epsilon/2}(\overline{X})}(X_i) = \E \big[ \phi_{\mu + Q_{2\epsilon}(\overline{X}), \mu + Q_{1 - \epsilon/2}(\overline{X})}(X)\big] \nonumber \\ &+ \frac{1}{N}\sum_{i=1}^N\Big( \phi_{\mu + Q_{2\epsilon}(\overline{X}), \mu + Q_{1 - \epsilon/2}(\overline{X})}(X_i) - \E \big[ \phi_{\mu + Q_{2\epsilon}(\overline{X}), \mu + Q_{1 - \epsilon/2}(\overline{X})}(X) \big]\Big) \label{sumboundedterms}.
\end{align}
The first term is dominated by:
\begin{align*}
    \E \big[&\phi_{\mu + Q_{2\epsilon}(\overline{X}), \mu + Q_{1 - \epsilon/2}(\overline{X})}(X)\big] = \E \big[\phi_{Q_{2\epsilon}(X), Q_{1 - \epsilon/2}(X)}(X)\big]\\
    &= \E \big[ Q_{2\epsilon}(X)\ind{X \leq Q_{2\epsilon}(X)} + X\ind{Q_{2\epsilon}(X) < X < Q_{1 - \epsilon/2}(X)} + Q_{1 - \epsilon/2}(X)\ind{X \geq Q_{1 - \epsilon/2}(X)}\big] \\
    &= \mu + \E \big[(Q_{2\epsilon}(X) - X)\ind{X \leq Q_{2\epsilon}(X)} + (Q_{1 - \epsilon/2}(X) - X)\ind{X \geq Q_{1 - \epsilon/2}(X)}\big] \\
    &\leq \mu + \E \big[(Q_{2\epsilon}(X) - X)\ind{X \leq Q_{2\epsilon}(X)}\big] = \mu + \E \big[(Q_{2\epsilon}(\overline{X}) - \overline{X})\ind{\overline{X} \leq Q_{2\epsilon}(\overline{X})}\big] \\
    &\leq \mu + \overline{\mathcal{E}} (4\epsilon, X),
\end{align*}
and lower bounded by:
\begin{align*}
    \E \big[&\phi_{\mu + Q_{2\epsilon}(\overline{X}), \mu + Q_{1 - \epsilon/2}(\overline{X})}(X)\big] \\
    &= \mu + \E \big[(Q_{2\epsilon}(X) - X)\ind{X \leq Q_{2\epsilon}(X)} + (Q_{1 - \epsilon/2}(X) - X)\ind{X \geq Q_{1 - \epsilon/2}(X)}\big] \\
    &\geq \mu + \E \big[(Q_{1 - \epsilon/2}(X) - X)\ind{X \geq Q_{1 - \epsilon/2}(X)}\big]
    = \mu + \E \big[(Q_{1 - \epsilon/2}(\overline{X}) - \overline{X})\ind{\overline{X} \geq Q_{1 - \epsilon/2}(\overline{X})}\big] \\
    &\geq \mu - \overline{\mathcal{E}} (\epsilon, X).
\end{align*}
The sum in~\eqref{sumboundedterms} above has terms upper bounded by $Q_{1 - \epsilon/2}(\overline{X}) + \overline{\mathcal{E}}(\epsilon, X)$. 
We need to work with the knowledge that $\E |\overline{X}|^{1+\gamma} = M < \infty$ in order to bound their variance:
\begin{align*}
    \E \big[ &\phi_{\mu + Q_{2\epsilon}(\overline{X}), \mu + Q_{1 - \epsilon/2}(\overline{X})}(X) - \E [\phi_{\mu + Q_{2\epsilon}(\overline{X}), \mu + Q_{1 - \epsilon/2}(\overline{X})}(X)] \big]^2 \\
    &\leq \E \big[ \phi_{\mu + Q_{2\epsilon}(\overline{X}), \mu + Q_{1 - \epsilon/2}(\overline{X})}(X) - \mu\big]^2 
    = \E \big[\phi_{Q_{2\epsilon}(\overline{X}), Q_{1 - \epsilon/2}(\overline{X})}(\overline{X})^2\big] \\
    &= \E \big[ Q_{2\epsilon}(\overline{X})\ind{\overline{X} \leq Q_{2\epsilon}(\overline{X})} + \overline{X}\ind{Q_{2\epsilon}(\overline{X}) < \overline{X} < Q_{1 - \epsilon/2}(\overline{X})} + Q_{1 - \epsilon/2}(\overline{X})\ind{\overline{X} \geq Q_{1 - \epsilon/2}(\overline{X})}\big]^2 \\
    &= \E \big[ Q_{2\epsilon}(\overline{X})^2\ind{\overline{X} \leq Q_{2\epsilon}(\overline{X})} + \overline{X}^2\ind{Q_{2\epsilon}(\overline{X}) < \overline{X} < Q_{1 - \epsilon/2}(\overline{X})} + Q_{1 - \epsilon/2}(\overline{X})^2\ind{\overline{X} \geq Q_{1 - \epsilon/2}(\overline{X})}\big].
\end{align*} 
To control the three terms in the previous expression we mimic the proof of Chebyshev's inequality to obtain that, when $Q_{2\epsilon}(\overline{X}) < 0$:
\begin{equation}
    2\epsilon = \Proba \big(\overline{X} \leq Q_{2\epsilon}(\overline{X})\big) \leq \Proba\big(|\overline{X}|^{1+\gamma} \geq |Q_{2\epsilon}(\overline{X})|^{1+\gamma}\big) 
    \leq \frac{M}{|Q_{2\epsilon}(\overline{X})|^{1+\gamma}},
    \label{ineq:moment1}
\end{equation}
analogously, when $Q_{1 - \epsilon/2}(\overline{X}) > 0$ we have:
\begin{equation}
    \epsilon/2 = \Proba\big(\overline{X} \geq Q_{1-\epsilon/2}(\overline{X})\big) \leq \Proba\big(|\overline{X}|^{1+\gamma} \geq |Q_{1-\epsilon/2}(\overline{X})|^{1+\gamma}\big) 
    \leq \frac{M}{|Q_{1-\epsilon/2}(\overline{X})|^{1+\gamma}},
    \label{ineq:moment2}
\end{equation}
from~\eqref{ineq:moment1}, we deduce that 
\begin{equation*}
    \E \big[Q_{2\epsilon}(\overline{X})^2\ind{\overline{X} \leq Q_{2\epsilon}(\overline{X})}\big] = 2\epsilon Q_{2\epsilon}(\overline{X})^2 \leq 2\epsilon \Big(\frac{M}{2\epsilon}\Big)^{\frac{2}{1 + \gamma}} \leq 2\epsilon \Big(\frac{2M}{\epsilon}\Big)^{2/(1 + \gamma)},
\end{equation*}
and from~\eqref{ineq:moment2} we find
\begin{equation*}
    \E \big[Q_{1 - \epsilon/2}(\overline{X})^2\ind{\overline{X} \geq Q_{1 - \epsilon/2}(\overline{X})}\big] = Q_{1 - \epsilon/2}(\overline{X})^2 \epsilon/2 \leq 2\epsilon \Big(\frac{2M}{\epsilon}\Big)^{2/(1 + \gamma)}.
\end{equation*}
In the pathological case where we have $Q_{2\epsilon}(\overline{X}) \geq 0$ we use that $Q_{2\epsilon}(\overline{X}) \leq Q_{1 - \epsilon/2}(\overline{X})$ (for $\epsilon \leq 2/5$) we deduce $|Q_{2\epsilon}(\overline{X})| \leq |Q_{1 - \epsilon/2}(\overline{X})|$ and hence we still have
\begin{equation*}
    \E \big[Q_{2\epsilon}(\overline{X})^2\ind{\overline{X} \leq Q_{2\epsilon}(\overline{X})}\big] \leq 2\epsilon Q_{1 - \epsilon/2}(\overline{X})^2 \leq 2\epsilon \Big(\frac{2M}{\epsilon}\Big)^{2/(1 + \gamma)}.
\end{equation*}
The case $Q_{1 - \epsilon/2}(\overline{X})\leq 0$ is similarly handled. Moreover, a simple calculation yields
\begin{equation*}
    \E \big[\overline{X}^2\ind{Q_{2\epsilon}(\overline{X}) \leq \overline{X} \leq Q_{1 - \epsilon/2}(\overline{X})}\big] \leq M \max \big\{|Q_{2\epsilon}(\overline{X})|, |Q_{1 - \epsilon/2}(\overline{X})|\big\}^{1 - \gamma} \leq 2\epsilon \Big(\frac{2M}{\epsilon}\Big)^{2/(1 + \gamma)}.
\end{equation*} 
All in all, we have shown the inequality:
\begin{equation*}
    \E \big[ \phi_{\mu + Q_{2\epsilon}(\overline{X}), \mu + Q_{1 - \epsilon/2}(\overline{X})}(X) - \E [\phi_{\mu + Q_{2\epsilon}(\overline{X}), \mu + Q_{1 - \epsilon/2}(\overline{X})}(X)]\big]^2 \leq 6\epsilon \Big(\frac{2M}{\epsilon}\Big)^{2/(1 + \gamma)},
\end{equation*}
which we now use to apply Bernstein's inequality on the sum in~\eqref{sumboundedterms} to find, conditionally on $Y_1, \dots, Y_n$, with probability at least $1 - \delta/4$:
\begin{align*}
    \frac{1}{N}\sum_{i=1}^N \phi_{\alpha, \beta}(X_i) &\leq \mu + \overline{\mathcal{E}}(4\epsilon, X) + \sqrt{\frac{6 \epsilon \log(4/\delta)}{N}} \Big(\frac{2M}{\epsilon}\Big)^{1/(1+\gamma)} + \frac{\log(4/\delta)}{3N}(Q_{1 - \epsilon/2}(\overline{X}) + \overline{\mathcal{E}}(\epsilon, X)) \\
    &\leq \mu + 2\overline{\mathcal{E}}(4\epsilon, X) + \sqrt{\frac{6 \epsilon \log(4/\delta)}{N}} \Big(\frac{2M}{\epsilon}\Big)^{1/(1+\gamma)} + \frac{\log(4/\delta)}{3N}Q_{1 - \epsilon/2}(\overline{X}) \\
    &\leq \mu + 2\overline{\mathcal{E}}(4\epsilon, X) + (3/2) M^{1/(1+\gamma)} (\epsilon/2)^{\gamma/(1+\gamma)},
\end{align*}
where we used (\ref{ineq:moment2}), the fact that $\frac{\log(4/\delta)}{N} \leq \epsilon/12$ and the assumption that $\delta \geq e^{-N}/4$.
Using the same argument on the lower tail, we obtain, on the event $E$, that with probability at least $1 - \delta/2$
\begin{equation*}
    \Big|\frac{1}{N}\sum_{i=1}^N \phi_{\alpha, \beta}(X_i) - \mu\Big| \leq 2\overline{\mathcal{E}}(4\epsilon, X) + (3/2)M^{\frac{1}{1+\gamma}} (\epsilon/2)^{\gamma/(1+\gamma)}.
\end{equation*}

\paragraph{Step 3.} 

Now we show that $\Big|\frac{1}{N}\sum_{i=1}^N \phi_{\alpha, \beta}(X_i) - \frac{1}{N}\sum_{i=1}^N \phi_{\alpha, \beta}(\widetilde{X}_i)\Big|$ is of the same order as the previous bounds. 
There are at most $2\eta N$ indices such that $X_i \neq \widetilde{X}_i$ and for such differences we have the bound:
\begin{equation*}
\big|\phi_{\alpha, \beta}(X_i) - \phi_{\alpha, \beta}(\widetilde{X}_i)\big| \leq |Q_{\epsilon/2}(\overline{X})| +  |Q_{1 - \epsilon/2}(\overline{X})|,
\end{equation*}
 and since we have $\eta \leq \epsilon/8$ then 
 \begin{align*}
     \Big|\frac{1}{N} \sum_{i=1} \phi_{\alpha, \beta}(X_i) - \frac{1}{N} \sum_{i=1} \phi_{\alpha, \beta}(\widetilde{X}_i)\Big| &\leq 2\eta \big(|Q_{\epsilon/2}(\overline{X})| +  |Q_{1 - \epsilon/2}(\overline{X})|\big) \\
     &\leq \frac{\epsilon}{2}\max \big\{|Q_{\epsilon/2}(\overline{X})|, |Q_{1 - \epsilon/2}(\overline{X})|\big\} \\
     &\leq M^{1/(1+\gamma)}(\epsilon/2)^{\gamma/(1+\gamma)},
 \end{align*}
where the last step follows from (\ref{ineq:moment1}) and (\ref{ineq:moment2}). Finally, using similar arguments along with Hölder's inequality, we show that:
\begin{align*}
    \E \big[ |\overline{X} - Q_{\epsilon/2}(\overline{X})| \ind{\overline{X} \leq Q_{\epsilon/2}(\overline{X})}\big] &\leq \E \big[ |\overline{X}| \ind{\overline{X} \leq Q_{\epsilon/2}(\overline{X})}\big] + \E \big[ |Q_{\epsilon/2}(\overline{X})| \ind{\overline{X} \leq Q_{\epsilon/2}(\overline{X})}\big] \\
    &\leq M^{1/(1+\gamma)}(\epsilon/2)^{\gamma/(1+\gamma)} + |Q_{\epsilon/2}(\overline{X})|(\epsilon/2) \\
    &\leq 2M^{1/(1+\gamma)}(\epsilon/2)^{\gamma/(1+\gamma)},
\end{align*}
and a similar computation for $\E \big[ |\overline{X} - Q_{1 - \epsilon/2}(\overline{X})| \ind{\overline{X} \geq Q_{1 - \epsilon/2}(\overline{X})}\big]$ leads to \begin{equation*}
    \overline{\mathcal{E}}(4\epsilon, X) \leq 2M^{1/(1+\gamma)}(2\epsilon)^{\gamma/(1+\gamma)}.
\end{equation*}
This completes the proof of Lemma~\ref{lem:tmeanGeneric}.
\end{proof}

\subsection{Proof of Proposition~\ref{prop:uniformTM}}

\textbf{Step 1.} Notice that the $\mathtt{TM}$ estimator is also a monotonous non decreasing function of each of its entries when the others are fixed. This allows us to replicate Step 1 of the proof of Proposition~\ref{prop:uniformMOM}.
We define an $\varepsilon$-net $N_{\varepsilon}$ on the set $\Theta$, fix $\theta \in\Theta$ and let $\widetilde{\theta}$ be the closest point in $N_{\varepsilon}$. We obtain, for all $j\in\setint{d}$, the inequalities:
\begin{align}
    \big|\wh g_j^{\mathtt{TM}}(\theta) - g_j(\theta)\big| &\leq \big|\wh g_j^{\mathtt{TM}}(\theta) - g_j(\widetilde{\theta})\big| + \big|g_j(\widetilde{\theta}) - g_j(\theta)\big| \nonumber \\
    &\leq \big|\widecheck{g}_j^{\mathtt{TM}}(\widetilde{\theta}) - g_j(\widetilde{\theta})\big| + \varepsilon L_j,
    \label{eq:tmstab1}
\end{align}
where $\widecheck{g}_j^{\mathtt{TM}}(\widetilde{\theta})$ is the $\mathtt{TM}$ estimator obtained for the entries $\ell'\big(\widetilde{\theta}^{\top} X_i, Y_i\big)X_i^j + \varepsilon \gamma \|X_i\|^2 = g_j^i(\widetilde{\theta}) + \varepsilon \gamma \|X_i\|^2$.

\paragraph{Step 2.}

We use the concentration property of the $\mathtt{TM}$ estimator to bound the previous quantity which is in terms of $\widetilde{\theta}$.
The terms $\big(g_j^i(\widetilde{\theta}) + \varepsilon \gamma \|X_i\|^2\big)_{i\in\setint{n}}$ are independent and distributed according to $Z:= \ell'\big(\widetilde{\theta}^\top X, Y\big)X^j + \gamma \varepsilon \|X\|^2.$ 
Obviously we have $\E \ell'\big(\theta^\top X, Y\big)X^j = g_j(\theta).$ 
Furthermore, let $\overline{L} = \E \gamma \|X\|^2$, so that $\E\big[g_j^i(\widetilde{\theta}) + \varepsilon \gamma \|X_i\|^2\big] = g_j(\theta) + \varepsilon \overline{L}$.
We will apply Lemma~\ref{lem:tmeanGeneric} for $\widecheck{g}_j^{\mathtt{TM}}(\widetilde{\theta}).$ Before we do so, we need to compute the centered $(1+\alpha)$-moment of $Z$. Let $m_{j,\alpha}(\widetilde{\theta})$ and $m_{L,\alpha}$ be the centered $(1+\alpha)$-moments of $\ell'(\theta^\top X, Y)X^j$ and $\gamma \|X\|^2$ respectively, we have:
\begin{equation*}
    \E\big|Z - \E Z\big|^{1+\alpha} \leq 2^{\alpha}\big(m_{j, \alpha}(\theta) + \varepsilon^{1+\alpha}m_{L,\alpha}\big).
\end{equation*}
Now applying Lemma~\ref{lem:tmeanGeneric} we find with probability no less than $1-\delta$
\begin{equation*}
    \big|\widecheck{g}^{\mathtt{TM}}_j(\widetilde{\theta}) - g_j(\widetilde{\theta}) - \varepsilon\overline{L}\big|  \leq 7\big(m_{j, \alpha}(\widetilde{\theta}) + \varepsilon^{1+\alpha}m_{L,\alpha}\big)^{1/(1+\alpha)}(2\epsilon)^{\alpha/(1+\alpha)},
\end{equation*}
with $\epsilon_{\delta} = 8\eta +12 \frac{\log(4/\delta)}{n}$. By combining with \eqref{eq:tmstab1} and using a union bound argument, we deduce that with the same probability, we have for all $j\in\setint{d}$
\begin{equation}\label{eq:fixedthetaTM}
    \big|\widecheck{g}^{\mathtt{TM}}_j(\widetilde{\theta}) - g_j(\widetilde{\theta})\big| \leq 7\big(m_{j, \alpha}(\widetilde{\theta}) + \varepsilon^{(1+\alpha)^2} m_{L,\alpha}\big)^{1/(1+\alpha)}(4\epsilon_{d\delta})^{\alpha/(1+\alpha)} + \varepsilon \overline{L}.
\end{equation}

\paragraph{Step 3.} 

We use the $\varepsilon$-net to obtain a uniform bound.
We proceed similarly as in the proof of Proposition~\ref{prop:uniformMOM}. 
For $\theta \in \Theta$ denote $\widetilde{\theta}(\theta) \in N_{\varepsilon}$ the closest point in $N_{\varepsilon}$ satisfying in particular $\|\widetilde{\theta}(\theta) - \theta\| \leq \varepsilon$, we write, following previous arguments
\begin{align*}
    \sup_{\theta \in \Theta} \big| \wh g_j^{\mathtt{TM}} (\theta) - g_j(\theta) \big| &\leq \sup_{\theta \in \Theta} \big| \wh g_j^{\mathtt{TM}} (\theta) - g_j(\widetilde{\theta}(\theta)) \big| + \big| g_j (\widetilde{\theta}(\theta)) - g_j(\theta) \big| \\
    &\leq \sup_{\theta \in \Theta} \big| \widecheck g_j^{\mathtt{TM}} (\widetilde{\theta}(\theta)) - g_j(\widetilde{\theta}(\theta)) \big| + \varepsilon L_j \\
    &= \max_{\widetilde{\theta} \in N_{\varepsilon}} \big| \widecheck g_j^{\mathtt{TM}} (\widetilde{\theta}) - g_j(\widetilde{\theta}) \big| + \varepsilon L_{j}.
\end{align*}
Taking union bound over $\widetilde{\theta} \in N_{\varepsilon}$ for the inequality \eqref{eq:fixedthetaTM} and choosing $\varepsilon = n^{-\alpha/(1+\alpha)}$ concludes the proof of Proposition~\ref{prop:uniformTM}.

\subsection{Proof of Lemma~\ref{lem:basicCH}}

Similarly to the proof of Lemma~\ref{lem:basicMOM}, the assumptions, this time taken with $\alpha = 1$, imply that the gradient has a second moment so that the existence of $\sigma_j^2 = \V (g_j(\theta))$ is guaranteed. 
We apply Lemma~1 from~\cite{pmlr-v97-holland19a} with $\delta/2$ to obtain:
\begin{equation*}
     \frac{1}{2}|\widehat{g}_j^{\mathtt{CH}}(\theta) - g_j(\theta)| \leq \frac{C \sigma_j^2}{s} + \frac{s \log(4\delta^{-1})}{n} 
\end{equation*}
with probability at least $1 - \delta/2$, where $C$ is a constant such that we have:
\[
-\log(1 - u + Cu^2) \leq \psi(u) \leq \log(1 + u + Cu^2),
\]
and one can easily check that our choice of $\psi$, the Gudermannian function, satisfies the previous inequality for $C = 1/2$. 
This, along with the choice of scale $s$ according to~\eqref{eq:ch-scale-estimator} and our assumption on $\widehat{\sigma}_j$ yields the announced deviation bound by a simple union bound argument.

\subsection{Proof of Proposition~\ref{prop:uniformCH}}

In this proof, for a scale $s > 0$ and a set of real numbers $(x_i)_{i\in\setint{n}}$, we let $\bar{x}=\frac 1 n \sum_{i\in\setint{n}} x_i$ be their mean and define the function $\zeta_s\big((x_i)_{i\in\setint{n}}\big)$ as the unique $x$ satisfying \begin{equation*}
    \sum_{i\in\setint{n}} \psi\Big( \frac{x - \bar{x}}{s} \Big) = 0.
\end{equation*}
Since the function $\psi$ is increasing the previous equation has a unique solution. Moreover, for fixed scale $s$, the function $\zeta_s\big((x_i)_{i\in\setint{n}}\big)$ is monotonous non decreasing w.r.t. each $x_i$ when the others are fixed.

\paragraph{Step 1.}

We proceed similarly as in the proof of Proposition~\ref{prop:uniformMOM} except that we only use the monotonicity of the $\mathtt{CH}$ estimator with fixed scale.
Let $N_{\varepsilon}$ be an $\varepsilon$-net for $\Theta$ with $\varepsilon = 1/\sqrt{n}$. We have $|N_{\varepsilon}| \leq (3\Delta/2\varepsilon)^d$ with $\Delta$ the diameter of $\Theta$. Fix a coordinate $j\in\setint{d}$, a point $\theta \in \Theta$ and let $\widetilde \theta$ be the closest point to it in the $\varepsilon$-net. We wish to bound the difference 
\begin{align*}
    \big|\wh g_j^{\mathtt{CH}}(\theta) - g_j(\theta)\big| &\leq \big|\wh g_j^{\mathtt{CH}}(\theta) - g_j(\widetilde \theta)\big| + \big| g_j(\widetilde \theta) - g_j(\theta)\big| \\
    &\leq \big|\wh g_j^{\mathtt{CH}}(\theta) - g_j(\widetilde \theta)\big| + \varepsilon L_j,
\end{align*}
where we have the $\mathtt{CH}$ estimator $\wh g_j^{\mathtt{CH}}(\theta) = \zeta_{s(\theta)}\big((g_j^i(\theta))_{i\in\setint{n}}\big)$ with scale $s(\theta)$ computed according to \eqref{eq:ch-scale-estimator} and \eqref{eq:scaleMestimator}. Assume, without loss of generality that $\wh g_j^{\mathtt{CH}}(\theta) - g_j(\widetilde \theta) \geq 0$. Using the non-decreasing property of the $\mathtt{CH}$ estimator at a fixed scale, we find that
\begin{align*}
    \big|\wh g_j^{\mathtt{CH}}(\theta) - g_j(\widetilde \theta)| &= \big|\zeta_{s(\theta)}\big((g_j^i(\theta))_{i\in\setint{n}}\big) - g_j(\widetilde \theta)\big| \\
    &\leq \big|\zeta_{s(\theta)}\big((g_j^i(\widetilde \theta) + \varepsilon \gamma \|X_i\|^2)_{i\in\setint{n}}\big) - g_j(\widetilde \theta)\big|.
\end{align*}
Indeed, one has
\begin{align*}
    g_j^i(\theta) &= g_j^i(\widetilde \theta) + \big(g_j^i(\theta) - g_j^i(\widetilde \theta)\big) \\
    &\leq g_j^i(\widetilde \theta) + \gamma \|\widetilde \theta - \theta\| \cdot \|X_i\| \cdot |X_i^j| \\
    &\leq g_j^i(\widetilde \theta) + \varepsilon \gamma \|X_i\|^2.
\end{align*}
We introduce the notation $\widecheck{g}_j^{\mathtt{CH}}(\widetilde \theta):= \zeta_{s(\theta)}\big((g_j^i(\widetilde \theta) + \varepsilon \gamma \|X_i\|^2)_{i\in\setint{n}}\big)$ so that:
\begin{equation*}
    \big|\wh g_j^{\mathtt{CH}}(\theta) - g_j(\widetilde \theta)\big| \leq \big|\widecheck g_j^{\mathtt{CH}}(\widetilde \theta) - g_j(\widetilde \theta)\big| .
\end{equation*}

\paragraph{Step 2.}

We now use the concentration property of $\mathtt{CH}$ to bound the previous quantity which is in terms of $\widetilde{\theta}$.
We apply Lemma~1 from~\cite{pmlr-v97-holland19a} with $\delta/2$ and scale $s(\theta)$ to the samples $(g_j^i(\widetilde \theta) + \varepsilon \gamma \|X_i\|^2)_{i\in\setint{n}}$ which are independent and distributed according to the random variable $\ell'\big(\widetilde \theta^\top X, Y\big)X^j + \varepsilon \gamma \|X\|^2$ with expectation $g_j(\widetilde \theta) + \varepsilon \overline{L}$. Using our assumptions on $\sigma_L, \sigma_j(\theta), \sigma_j(\widetilde \theta), \wh \sigma_j(\theta)$ and the definition of the scale $s(\theta)$ according to \eqref{eq:ch-scale-estimator} we find:
\begin{align*}
    \frac 1 2 \big|\widecheck{g}_j^{\mathtt{CH}}(\widetilde \theta)- g_j(\widetilde \theta) - \varepsilon \overline{L}\big| &=
    \frac 1 2 \big|\zeta_{s(\theta)}\big((g_j^i(\widetilde \theta) + \varepsilon \gamma \|X_i\|^2)_{i\in\setint{n}}\big) - g_j(\widetilde \theta) - \varepsilon \overline{L}\big| \\ 
    &\leq \frac{C \V (g_j^i(\widetilde \theta) + \varepsilon \gamma \|X_i\|^2)}{s(\theta)} + \frac{s(\theta) \log(4/\delta)}{n} \\
    &\leq \frac{CC' \V (g_j^i(\widetilde \theta) + \varepsilon \gamma \|X_i\|^2)}{\sigma_j(\theta)}\sqrt{\frac{2\log(4/\delta)}{n}} + C'\sigma_j(\theta) \sqrt{\frac{2\log(4/\delta)}{n}}\\
    &\leq \frac{CC' 2(\sigma^2_j(\widetilde \theta) + \varepsilon^2 \sigma_L^2)}{\sigma_j(\theta)}\sqrt{\frac{2\log(4/\delta)}{n}} + C'\sigma_j(\theta)\sqrt{\frac{ 2\log(4/\delta)}{n}}\\
    &\leq CC' 2\big(\sqrt{2}\sigma_j(\widetilde \theta) + \varepsilon \sigma_L\big)\sqrt{\frac{2\log(4/\delta)}{n}} +  2C'\sigma_j(\widetilde \theta) \sqrt{\frac{\log(4/\delta)}{n}} \\
    &\leq 4C'\sigma_j(\widetilde \theta)\sqrt{\frac{\log(4/\delta)}{n}} + 2C'\varepsilon \sigma_L \sqrt{\frac{\log(4/\delta)}{n}}\\
    &\leq 2C'(2\sigma_j(\widetilde \theta) + \varepsilon \sigma_L)\sqrt{\frac{\log(4/\delta)}{n}}.
\end{align*}
A simple union bound yields that for all $j\in\setint{d}$
\begin{equation}\label{eq:fixedthetaCH}
    \big|\widecheck g_j^{\mathtt{CH}}(\widetilde \theta) - g_j(\widetilde \theta)\big|\leq 4C'(2\sigma_j(\widetilde \theta) + \varepsilon \sigma_L)\sqrt{\frac{\log(4d/\delta)}{n}} + \varepsilon \overline{L}.
\end{equation}

\paragraph{Step 3.}

We use the $\varepsilon$-net to obtain a uniform bound.
We proceed similarly to the proof of Proposition~\ref{prop:uniformMOM}. For $\theta \in \Theta$ denote $\widetilde{\theta}(\theta) \in N_{\varepsilon}$ the closest point in $N_{\varepsilon}$ satisfying in particular $\|\widetilde{\theta}(\theta) - \theta\| \leq \varepsilon$, we write, following previous arguments
\begin{align*}
    \sup_{\theta \in \Theta} \big| \wh g_j^{\mathtt{CH}} (\theta) - g_j(\theta) \big| &\leq \sup_{\theta \in \Theta} \big| \wh g_j^{\mathtt{CH}} (\theta) - g_j(\widetilde{\theta}(\theta)) \big| + \big| g_j (\widetilde{\theta}(\theta)) - g_j(\theta) \big| \\
    &\leq \sup_{\theta \in \Theta} \big| \widecheck g_j^{\mathtt{CH}} (\widetilde{\theta}(\theta)) - g_j(\widetilde{\theta}(\theta)) \big| + \varepsilon L_j \\
    &= \max_{\widetilde{\theta} \in N_{\varepsilon}} \big| \widecheck g_j^{\mathtt{CH}} (\widetilde{\theta}) - g_j(\widetilde{\theta}) \big| + \varepsilon L_{j}.
\end{align*}
Taking union bound over $\widetilde{\theta} \in N_{\varepsilon}$ for the inequality \eqref{eq:fixedthetaCH} and using the choice $\varepsilon = 1/\sqrt{n}$ concludes the proof of Proposition~\ref{prop:uniformCH}.

\subsection{Proof of Corollary~\ref{cor:estimate-lipschitz-constants}}

Under the assumptions made, the constants $(L_j)_{j\in  \setint d}$ are estimated using the $\mathtt{MOM}$ estimator and we obtain the bounds $(\overline L_j)_{j\in \setint d}$ which hold with probability at least $1 - \delta/2$ by a union bound argument. The rest of the proof is the same as that of Theorem~\ref{thm:linconv1expect} using a failure probability $\delta/2$ instead of $\delta$ and replacing the constants $(L_j)_{j\in \setint d}$ by their upperbounds accordingly. The result then follows after a simple union bound argument.

\subsection{Proof of Lemma~\ref{lem:mom-for-Lj}}

Let $B_1, \dots, B_K$ be the blocks used for the estimation so that $B_1 \cup \dots \cup B_K = \setint{n}$ and $B_{k_1} \cap B_{k_2} = \emptyset$ for $k_1\neq k_2$. Let $\cK$ denote the uncorrupted block indices $\cK = \{k \in\setint{K} \text{ such that  } B_k \cap \cO = \emptyset\}$ and assume $|\cO| \leq (1 - \varepsilon)K/2$.
For $k \in \setint{K}$ let $\wh{\sigma}^2_k = \frac{K}{n}\sum_{i\in B_k}X_i^2$ be the block means computed by MOM. Denote $N = n/K$, by using (a slight generalization of) Lemma~\ref{lem:mean-deviation-weak-moments} and the $L^{(1+\alpha)}$-$L^{1}$ condition satisfied by $X^2$ with a known constant $C$, we obtain that with probability at least $1-\delta$ we have 
\begin{equation*}
    |\wh{\sigma}^2_k - \sigma^2| \leq \Big( \frac{3\E |X^2 - \sigma^2|^{1+\alpha}}{\delta N^{\alpha}} \Big)^{\frac{1}{1+\alpha}} \leq \Big( \frac{3}{\delta N^{\alpha}} \Big)^{\frac{1}{1+\alpha}}C \E |X^2 - \sigma^2| \leq \Big( \frac{3}{\delta N^{\alpha}} \Big)^{\frac{1}{1+\alpha}}C \sigma^2,
\end{equation*}
which implies the inequality 
\begin{equation*}
    \sigma^2 \leq \Big(1 - C\Big( \frac{3}{\delta N^{\alpha}} \Big)^{\frac{1}{1+\alpha}}\Big)^{-1} \wh{\sigma}^2_k.
\end{equation*}
Define the Bernoulli random variables $U_k = \ind{}\Big\{\sigma^2 > \Big(1 - C\big( \frac{3}{\delta N^{\alpha}} \big)^{\frac{1}{1+\alpha}}\Big)^{-1} \wh{\sigma}^2_k\Big\}$ for $k \in \setint{K}$ which have success probability $\leq \delta$. Denote $S = \sum_k U_k$, we can bound the failure probability of the estimator as follows:
\begin{align*}
    \Proba \Big( \Big(1 - C\Big( \frac{3}{\delta N^{\alpha}} \Big)^{\frac{1}{1+\alpha}}\Big)^{-1} \wh{\sigma}^2 < \sigma^2 \Big) &\leq \P \big[ S > K/2 - |\cO|  \big] \\
    &= \P \big[ S - \E S > K/2 - |\cO| - \delta|\cK| \big] \\
    &\leq \P \big[ S - \E S > K (\varepsilon - 2\delta) / 2 \big] \\
    &\leq \exp \big(-K (\varepsilon - 2\delta)^2 / 2 \big),
\end{align*}
where we used the fact that $|\cO| \leq (1 - \varepsilon) K / 2$ and $|\cK| \leq K$ for the second inequality and Hoeffding's inequality for the last. The proof is finished by taking $\varepsilon = 5/6$ and $\delta = 1/4.$

\subsection{Proof of Lemma~\ref{lem:mom-for-grad-moment}}

Lemma~\ref{lem:mom-for-grad-moment} is a direct consequence of the following result.
\begin{lemma}\label{lem:mom-for-moment}
Let $X_1, \dots, X_n$ an i.i.d sample of a random variable $X$ with expectation $\E X = \mu$ and $(1+\alpha)$-moment $\E|X - \mu|^{1+\alpha} = m_{\alpha}< \infty$.
Assume that the variable $X$ satisfies the $L^{(1+\alpha)^2}$-$L^{(1+\alpha)}$ condition with constant $C >1$.
Let $\wh \mu$ be the median-of-means estimate of $\mu$ with $K$ blocks and $\wh{m}_{\alpha}$ a similarly obtained estimate of $m_{\alpha}$ from the samples $(|X_i - \wh \mu|^{1+\alpha})_{i\in \setint{n}}$. 
Then, with probability at least $1 - 2\exp(-K/18)$ we have
\begin{equation*}
    \wh{m}_{\alpha} \geq  (1 - \kappa)m_{\alpha},
\end{equation*}
with $\kappa = \epsilon + 24(1+\alpha) \Big(\frac{1 + \epsilon}{n/K}\Big)^{\frac{\alpha}{1+\alpha}}$ and $\epsilon = \Big( \frac{3\times 2^{2+\alpha}(1 + C^{(1+\alpha)^2})}{(n/K)^{\alpha}} \Big)^{\frac{1}{1+\alpha}}$.
\end{lemma}

\begin{proof}
Let $\wh \mu$ be the MOM estimate of $\mu$ with $K$ blocks, using Lemma~\ref{lem:basicMOM}, we have with probability at least $1 -\exp(-K/18)$,
\begin{equation}\label{eq:basicMOMholds}
    |\mu - \wh \mu| > (24 m_{\alpha})^{\frac{1}{1+\alpha}} \Big(\frac{K}{n}\Big)^{\frac{\alpha}{1+\alpha}}.
\end{equation}
Let $\wh{m}_{\alpha}$ be the MOM estimate of $m_{\alpha}$ obtained from the samples $\big(|X_i - \wh \mu|^{1+\alpha}\big)_{i\in\setint{n}}$. 
Denote $B_1, \dots, B_K$ the blocks we use, we have:
\begin{equation*}
    \wh{m}_{\alpha} = \median \bigg(\frac{K}{n} \sum_{i\in B_j} |X_i - \wh \mu|^{1+\alpha}\bigg)_{j \in \setint{K}}
\end{equation*}
for any $i \in \setint{n}$. Let $N = n/K$, using the convexity of the function $f(x) = |x|^{1+\alpha}$ we find that:
\begin{align}
    \frac{1}{N} \sum_{i\in B_j} \big|X_i - \wh \mu\big|^{1+\alpha} &= \frac{1}{N} \sum_{i\in B_j} \big|(X_i - \mu) + (\mu - \wh \mu)\big|^{1+\alpha} \nonumber \\
    &\geq \frac{1}{N} \sum_{i\in B_j} |X_i - \mu|^{1+\alpha} +\frac{1}{N} (1+\alpha) \sum_{i\in B_j} |X_i - \mu|^{\alpha}\sign (X_i - \mu)(\mu - \wh \mu) \nonumber \\
    &\geq \frac{1}{N} \sum_{i\in B_j} |X_i - \mu|^{1+\alpha} - (1+\alpha) |\mu - \wh \mu| \Big[  \frac{1}{N}\sum_{i\in  B_j} |X_i - \mu|^{\alpha} \Big] \nonumber \\
    &\geq \frac{1}{N} \sum_{i\in B_j} |X_i - \mu|^{1+\alpha} - (1+\alpha) |\mu - \wh \mu| \Big[  \frac{1}{N}\sum_{i\in  B_j} |X_i - \mu|^{1+\alpha} \Big]^{\frac{\alpha}{1+\alpha}}, \label{eq:mom-for-grad-moment-bound1}
\end{align}
where the last step uses Jensen's inequality. 
Using Lemma~\ref{lem:mean-deviation-weak-moments} we have, for $\delta > 0$, the concentration bound
\begin{equation*}
    \Proba \bigg( \Big|\frac{1}{N} \sum_{i\in B_j} \big|X_i - \mu\big|^{1+\alpha} - m_{\alpha}\Big| > \Big( \frac{3 \E\big||X - \mu|^{1+\alpha} - m_{\alpha}\big|^{1+\alpha}}{\delta N^{\alpha}} \Big)^{\frac{1}{1+\alpha}} \bigg) \leq \delta
\end{equation*}
which, using that $X$ satisfies the $L^{(1+\alpha)^2}$-$L^{(1+\alpha)}$ condition, translates to 
\begin{align*}
    \Proba \bigg( \Big|\frac{1}{N} \sum_{i\in B_j} \big|X_i - \mu\big|^{1+\alpha} - m_{\alpha}\Big| > \epsilon \bigg) &\leq \frac{3\E\big||X - \mu|^{1+\alpha} - m_{\alpha}\big|^{1+\alpha}}{\epsilon^{1+\alpha} N^{\alpha}} \\
    &\leq \frac{3\times 2^{\alpha} \big( \E |X-\mu|^{(1+\alpha)^2} + m_{\alpha}^{1+\alpha}\big)}{\epsilon^{1+\alpha} N^{\alpha}}\\
    &\leq \frac{3\times 2^{\alpha}m_{\alpha}^{1+\alpha} \big( 1 + C^{(1+\alpha)^2}\big)}{\epsilon^{1+\alpha} N^{\alpha}}.
\end{align*}
Replacing $\epsilon$ with $\epsilon m_{\alpha}$ we find
\begin{equation*}
    \Proba \bigg( \Big|\frac{1}{N} \sum_{i\in B_j} |X_i - \mu|^{1+\alpha} - m_{\alpha}\Big| > \epsilon m_{\alpha} \bigg) \leq \frac{3\times 2^{\alpha}\big(1 + C^{(1+\alpha)^2}\big)}{N^{\alpha} \epsilon^{1+\alpha}}.
\end{equation*}
Now conditioning on the event~\eqref{eq:basicMOMholds} and using the previous bound with $\epsilon = \Big( \frac{3\times 2^{\alpha}\big(1 + C^{(1+\alpha)^2}\big)}{N^{\alpha}\delta} \Big)^{\frac{1}{1+\alpha}}$ in~\eqref{eq:mom-for-grad-moment-bound1}, we obtain that
\begin{align*}
    \Proba \bigg( \frac{1}{N}&\sum_{i\in B_j} \big|X_i - \wh \mu\big|^{1+\alpha} \leq (1 - \epsilon)m_{\alpha} - (1+\alpha) \Big(\frac{24 m_{\alpha}}{N^\alpha}\Big)^{\frac{1}{1+\alpha}}((1 + \epsilon) m_{\alpha})^{\frac{\alpha}{1+\alpha}} \bigg) \leq \delta \\
    \implies &\Proba \bigg( \frac{1}{N}\sum_{i\in B_j} \big|X_i - \wh \mu\big|^{1+\alpha} \leq \underbrace{\Big(1 - \epsilon  - 24(1+\alpha) \Big(\frac{1 + \epsilon}{N}\Big)^{\frac{\alpha}{1+\alpha}}  \Big)}_{=:(1- \kappa)}m_{\alpha} \bigg) \leq \delta.
\end{align*}
Now define $U_j$ as the indicator variable of the event in the last probability. We have just seen it has success rate less than $\delta$. We can use the MOM trick, assuming the number of outliers satisfies $|\cO| \leq K(1- \varepsilon)/2$ for $\varepsilon \in (0,1)$, we have for $S = \sum_j U_j$
\begin{align*}
    \Proba (\wh{m}_{\alpha} \leq  (1 - \kappa)m_{\alpha} ) &\leq \Proba (S > K/2 - |\cO|) \\
    &= \P \big[ S - \E S > K/2 - |\cO| - \delta|\cK| \big] \\
    &\leq \P \big[ S - \E S > K (\varepsilon - 2\delta) / 2 \big] \\
    &\leq \exp \big(-K (\varepsilon - 2\delta)^2 / 2 \big).
\end{align*}
Taking $\varepsilon = 5/6$ and $\delta = 1/4$ yields that the previous probability is $\leq \exp(-K/18)$. 
Finally, recall that we conditioned on the event where the deviation $|\mu - \wh \mu|$ is bounded as previously stated and that this event holds with $\geq 1 - \exp(-K/18)$. Taking this conditioning into account and using a union bound argument leads to the fact that the bound
\begin{equation*}
    \wh{m}_{\alpha} \geq  (1 - \kappa)m_{\alpha}
\end{equation*}
holds with probability at least $1 - 2\exp(-K/18)$.
\end{proof}

\subsection{Proof of Theorem~\ref{thm:not-strgly-cvx}}

This proof is inspired from Theorem 5 in~\cite{nesterov2012efficiency} and Theorem 1 in~\cite{shalev2011stochastic} while keeping track of the degradations caused by the errors on the gradient coordinates.

We condition on the event~(\ref{eq:approx}) and denote $\epsilon_j = \epsilon_j(\delta)$ and $\epsilon_{Euc} = \|\epsilon(\delta)\|_2$. 
We define for all $\theta \in\Theta$ \begin{align*}u_j(\theta) &= \argmin_{\vartheta \in \Theta_j} \wh g_j(\theta)(\vartheta - \theta_j) + \frac{L_j}{2} (\vartheta - \theta_j)^2 + \epsilon_j|\vartheta - \theta_j|\\
    &= \proj_{\Theta_j} \big(\theta_j - \beta_j \tau_{\epsilon_{j}}\big(\wh g_j(\theta)\big)\big)
\end{align*}
and denote $\thetat$ the optimization iterates for $t=0,\dots, T$ and $j_t$ the random coordinate sampled at step $t$ and let $\wh g_t = \wh g_{j_t}(\thetat)$ for brevity. We have that $u_{j_t}(\thetat)$ satisfies the following optimality condition
\begin{equation*}
    \forall \vartheta \in \Theta_{j_t} \quad \big( \wh g_t + L_{j_t}\big(u_{j_t}(\thetat) - \thetat_{j_t}\big) + \epsilon_{j_t} \rho_t \big) \big( \vartheta - u_{j_t}(\thetat) \big) \geq 0,
\end{equation*}
where $\rho_t = \sign \big(u_{j_t}(\thetat) - \thetat_{j_t}\big)$. 
Using this condition for $\vartheta = \thetat_{j_t}$ and the coordinate-wise Lipschitz smoothness property of $R$ we find
\begin{align}
    R(\theta^{(t+1)}) &\leq R(\thetat) + g_{j_t}(\thetat)\big(u_{j_t}(\thetat) - \thetat_{j_t}\big) + \frac{L_{j_t}}{2}\big(u_{j_t}(\thetat) - \thetat_{j_t}\big)^2 \nonumber \\
    &\leq R(\thetat) + (\wh g_t + \epsilon_{j_t}\rho_t)\big(u_{j_t}(\thetat) - \thetat_{j_t}\big) + \frac{L_{j_t}}{2}\big(u_{j_t}(\thetat) - \thetat_{j_t}\big)^2 \label{eq:descent_var_thm3_1} \\
    &\leq R(\thetat) - \frac{L_{j_t}}{2}\big(u_{j_t}(\thetat) - \thetat_{j_t}\big)^2. \label{eq:descent_var_thm3}
\end{align}
Defining the potential $\Phi(\theta) = \sum_{j=1}^d L_j(\theta_j - \theta^\star_j)^2$, we have:
\begin{align*}
    \Phi(\theta^{(t+1)}) &= \Phi(\thetat) + 2L_{j_t}\big(u_{j_t}(\thetat) - \thetat_{j_t}\big)\big(\thetat_{j_t} - \theta^\star_{j_t}\big) + L_{j_t} \big(u_{j_t}(\thetat) - \thetat_{j_t}\big)^2 \\
    &= \Phi(\thetat) + 2L_{j_t}\big(u_{j_t}(\thetat) - \thetat_{j_t}\big)\big(u_{j_t}(\thetat) - \theta^\star_{j_t}\big) - L_{j_t} \big(u_{j_t}(\thetat) - \thetat_{j_t}\big)^2 \\
    &\leq \Phi(\thetat) - 2(\wh g_t + \epsilon_{j_t} \rho_t)\big(u_{j_t}(\thetat) - \theta^\star_{j_t}\big) - L_{j_t} \big(u_{j_t}(\thetat) - \thetat_{j_t}\big)^2 \\
    &= \Phi(\thetat) + 2(\wh g_t + \epsilon_{j_t} \rho_t)\big(\theta^\star_{j_t} - \thetat_{j_t}\big) - 2 \Big( (\wh g_t + \epsilon_{j_t} \rho_t)\big(u_{j_t}(\thetat)- \thetat_{j_t}\big) \\ & \quad +  \frac{L_{j_t}}{2} \big(u_{j_t}(\thetat) - \thetat_{j_t}\big)^2 \Big) \\
    &\leq \Phi(\thetat) + 2(\wh g_t + \epsilon_{j_t} \rho_t)\big(\theta^\star_{j_t} - \thetat_{j_t}\big) + 2 \big( R(\thetat) - R(\theta^{(t+1)}) \big) \\
    &\leq \Phi(\thetat) + 2g_{j_t}(\thetat) \big(\theta^\star_{j_t} - \thetat_{j_t}\big) + 2 \big( R(\thetat) - R(\theta^{(t+1)}) \big) + 4\epsilon_{j_t} \big|\theta^\star_{j_t} - \thetat_{j_t}\big|,
\end{align*}
where the first inequality uses the optimality condition with $\vartheta = \theta^\star_{j_t}$ and the second one uses~\eqref{eq:descent_var_thm3_1}. 
Now, defining $\Psi(\theta) = \frac 12 \Phi(\theta) + R(\theta)$, taking the expectation w.r.t. $j_t$ and using the convexity of $R$ and a Cauchy-Schwarz inequality, we find 
\begin{align*}
\E \big[\Psi(\thetat) - \Psi(\theta^{(t+1)})\big] \geq \frac{1}{d} \big(R(\thetat) - R(\theta^\star) - 2 \epsilon_{Euc} \big\|\thetat - \theta^\star\big\|_2\big).
\end{align*}
Recall that according to~\eqref{eq:descent_var_thm3}, we have $R (\theta^{(t+1)}) \leq R (\theta^{(t)})$, summing over $t = 0,\dots, T$ we find:
\begin{align*}
    \E\Big[\frac{T+1}{d}\big(R (\theta^{(T)}) - R (\theta^\star)\big)\Big] &\leq \E\Big[\frac{1}{d}\sum_{t=0}^T\big(R (\theta^{(t)}) - R (\theta^\star)\big)\Big] \\
    &\leq \sum_{t=0}^T \Big(\E\big[\Psi(\theta^{(t)}) - \Psi(\theta^{(t+1)})\big] + \frac{2\epsilon_{Euc}}{d}\E\big[\big\|\theta^{(t)} - \theta^\star\big\|_2\big]\Big)\\
    &=  \E\big[\Psi(\theta^{(0)}) - \Psi(\theta^{(t+1)})\big] + \frac{2\epsilon_{Euc}}{d}\sum_{t=0}^T\E\big[\big\|\theta^{(t)} - \theta^\star\big\|_2\big] \\
    &\leq  \Psi(\theta^{(0)}) + \frac{2\epsilon_{Euc}}{d}\sum_{t=0}^T\E\big[\big\|\theta^{(t)} - \theta^\star\big\|_2\big],
\end{align*}
which yields the result after multiplying by $\frac{d}{T+1}$. To finish, we show that conditionally on any choice of $j_t$ we have $\|\theta^{(t+1)} - \theta^\star\|_2 \leq \|\theta^{(t)} - \theta^\star\|_2.$ Indeed a straightforward computation yields 
\begin{equation*}
    \big\|\theta^{(t+1)} - \theta^\star\big\|_2^2 = \big\|\theta^{(t)} - \theta^\star\big\|_2^2 + \big(u_{j_t}(\thetat) - \thetat_{j_t}\big)^2 +2 \big(u_{j_t}(\thetat) - \thetat_{j_t}\big) \big(\theta^{(t)}_{j_t} - \theta^\star_{j_t}\big).
\end{equation*}
We need to show that $\delta_t^2  \leq -2 \delta_t \big(\theta^{(t)}_{j_t} - \theta^\star_{j_t}\big)$ with $\delta_t = (u_{j_t}(\thetat) - \thetat_{j_t})$. Notice that $\delta_t$ always has the opposite sign of $g_{j_t}(\theta^{(t)})$ (thanks to the thresholding) so by convexity of $R$ along the coordinate $j_t$ we have $\delta_t \big(\theta^{(t)}_{j_t} - \theta^\star_{j_t}\big) \leq 0$ and so it is down to showing $|\delta_t| \leq 2\big|\theta^{(t)}_{j_t} - \theta^\star_{j_t}\big|$ which can be seen from 
\begin{equation*}
    |\delta_t| \leq \frac{\big|g_{j_t}(\theta^{(t)})\big|}{L_{j_t}} = \frac{\big|g_{j_t}(\theta^{(t)}) - g_{j_t}(\theta^\star)\big|}{L_{j_t}} \leq \big|\theta^{(t)}_{j_t} - \theta^\star_{j_t}\big|,
\end{equation*}
which concludes the proof of Theorem~\ref{thm:not-strgly-cvx}.

\end{document}